\documentclass{article}
\usepackage{iclr2027_conference,times}

\usepackage{amsmath,amsfonts,bm}

\def\eqref#1{equation~\ref{#1}}

\def\1{\bm{1}}

\DeclareMathAlphabet{\mathsfit}{\encodingdefault}{\sfdefault}{m}{sl}
\SetMathAlphabet{\mathsfit}{bold}{\encodingdefault}{\sfdefault}{bx}{n}

\newcommand{\R}{\mathbb{R}}

\usepackage{url}
\usepackage{hyperref}
\usepackage{subcaption}
\usepackage{graphicx}
\usepackage{booktabs}
\usepackage{multirow}
\usepackage{siunitx}
\usepackage{accents}
\usepackage{ulem}
\usepackage{enumitem}
\usepackage{mathrsfs}
\usepackage{amsmath}
\usepackage{amsthm}
\usepackage{wrapfig}
\usepackage{microtype}
\usepackage{xurl}
\usepackage{thmtools}

\usepackage{natbib}
\usepackage{mathtools}
\usepackage{cleveref}
\usepackage{float}
\usepackage{tikz}
\usetikzlibrary{arrows.meta,positioning}

\usepackage[scaled=.92]{helvet} 
\usepackage[export]{adjustbox}

\newif\ifappendixdraftfigs

\newcommand{\appendixfig}[2][]{%
  \ifappendixdraftfigs
    \appendixfigdraft{#1}{#2}%
  \else
    \includegraphics[#1]{#2}%
  \fi
}
\def\appendixfigend{}
\def\appendixfigdraft#1#2{\appendixfigsplit #2\appendixfigend{#1}}
\def\appendixfigsplit figs_#1\appendixfigend#2{\includegraphics[#2]{figs_draft_#1}}

\title{WavePrune: One period is often enough for RoPE}

\author{%
  Guancheng Du\footnotemark[1] \quad
  Luotian Huang\footnotemark[1] \quad
  Shaowen Wang\footnotemark[1] \quad
  Si Li \quad
  Kaifeng Lyu \\
  Tsinghua University \\
  \texttt{\{dgc24,huanglt24,wangsw23,s-li24\}@mails.tsinghua.edu.cn}\\
  \texttt{klyu@mail.tsinghua.edu.cn}
}

\iclrpreprintcopy 
\begin{document}

\maketitle

{
\renewcommand{\thefootnote}{\fnsymbol{footnote}}
\footnotetext[1]{Equal contribution}
}

\renewcommand{\thefootnote}{\arabic{footnote}}
\setcounter{footnote}{0}

\begin{abstract}
Rotary Position Embedding (RoPE) encodes token positions by rotating each two-dimensional channel of the query and key vectors at a channel-specific frequency, making the attention logits invariant to a common shift of positions. However, this rotation is periodic, and it leads to position aliasing where relative positions separated by a full rotation period become hard to tell apart. To address this, we propose \textit{WavePrune}, which restricts each channel to its \textit{first rotation period}. We show that it removes the distractions in attention maps created by position aliasing and improves overall long-context performance. Specifically, WavePrune raises the HELMET score on four of five models we test without any extra tuning (e.g., 35.7 → 40.0 on Qwen3-8B).
When pretraining models from scratch, WavePrune also achieves lower validation loss at extrapolated lengths than pretraining without it.
Because WavePrune restricts each channel to a sliding window, it induces a fine-grained sparsity that our hardware-aligned CUDA kernels exploit for 1.15× prefill and 1.24× decoding speedups over FlashAttention-2 at 32K context. Together, these results show that RoPE's periodic structure, widely regarded as essential, is largely redundant beyond the first rotation period.\footnote{\url{https://github.com/ducatiqwq/WavePrune}}

\end{abstract}

\section{Introduction}\label{sec:intro}
Rotary Position Embedding (RoPE)~\citep{su2024roformer} is commonly used in large language models (LLMs) to encode token positions by applying position-dependent rotations to the query and key vectors in attention layers.
A key advantage of RoPE is that it makes the attention logit invariant under a common shift of the token positions,
a property usually referred to as \textit{translation invariance}.

More formally, let $q, k  \in \R^{d_h}$ be the query and key vectors at positions $i$ and $i'$, respectively. 
To compute the attention logit $L_{i,i'}$ between positions $i$ and $i'$,
RoPE first breaks $q, k$ into 2D channels $\bar{q}_c \coloneqq (q_{2c+1},q_{2c+2})$, $\bar{k}_c \coloneqq (k_{2c+1},k_{2c+2})$, where $c = 0, \ldots, d_h/2-1$, and then sums up the inner products between corresponding channels after rotating them by position-dependent angles:
\begin{equation}
    L_{i, i'} = \sum_{c=0}^{d_h/2-1} \ell_c(i-i'), \qquad
    \ell_c(\delta) \coloneqq
    \bar{k}_c^{\top}
    \begin{pmatrix}
    \cos(\delta f_c)&-\sin(\delta f_c)\\[2pt]
    \sin(\delta f_c)&\phantom{-}\cos(\delta f_c)
    \end{pmatrix}
    \bar{q}_c.
\end{equation}
\begin{center}
\includegraphics[width=0.65\textwidth]{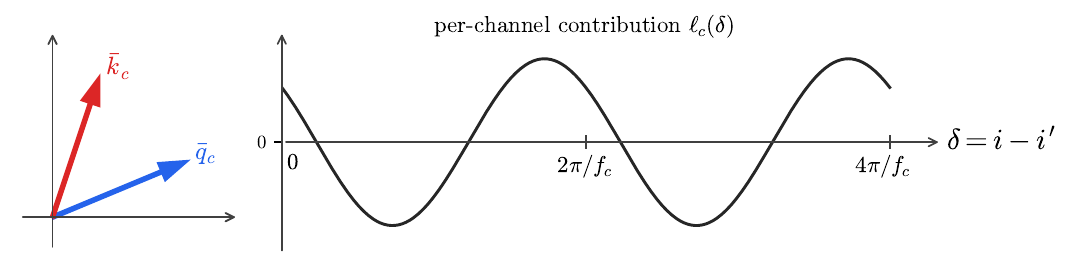}
\end{center}
Here $f_c$ is a fixed value that is called the \textit{rotation frequency} of channel $c$.
One can see that each per-channel contribution in RoPE is a sinusoidal function of the relative displacement $\delta \coloneqq i-i'$ with period $2\pi/f_c$ (for fixed $\overline{q}_c$ and $\overline{k}_c$), which is independent of the absolute positions $i$ and $i'$.
Consequently, the full attention logit $L_{i,i'}$ remains unchanged when both positions are shifted by the same amount, thus leading to translation invariance.

While this translation invariance is natural for language modeling, RoPE is not the only way to achieve it.
In fact, attention logits are translation invariant as long as each per-channel contribution is a function of the relative displacement $\delta \coloneqq i-i'$.
If so, \textit{why does the per-channel contribution in RoPE have to be periodic?}


There are two opposing views on the role of periodicity in RoPE.
On the one hand, these periodic components provide a Fourier-like representation of relative positions, which LLMs can exploit to construct specific positional attention patterns~\citep{barbero2025round,hua2025fourier,wang2026adarope}.
On the other hand, periodicity also introduces a fundamental ambiguity in positional encoding, as different relative positions can acquire similar rotary phases and consequently similar attention patterns~\citep{du2026ropedistinguishespositionstokens}.
This ambiguity can become especially problematic in long-context settings, where the model needs to distinguish relative positions over much longer distances. We refer to this phenomenon as \textit{position aliasing}, following the terminology in \citet{du2026ropedistinguishespositionstokens}.

In this paper, we support the second view by showing that repeated rotations beyond the first period are in fact unnecessary: \textbf{one period is often enough for RoPE}.



\textbf{Our Contributions.}
We propose to restrict each RoPE channel to its first rotation period, a method we call \textbf{WavePrune}.
Across several model families including Qwen, Gemma and Ministral, WavePrune does not harm long-context performance and often improves it (see \Cref{tab:main_results}).
To understand why this is the case, we visualize the attention maps of the original and WavePrune models (\Cref{fig:periodicity_attn_scores}),
and find that WavePrune eliminates many slash patterns that lie along sub-diagonals in the attention maps, which we refer to as \textit{``ghost'' sub-diagonals}.
While slash patterns have often been interpreted as useful attention structures for long-context modeling~\citep{yang2026why,cheng2026demystifying}, we demonstrate that some of them are instead a byproduct of position aliasing in RoPE, and that WavePrune may succeed by removing these distractions. Our key findings are summarized as follows:

\begin{figure}[t]
\centering
\begin{subfigure}[c]{0.3\textwidth}
\centering
\includegraphics[height=4.44cm]{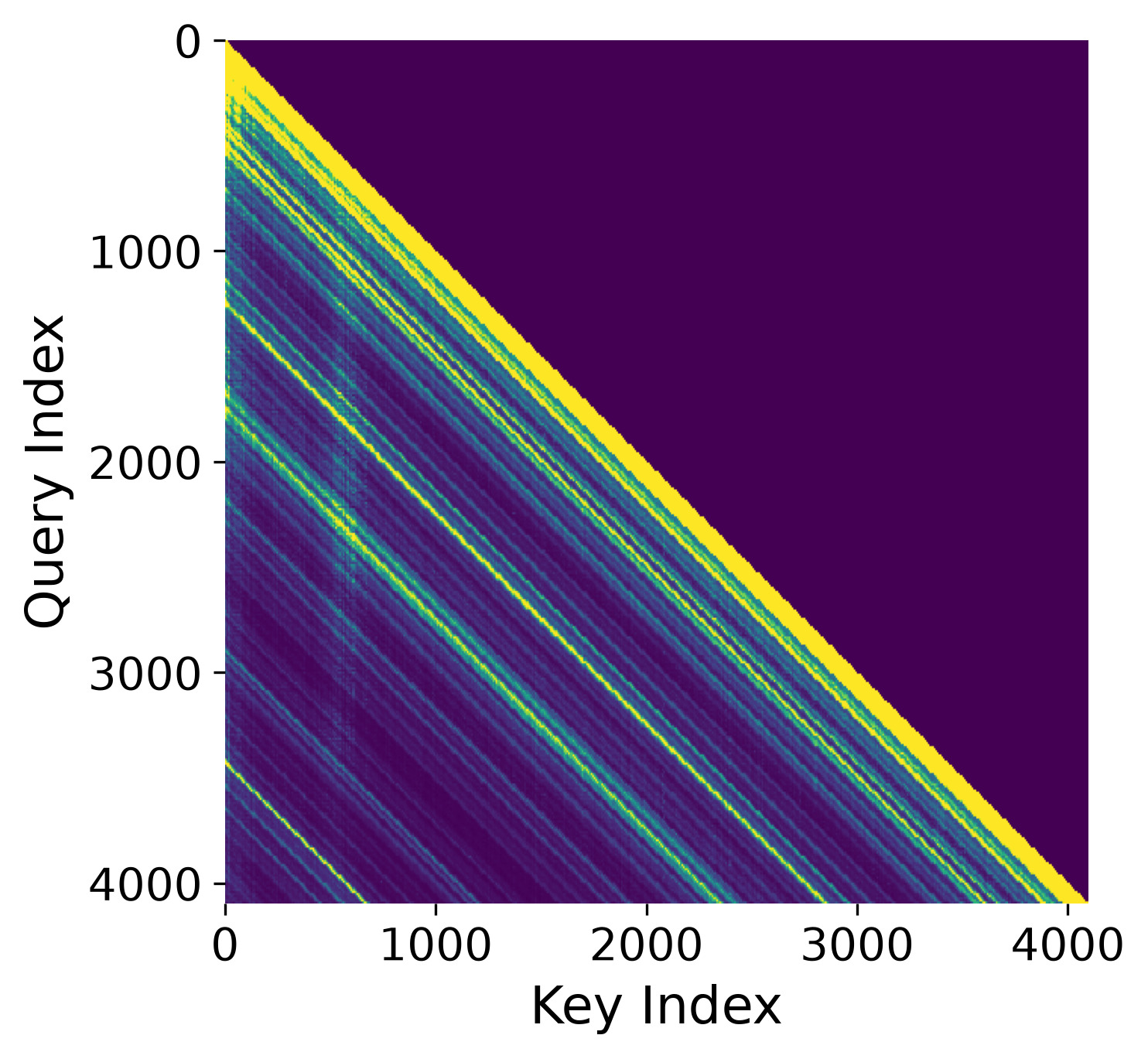}
\captionsetup{margin={1.3cm,0cm}}
\caption{Original}
\label{fig:periodicity_orig_head20}
\end{subfigure}%
\hspace{0.01\textwidth}%
\raisebox{0.4cm}{%
\makebox[0.38\textwidth][l]{\hspace{0.3cm}%
\begin{minipage}[c]{0.38\textwidth}
\centering
\includegraphics[width=4.3cm]{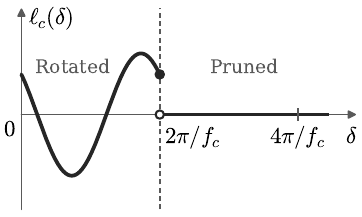}\\[-0.2em]
\begin{tikzpicture}
  \draw[-{Latex[length=2.4mm]}, thick] (0,0) -- (4.6cm,0);
\end{tikzpicture}
\end{minipage}%
}%
}%
\hspace{0.01\textwidth}%
\begin{subfigure}[c]{0.3\textwidth}
\centering
\includegraphics[height=4.44cm]{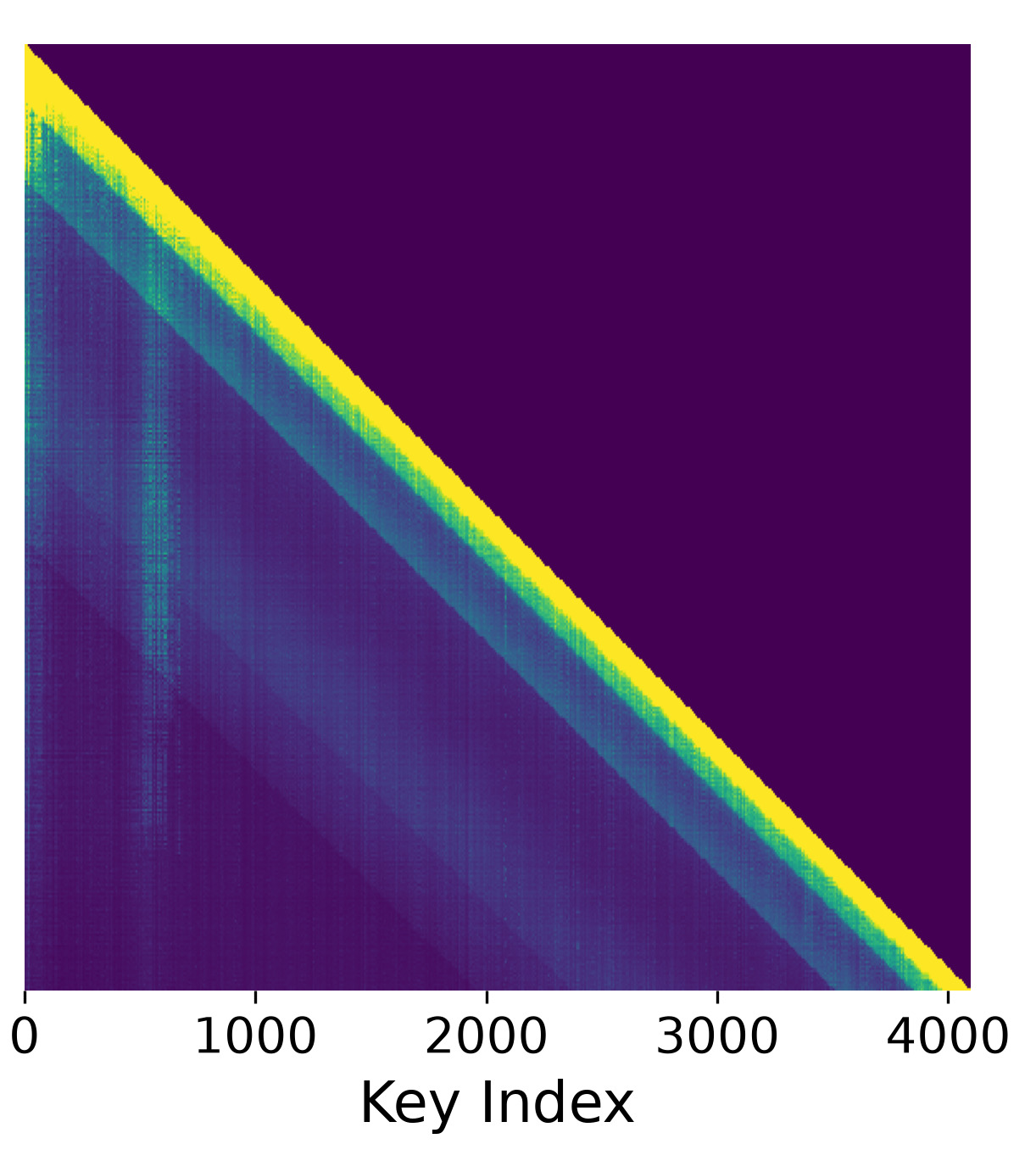}
\captionsetup{margin={0cm,0.16cm}}
\caption{WavePrune}
\label{fig:periodicity_wp_head20}
\end{subfigure}
\caption{In WavePrune, we restrict each channel to its first RoPE period, which removes several ``ghost'' sub-diagonals in the Qwen3-8B attention map.  See Appendix~\ref{app:attention_visualization_details} for details.}
\label{fig:periodicity_attn_scores}
\end{figure}
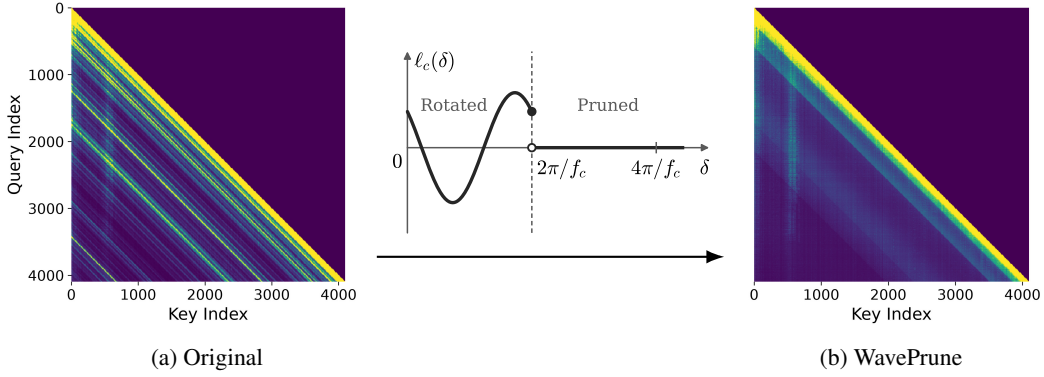


\begin{itemize}[leftmargin=2.4em,topsep=3.2pt,partopsep=3.2pt,parsep=3.2pt,itemsep=0pt]
    \item We propose WavePrune, a training-free method that restricts each RoPE channel to one rotation period, and show that it removes ``ghost'' sub-diagonals and improves long-context performance, raising the HELMET score on four of five models (e.g., $35.7 \to 40.0$ on Qwen3-8B).
    \item To assess the full potential of WavePrune, we also pretrain models with WavePrune from scratch, which consistently improves length extrapolation over pretraining without it. Controlled experiments on synthetic tasks suggest that WavePrune's success may be related to resolving two distinct failure modes of RoPE in handling short-range and long-range dependencies.
    \item We develop an efficient implementation of WavePrune, called \textbf{FlashWavePrune}, which reduces the computational cost and memory traffic of attention  by leveraging the sparse attention pattern of WavePrune.
    It achieves $1.15\times$ prefill and $1.24\times$ decoding speedups over FlashAttention-2 at 32K context.
\end{itemize}

\section{Preliminaries}\label{sec:preliminaries}
We first review the standard Transformer architecture \citep{NIPS2017_3f5ee243} and Rotary Position Embedding (RoPE) \citep{su2024roformer}, and introduce the notations used throughout the paper.

\textbf{Projecting tokens into QKV.} Consider a Transformer \citep{NIPS2017_3f5ee243} operating on a sequence of $n$ tokens, represented by an input matrix $X \in \mathbb{R}^{n \times d}$. For clarity, we describe a single attention head of dimension $d_h$; a multi-head attention layer simply runs multiple such heads in parallel and concatenates their outputs. The head first maps $X$ through three learned projection matrices $W_Q, W_K, W_V \in \mathbb{R}^{d \times d_h}$ to obtain the query, key, and value matrices:\vspace{0.25em}
\begin{equation}\label{eq:qkv}
Q \coloneqq X W_Q, \qquad K \coloneqq X W_K, \qquad V \coloneqq X W_V \;\in\; \mathbb{R}^{n \times d_h}.
\vspace{0.25em}\end{equation}
We write $q_i, k_i, v_i \in \mathbb{R}^{d_h}$ for the $i$-th rows of $Q, K, V$ (viewed as column vectors), and call them the \textit{query}, \textit{key}, and \textit{value} of token $i$, respectively.

\textbf{Scaled Dot-Product Attention (SDPA).} To determine how much every token attends to another, each query is compared against keys via scaled dot products. Under causal (autoregressive) masking, a token attends only to itself and to preceding tokens, so 
for $j \le i$ the \textit{attention logit} between query $i$ and key $j$ is defined as\vspace{0.25em}
\begin{equation}\label{eq:attention_logit_definition}
L_{i,j} \coloneqq 
\frac{1}{\sqrt{d_h}}\, q_i^{\top} k_j,
\vspace{0.25em}\end{equation}
and for $j > i$ the logit is defined as $L_{i,j} \coloneqq -\infty$.
Applying a row-wise softmax to the \textit{attention logit matrix} $L \in \mathbb{R}^{n \times n}$ yields the \textit{attention score matrix}
$A_{i,j} \coloneqq \exp(L_{i,j}) / \left(\sum_{k=1}^{n} \exp(L_{i,k})\right)$.
Consequently, every row of the resulting \textit{attention map} $A \in \mathbb{R}^{n \times n}$ forms a probability distribution over the keys. The head output then aggregates the value vectors as 
$Y \coloneqq A V \in \mathbb{R}^{n \times d_h}$, a convex combination weighted by the attention scores.

\textbf{Injecting positions by rotating queries and keys.} While the order of tokens is crucial for language modeling, the standard dot product $q_i^{\top} k_j$ in Equation~\ref{eq:attention_logit_definition} is oblivious to token position.
To inject positional information, Rotary Position Embedding (RoPE) \citep{su2024roformer} rotates the coordinates of queries and keys according to their token positions prior to the dot product.
More specifically, for a vector $x$ of dimension $d_h$, we split it into $d_h/2$ channels, where channel $c$ is the 2D vector formed by $(x_{2c+1}, x_{2c+2})$ ($c = 0, 1, \ldots, d_h/2-1$).
The rotated channel $c$ at position $i$ is defined as\vspace{0.25em}
\begin{equation}\label{eq:rope}
\varphi_i^{(c)}(x)
\coloneqq
\begin{pmatrix} \cos(i\,f_c) & -\sin(i\,f_c) \\[2pt] \sin(i\,f_c) & \phantom{-} \cos(i\,f_c) \end{pmatrix}
\begin{pmatrix} x_{2c+1} \\[2pt] x_{2c+2} \end{pmatrix}, \quad c = 0, 1, \ldots, d_h/2-1,
\vspace{0.25em}\end{equation}
where $f_c \in \R$ is a fixed frequency for channel $c$. We further define the position-dependent map $\varphi_i : \mathbb{R}^{d_h} \to \mathbb{R}^{d_h}$ as the vector formed by concatenating all the rotated channels, so that the channel $c$ of $\varphi_i(x)$ is $\varphi_i^{(c)}(x)$.
For unmasked tokens ($j \le i$), the attention logit in Equation~\ref{eq:attention_logit_definition} becomes
\begin{equation}\label{eq:rope_logit} 
L_{i,j} \coloneqq \frac{1}{\sqrt{d_h}}\, \varphi_i(q_i)^{\top} \varphi_j(k_j) = \frac{1}{\sqrt{d_h}} \! \sum_{c=0}^{d_h/2-1} \! \varphi_i^{(c)}(q_i)^{\top} \varphi_j^{(c)}(k_j).
\end{equation}
\textbf{Frequency schedule.} The sequence $\{f_c\}_{c=0}^{d_h/2-1}$ is called the \textit{frequency schedule}, which dictates how rapidly the queries and keys rotate with respect to their positions. A high frequency $f_c$ causes channel $c$ to rotate rapidly, making it ideal for distinguishing positions but unreliable for resolving long-range dependencies. Conversely, a low-frequency channel rotates slowly, making it better suited for tracking distant tokens \citep{barbero2025round,cheng2026demystifying}. In the standard RoPE formulation, the frequency schedule follows a geometric progression $f_c \coloneqq \theta^{-2c/d_h}$ governed by a large base $\theta$ (e.g., $\theta = 10^6$ for Qwen3~\citep{yang2025qwen3technicalreport}).
Several subsequent works have proposed alternative frequency schedules \citep{chen2023extendingcontextwindowlarge,peng2024yarn,barbero2025round} to extend the model's effective context length.



\section{WavePrune}


\subsection{WavePrune: Channel-wise Attention Truncation}\label{sec:channel_wise_attn_truncation}

In this paper, we propose WavePrune, which exploits RoPE's periodic structure to prune attention at \textit{channel level}. Within each RoPE channel $c$, WavePrune restricts the attention computation to a local window whose size equals the wavelength $w_c \coloneqq 2\pi / f_c$. The attention logit $L_{i,j}$ is then computed as
\begin{equation}\label{eq:waveprune}
L_{i,j} \coloneqq \frac{1}{\sqrt{d_h}} \! \sum_{c=0}^{d_h/2-1} \! \varphi_i^{(c)}(q_i)^{\top} \varphi_j^{(c)}(k_j) \cdot \mathbf{1}_{|i-j| \le w_c}
\end{equation}
Conceptually, WavePrune acts as a fine-grained mask over both the sequence and head dimensions. 
Compared to conventional attention masks that operate on tokens, such as those in sliding window attention (SWA)~\citep{beltagy2020longformerlongdocumenttransformer}, the key difference here is that WavePrune truncates the attention span in a channel-dependent manner (Figure~\ref{fig:waveprune_mask_concept}).

\begin{figure}[t]
    \centering
    \includegraphics[width=0.84\textwidth]{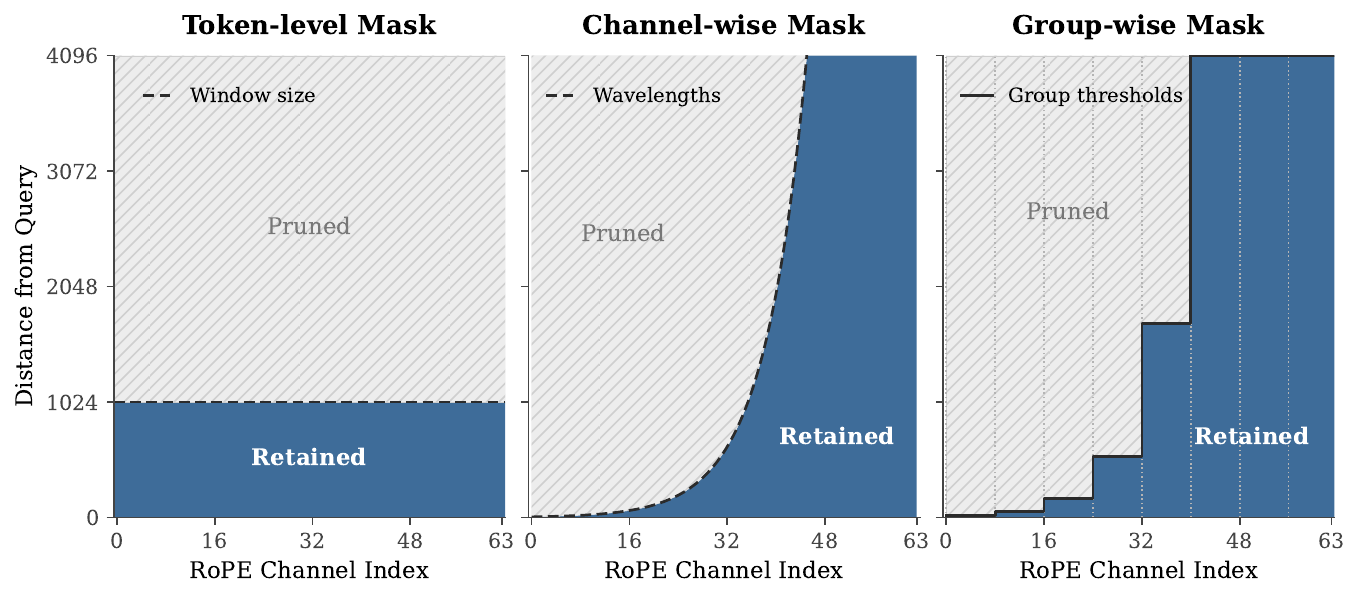}
    \caption{Conceptual comparison between the conventional token-level mask (e.g., sliding window mask), WavePrune's channel-wise mask, and its group-wise approximation. We use the group-wise approximation in our efficient implementation of WavePrune, FlashWavePrune (\Cref{sec:system_design}).}
    \label{fig:waveprune_mask_concept}
\end{figure}

\begin{wrapfigure}{r}{0.47\textwidth}
\vspace{-\intextsep}
\centering
\includegraphics[width=\linewidth]{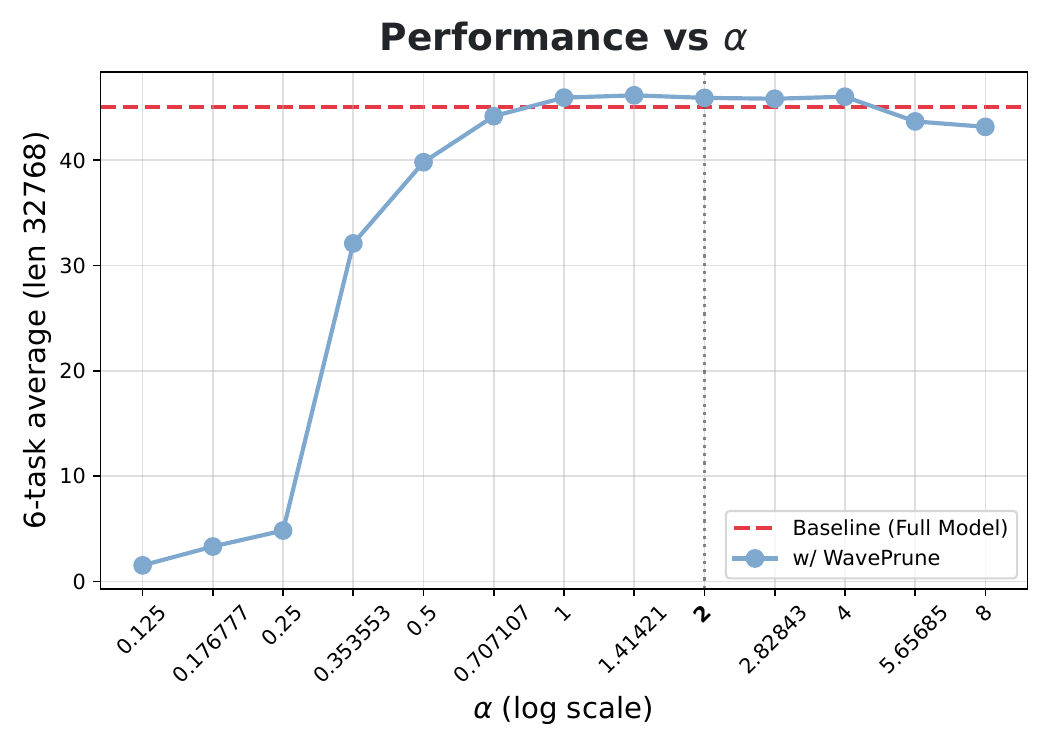}
\vspace{-2em}
\caption{
    When tuning the pruning coefficient $\alpha$ in the exact version of WavePrune, the HELMET score peaks slightly below $\alpha = 2$ (one period) and declines beyond it. The experiment is conducted on Qwen3-8B with 32K context. 
    }
\label{fig:alpha_exact}
\vspace{-0.1in}
\end{wrapfigure}

\textbf{One period is often enough for RoPE.} Since WavePrune discards all but the first rotation period of each channel, a natural question is whether such truncation loses useful context. To examine this, we scale each channel's threshold by a coefficient $\alpha$ (i.e., $w_c \coloneqq \alpha\pi/f_c$), where $\alpha=2$ corresponds to one period, the default setting.
As shown in Figure~\ref{fig:alpha_exact}, the HELMET score (Qwen3-8B, 32K context) peaks slightly below $\alpha=2$ and declines beyond it. This suggests that the first period already captures the required context information, while retaining too many periods hurts accuracy. Conversely, aggressive pruning at small $\alpha$ significantly degrades performance. Section~\ref{sec:experiments} details additional experimental results demonstrating WavePrune's benefits over the original model.

\textbf{Efficiency benefits of WavePrune.} If implemented properly, WavePrune can reduce the computational cost and memory traffic of attention. In Section~\ref{sec:system_design}, we develop an efficient implementation of WavePrune based on FlashAttention called \textbf{FlashWavePrune}, where the exact version of WavePrune is approximated by splitting the head dimensions into groups of size 16 and letting all dimensions in the same group share the same threshold for pruning (see Figure~\ref{fig:waveprune_mask_concept} for an illustration).
Beyond these computational benefits, WavePrune could also reduce KV cache memory by evicting cached key components once their distance from the current query exceeds the corresponding pruning thresholds. We leave this optimization for future work.

\subsection{Analysis: Periodicity of RoPE Distracts Attention}\label{sec:periodicity_distracts}

Having shown that WavePrune preserves model quality, a natural question is why attention beyond the first RoPE period can be safely discarded. We address this by examining LLMs' attention maps.

\begin{figure}[t]
    \centering
    \includegraphics[width=0.63\textwidth]{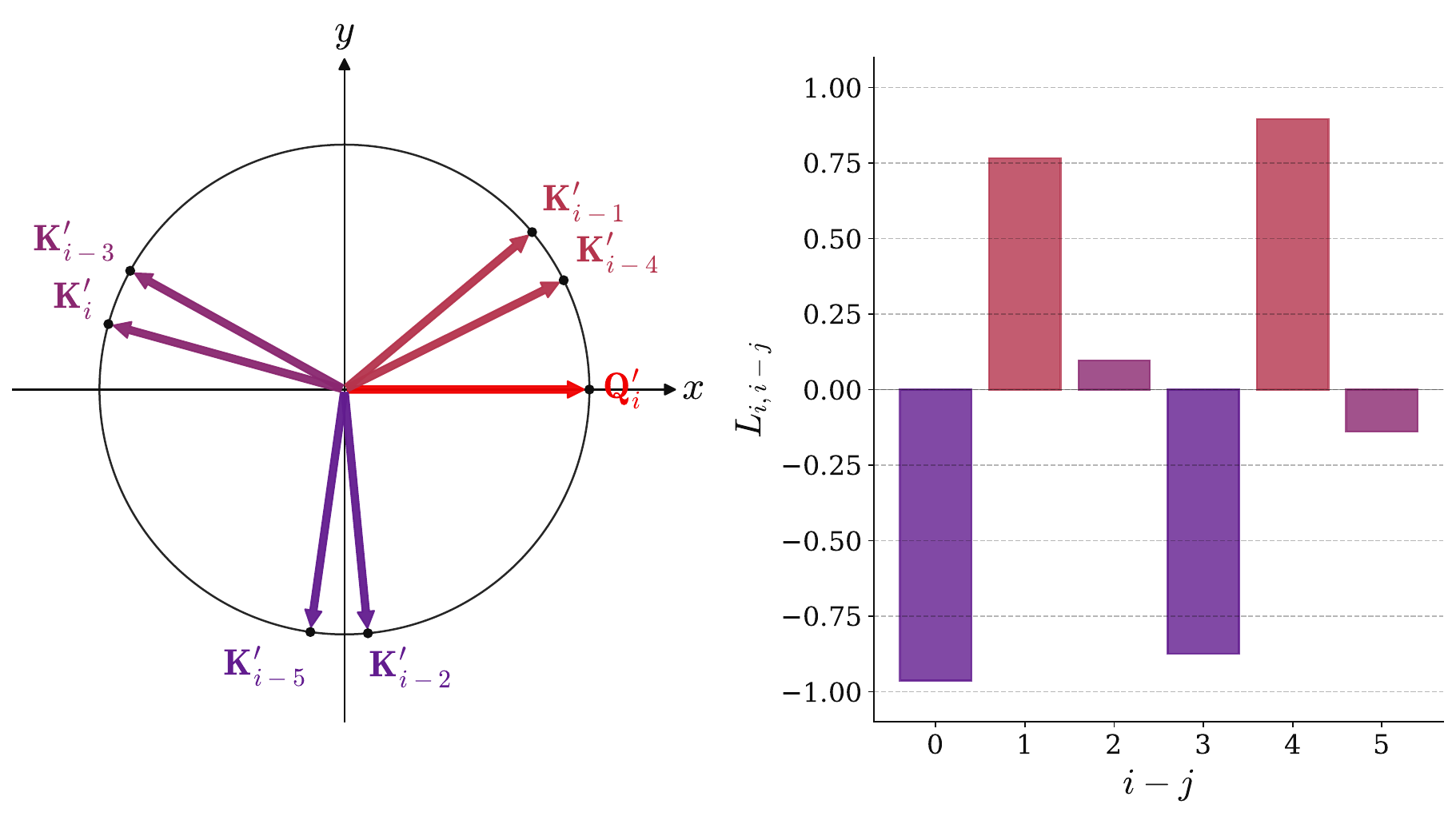}
    \vspace{-0.5em}
    \caption{Illustration of position aliasing in a single RoPE channel. \textbf{Left:} rotated query $\varphi_i(q)$ and keys $\varphi_j(k)$ on the unit circle; keys three positions apart nearly coincide (e.g., $\varphi_{i-1}(k) \approx \varphi_{i-4}(k)$) as their rotation angles differ by almost $2\pi$. \textbf{Right:} the attention logits $L_{i,j}$ are therefore near-periodic with period $3$ (e.g., $L_{i,i-1} \approx L_{i,i-4}$).}
    \label{fig:slash_dominant_head_rotation_collisions}
\end{figure}

\textbf{``Ghost'' sub-diagonals.} To reveal the underlying structural patterns of attention, we visualize the attention maps of LLMs, averaged over $100$ sequences from the PG19 dataset~\citep{rae2020compressive} with fixed sequence length $4096$. As a representative example, we plot the attention map of layer~0, head~20 of Qwen3-8B. As shown in Figure~\ref{fig:periodicity_orig_head20}, attention concentrates near the main diagonal, specifically on the lag-1 sub-diagonal $\{A_{i,i-1}\}_{i=2}^{n}$ (Table~\ref{tab:qwen3_attention_percentiles}), 
which is a typical pattern of a previous-token head \citep{elhage2021mathematical,olsson2022context,yedidia2023positional} where each query attends to the immediately preceding key. However, a considerable portion of attention also leaks onto many other sub-diagonals $\{A_{i, i-d}\}_{i=d+1}^{n}$ at larger lags $d > 1$. We refer to them as \textit{``ghost'' sub-diagonals}.


Importantly, these ``ghost'' sub-diagonals are not an isolated artifact but are prevalent across various heads and models (see Appendix~\ref{app:additional_attention_score_visualizations} for other examples), and analogous patterns have been observed in prior work \citep{yang2026why,shen2026ssa,cheng2026demystifying,du2026ropedistinguishespositionstokens}. As context length grows, the ``ghost'' sub-diagonals increasingly dilute attention away from the desired positions.

\textbf{Origin of the ``ghost'' sub-diagonals.} To identify their source, we decompose the attention logit matrix $L$ (Equation~\ref{eq:rope_logit}) into its per-channel contributions $L_{i,j}^{(c)} \coloneqq \varphi_i^{(c)}(q_i)^{\top} \varphi_j^{(c)}(k_j)$, the QK product restricted to individual 2D RoPE channels. As shown in Figure~\ref{fig:qwen3_layer0_head20_channel_visualization_original}, we observe that the logits $L_{i,j}^{(c)}$ of high-frequency channels oscillate rapidly between positive and negative values as the lag $d \coloneqq i-j$ increases. When the per-channel logits are summed and softmax-normalized, such oscillatory structure manifests as the ``ghost'' sub-diagonals (Figure~\ref{fig:periodicity_orig_head20}).

\textbf{Why per-channel attention logits oscillate.} Consider a single RoPE channel with constant queries and keys (i.e., $q_i=q \in \mathbb{R}^2$ and $k_i=k \in \mathbb{R}^2$ for all $1 \le i \le n$) and rotation frequency $\theta$. Ignoring scaling factors, the channel's attention logit $L_{i,j}^{(c)} (i \ge j)$ is the dot product between the rotated query and key, which is proportional to $\cos(\theta (i-j) + \omega)$, where $\omega$ is the initial angle between $q$ and $k$. This leads to \textit{position aliasing} \citep{du2026ropedistinguishespositionstokens}: different lags $i-j$ produce highly similar attention logits due to near-identical rotation angles $\theta(i-j)$ modulo $2\pi$. Figure~\ref{fig:slash_dominant_head_rotation_collisions} shows this on the unit circle, where keys several positions apart rotate to nearly the same angle, giving near-periodic attention logits and hence infinitely many ``ghost'' sub-diagonals.



The theorem below formalizes this intuition for multiple channels, extending the single-channel analysis in prior work \citep{yang2026why}. We consider the simplest case where all queries and all keys are both identical, so that the attention logit depends only on the lag. Even in this setting, attention is provably distracted by infinitely many ``ghost'' sub-diagonals; see Figure~\ref{fig:controlled_qk_attn} for an illustration. We prove it in Appendix~\ref{app:direct_proof} and give a stronger version in Appendix~\ref{app:theory}.

\begin{restatable}[Infinitely Many Ghost Sub-Diagonals]{theorem}{periodicitythm}\label{thm:periodicity}
Assume $q_i = q$ and $k_i = k$ for all $i$, then $L_{i,j}$ depends only on the lag $i-j$. Let $L^\star \coloneqq \sup_{i \ge j \ge 1} L_{i,j}$ be the largest logit in the attention map. For any tolerance $\tau > 0$ and frequency schedule $\{f_c\}$, there exist infinitely many lags $\Delta \ge 0$ such that
\begin{equation}
L_{i,i-\Delta} \ge L^\star - \tau, \qquad \text{for all } i > \Delta.
\end{equation}
\end{restatable}
That is, no single sub-diagonal can dominate the attention map: the largest logit is approximately replicated onto infinitely many ``ghost'' sub-diagonals.

\textbf{WavePrune eliminates ``ghost'' sub-diagonals.} By confining each channel to its primary period, WavePrune ensures the relative rotation angle between two positions never wraps around $2\pi$, which removes position aliasing at its source. As shown in Figure~\ref{fig:qwen3_layer0_head20_channel_visualization_waveprune}, WavePrune suppresses the oscillations in every channel, restoring a clean attention map without any ``ghost'' sub-diagonals (Figure~\ref{fig:periodicity_wp_head20}).
Figure~\ref{fig:qwen3_nll_without_yarn} further shows that applying WavePrune in this case indeed does not hurt model performance.

\subsection{Efficient Implementation: FlashWavePrune}\label{sec:system_design}

We now present \textbf{FlashWavePrune}, an efficient implementation of WavePrune that leverages its sparse attention pattern to reduce the cost of attention computation. To fully utilize Tensor Cores in modern GPUs, it uses the group-wise approximation of WavePrune (\Cref{fig:waveprune_mask_concept}).



\textbf{Overview.} FlashWavePrune builds on FlashAttention (FA) \citep{dao2024flashattention} and only differs in its QK multiplication path.
A key challenge here is that WavePrune assigns a different pruning threshold $w_c$ to each channel~$c$ (\Cref{eq:waveprune}), so its logit matrix no longer corresponds to a standard General Matrix Multiplication (GEMM) and cannot fully utilize Tensor Cores.
A straightforward implementation would instead run on CUDA cores, but then the throughput would be significantly lower.
To solve this, we split the head dimension into groups of size 16 and let all dimensions in a group share one pruning threshold. Then, each group has a single mask for a given pair of $Q$ and $K$ blocks. We can therefore compute each group with a standard GEMM on Tensor Cores. For most groups, the mask is either all zeros or all ones: we skip fully masked groups altogether, saving FLOPs and, during decoding, key-cache loads. Fully retained groups follow the original FA path, and only the few remaining groups require an explicit mask after the GEMM. The kernel design is detailed below.

\textbf{Preliminaries: MMA atoms and head-dimension groups.} Efficient GEMM implementations tile the query ($M$), key ($N)$, and head ($K$) dimensions into small, fixed-size \textit{MMA atoms} (e.g., \texttt{M16N8K16}), where each atom consumes contiguous fragments along these three dimensions and executes an FMA (Fused-Multiply-Add) operation on Tensor Cores. Note that an MMA atom is indivisible: masking can only be applied to an entire group, not to individual elements within it. Since the head dimension $d_h$ typically exceeds the atom's $K$ size (typically $16$), the kernel splits it into $d_h/16$ \textit{head-dimension groups}, accumulating each group's partial GEMM result to form the final logits.

\textbf{Approximating WavePrune for Tensor Core MMA.} Because an MMA atom cannot be split, we propose an approximation for WavePrune to align the dimension-wise pruning thresholds $\{d_k\}$ derived from channel-wise thresholds $\{w_c\}$ with the head-dimension groups. Specifically, for each group $g$, we assign all head dimensions in the group a uniform threshold equal to the group-wise maximum $\max_{k \in \text{group} g} \{d_k\}$ (Figure~\ref{fig:waveprune_mask_concept}). Under this approximation, we can perform a standard GEMM for the group, and apply the distance predicate $\mathbf{1}_{|i-j| \le d_k}$ directly to the group's GEMM result. Formally, the GEMM result $L_{i,j}^{g}$ for the $g$-th head-dimension group is added to the logit $L_{i,j}$ iff
\begin{equation}\label{eq:waveprune_approximation}
\begin{aligned}
|i-j| \le d'_g
\coloneqq \max_{k \in \text{group } g} \{d_k\}
= d_{\mkern 1.5mu \text{last dimension in group } g}.
\end{aligned}
\end{equation}
In summary, by coalescing per-channel distance predicates into a single predicate per head-dimension group, we preserve the integrity of each MMA atom, enabling efficient execution on Tensor Cores. Empirically, this approximation has minimal effect on model performance (Section~\ref{sec:ablation_studies}).

\textbf{Reducing FLOPs and memory traffic.} FA partitions $Q,K \in \mathbb{R}^{n \times d_h}$ into blocks along the sequence dimension and multiplies each pair of $Q$- and $K$-blocks. Let $I_Q,I_K \subseteq \{1,\ldots,n\}$ denote the sequence indices of a $Q$-block and a $K$-block, respectively, and let $D(I_Q, I_K) \coloneqq \{|i-j| \mid i \in I_Q, j \in I_K\}$ denote their pairwise distances. Since block indices are contiguous, $D(I_Q, I_K)$ is also contiguous and forms an integer interval $[l_D, r_D]$. For each head-dimension group $g$, we compare its group-wise threshold $d'_g$ (Equation~\ref{eq:waveprune_approximation}) with $l_D$ and $   r_D$ to classify it into one of the three categories:
\begin{itemize}[leftmargin=2.4em,topsep=1.6pt,partopsep=1.6pt,parsep=3.2pt,itemsep=0pt]
\item \textit{Fully masked groups:} $d'_g < l_D$. The entire head-dimension group is masked.
\item \textit{Fully retained groups:} $d'_g \ge r_D$. The entire head-dimension group is retained.
\item \textit{Critical groups:} $l_D \le d'_g < r_D$. The head-dimension group is partially masked.
\end{itemize}

Since $d_k$ increases monotonically with $k$, head-dimension groups are partitioned by two split points: \texttt{splitL}, the first critical group, and \texttt{splitR}, the first fully retained group.

Since groups before \texttt{splitL} are fully masked, we skip all their QK multiplications and, during decode, bypass their key-cache loads. This reduces both FLOPs and memory traffic. For groups at or after \texttt{splitR}, we use FA2's standard FMA path; only the narrow critical region [\texttt{splitL}, \texttt{splitR}) requires group-wise GEMM and masking. Figure~\ref{fig:waveprune_split_design} summarizes this design.

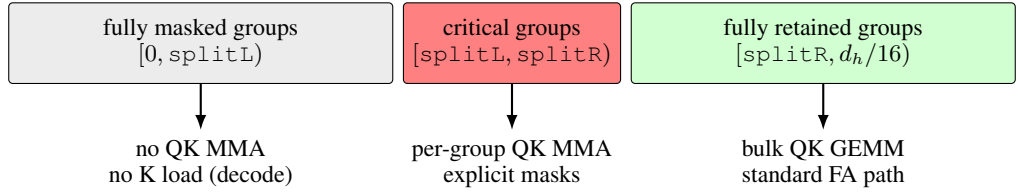
\begin{figure}[t]
\centering
\begin{tikzpicture}[
  font=\small,
  group/.style={draw, rounded corners=1.5pt, minimum height=0.42in, align=center},
  arrow/.style={-{Latex[length=2mm]}, thick}
]
\node[group, fill=gray!15, minimum width=2.0in] (skip) {fully masked groups\\[-0.3pt]$[0,\texttt{splitL})$};
\node[group, fill=red!50, minimum width=0.8in, right=0.05in of skip] (boundary) {critical groups\\[-0.3pt]$[\texttt{splitL},\texttt{splitR})$};
\node[group, fill=green!18, minimum width=2.0in, right=0.05in of boundary] (suffix) {fully retained groups\\[-0.3pt]$[\texttt{splitR},d_h/16)$};
\node[below=0.24in of skip, align=center] (zero) {no QK MMA\\no K load (decode)};
\node[below=0.24in of boundary, align=center] (mask) {per-group QK MMA\\explicit masks};
\node[below=0.24in of suffix, align=center] (gemm) {bulk QK GEMM\\standard FA path};
\draw[arrow] (skip) -- (zero);
\draw[arrow] (boundary) -- (mask);
\draw[arrow] (suffix) -- (gemm);
\end{tikzpicture}
\caption{Head-dimension groups are partitioned into three regions with distinct compute paths.}
\label{fig:waveprune_split_design}
\end{figure}

\textbf{Other optimizations.} To further improve efficiency, we conduct compile-time specialization for \texttt{splitL} and \texttt{splitR} while carefully controlling branching and SASS code size. We defer these techniques and the backward kernel implementation to Appendix~\ref{app:kernel_implementation}.


\section{Experiments}\label{sec:experiments}
In this section, we evaluate WavePrune's effectiveness in two primary settings: training-free integration (Section~\ref{sec:pretrained_integration}) and training from scratch (Section~\ref{sec:pretraining_experiments}). We then examine its system efficiency (Section~\ref{sec:system_efficiency}) and test our design choices with ablation studies (Section~\ref{sec:ablation_studies}). Unless specified otherwise, all experiments use FlashWavePrune (\Cref{sec:system_design}), our efficient implementation of WavePrune.




\subsection{Training-Free Integration}\label{sec:pretrained_integration}


We first apply WavePrune \textit{training-free} to diverse LLMs, evaluating its performance on long-context and general benchmarks.

\textbf{Long-context evaluation.} We evaluate WavePrune on HELMET \citep{yen2025helmet}, a comprehensive benchmark suite extending to sequence lengths as long as 128K. To ensure broad applicability, we test a diverse range of long-context models: Qwen3-8B \citep{yang2025qwen3technicalreport}, Qwen3-32B, Gemma3-12B \citep{gemmateam2025gemma3technicalreport}, Ministral3-3B \citep{liu2026ministral3}, and Llama3.1-8B \citep{grattafiori2024llama3herdmodels}. These selections span multiple model scales, underlying architectures, and model families (Table~\ref{tab:pretrained_integration_models}). Note that these models rescale their RoPE frequencies at long contexts (e.g., via YaRN \citep{peng2024yarn}); WavePrune derives each channel's wavelength from the updated frequencies. Appendix~\ref{app:pretrained_integration} provides additional details on the evaluation setup.

\textbf{Long-context results.} As shown in Table~\ref{tab:main_results}, WavePrune improves the HELMET score on four of the five models, all without any fine-tuning (separate 64K and 128K results are in Tables~\ref{tab:helmet_64k} and~\ref{tab:helmet_128k}). The gain is largest on Qwen3-8B ($35.7 \to 40.0$), where Recall and ICL improve by $6.0$ and $10.0$ points respectively; its negative log-likelihoods (NLL) are also reduced (Figure~\ref{fig:qwen3_nll_yarn}). Ministral3-3B, Gemma3-12B, and Qwen3-32B improve by smaller yet consistent margins. Conversely, WavePrune fails on Llama3.1-8B; Appendix~\ref{app:llama31_failure} describes our diagnosis of the failure and attempts to fix it.

\newcommand{\mstd}[1]{{\textcolor{gray}{\tiny$\,\pm\,$#1}}}
\begin{table}[t]
\centering
\caption{Results on the HELMET benchmark; metrics are averaged across the 64K and 128K context splits. We report cross-task averages and standard deviations (gray) to account for task-level variance. Despite pruning a significant portion of the attention computation, WavePrune maintains or improves performance on most models (Ministral3-3B, Qwen3-8B, Gemma3-12B, Qwen3-32B). Conversely, performance degrades when applying WavePrune to Llama3.1-8B.}
\label{tab:main_results}
{\fontsize{9}{10.8}\selectfont
\setlength{\tabcolsep}{3pt}
\renewcommand{\arraystretch}{1.1}
{%
\begin{tabular}{l@{\hspace{5pt}}|@{\hspace{5pt}}l@{\hspace{5pt}}|@{\hspace{5pt}}cccccc@{\hspace{5pt}}|@{\hspace{5pt}}c}
\toprule
\multicolumn{2}{c}{} & \multicolumn{6}{c}{{\textbf{HELMET}}} & \multicolumn{1}{c}{} \\
\cmidrule(lr){3-8}
{\textbf{Model}} & {\textbf{Method}} & {\textbf{Recall}} & {\textbf{RAG}} & {\textbf{Re-rank}} & {\textbf{Cite}} & {\textbf{LongQA}} & {\textbf{ICL}} & {\textbf{Avg.}} \\
\midrule
\multirow{2}{*}{Ministral3-3B} & Baseline & 85.7\mstd{1.2} & 58.0\mstd{1.8} & 23.8\mstd{1.7} & 9.7\mstd{0.9} & 36.9\mstd{1.7} & 20.5\mstd{1.3} & 39.1 \\
 & WavePrune & 86.1\mstd{1.2} & 57.8\mstd{1.8} & 23.2\mstd{1.6} & 9.8\mstd{0.9} & 36.7\mstd{1.7} & 24.8\mstd{1.4} & \textbf{39.7} \\
\midrule
\multirow{2}{*}{Qwen3-8B} & Baseline & 63.7\mstd{1.7} & 54.3\mstd{1.8} & 14.7\mstd{1.6} & 15.7\mstd{1.2} & 42.7\mstd{1.7} & 22.9\mstd{1.3} & 35.7 \\
 & WavePrune & 69.7\mstd{1.6} & 55.6\mstd{1.8} & 21.6\mstd{1.8} & 15.2\mstd{1.2} & 44.8\mstd{1.7} & 32.9\mstd{1.5} & \textbf{40.0} \\
\midrule
\multirow{2}{*}{Gemma3-12B} & Baseline & 31.5\mstd{1.5} & 54.1\mstd{1.8} & 23.5\mstd{1.8} & 0.8\mstd{0.3} & 43.2\mstd{1.8} & 76.4\mstd{1.3} & 38.2 \\
& WavePrune & 33.1\mstd{1.5} & 53.0\mstd{1.8} & 23.7\mstd{1.8} & 1.1\mstd{0.4} & 44.5\mstd{1.8} & 77.0\mstd{1.3} & \textbf{38.7} \\
\midrule
\multirow{2}{*}{Qwen3-32B} & Baseline & 79.1\mstd{1.4} & 61.4\mstd{1.7} & 34.0\mstd{1.8} & 27.1\mstd{1.7} & 49.5\mstd{1.7} & 79.5\mstd{1.3} & 55.1 \\
 & WavePrune & 86.5\mstd{1.2} & 61.3\mstd{1.7} & 34.4\mstd{1.9} & 23.9\mstd{1.5} & 50.1\mstd{1.7} & 81.6\mstd{1.2} & \textbf{56.3} \\
\midrule\\[-1.55em]
\midrule
\multirow{2}{*}{Llama3.1-8B} & Baseline & 84.9\mstd{1.2} & 56.4\mstd{1.8} & 20.0\mstd{1.7} & 2.2\mstd{0.0} & 39.3\mstd{1.7} & 26.5\mstd{1.4} & \textbf{38.2} \\
 & WavePrune & 75.5\mstd{1.5} & 55.6\mstd{1.8} & 16.2\mstd{1.7} & 2.2\mstd{0.4} & 39.3\mstd{1.7} & 22.5\mstd{1.3} & 35.2 \\
\bottomrule
\end{tabular}%
}}
\renewcommand{\arraystretch}{1}
\end{table}

\textbf{General evaluation.} We also evaluate the models on a comprehensive suite of general benchmarks (see Appendix~\ref{app:pretrained_integration} for details), using the lm-evaluation-harness framework \citep{eval-harness} with default settings. As shown in Table~\ref{tab:short_context_results}, WavePrune maintains and often improves model performance across these diverse tasks, confirming no loss of general capability.

\subsection{Pretraining Experiments}\label{sec:pretraining_experiments}

Beyond applying WavePrune to existing models, to assess its full potential, we explore training with WavePrune from scratch. We focus on the length extrapolation capability, i.e., how well the model generalizes zero-shot to sequences longer than those seen during training \citep{press2022train}.

\textbf{Experimental setup.} We train models using the GPT-2 architecture \citep{radford2019language} at three scales: GPT2-base (117M), GPT2-medium (345M), and GPT2-large (774M). We replace the original absolute positional embedding with RoPE \citep{su2024roformer} and set the RoPE base $\theta$ to $10^4$. All pretraining runs use the nanoGPT codebase \citep{Karpathy2022} and the OpenWebText dataset\footnote{\url{https://huggingface.co/datasets/openwebtext}}. To measure length extrapolation capabilities during training, we use \textit{different context lengths for the training and validation sets}: the training data has a fixed sequence length of 2048, while the validation data has fixed length 4096. For each model scale, we conduct paired training runs, either enabling or disabling WavePrune. Appendix~\ref{app:pretraining_experiments} summarizes all pretraining hyperparameters.


\textbf{Pretraining results.} Figure~\ref{fig:pretrain_results_val} shows the validation loss curves during pretraining. Across all model scales, WavePrune consistently improves length extrapolation by a large margin, as evidenced by the substantially lower loss on the length-extrapolated validation set, suggesting that WavePrune encourages the model to learn more length-generalizable patterns rather than overfit to the training sequence length. These gains over the baseline persist when increasing the RoPE base $\theta$ after training (Figure~\ref{fig:ppl_vs_length}). Since the in-distribution validation loss (length 2048) with WavePrune is comparable to the baseline (Table~\ref{tab:gpt2_train_eval_2048}), WavePrune preserves the model's ability to fit in-distribution data.

\begin{figure}[t]
\centering
\includegraphics[width=0.96\textwidth]{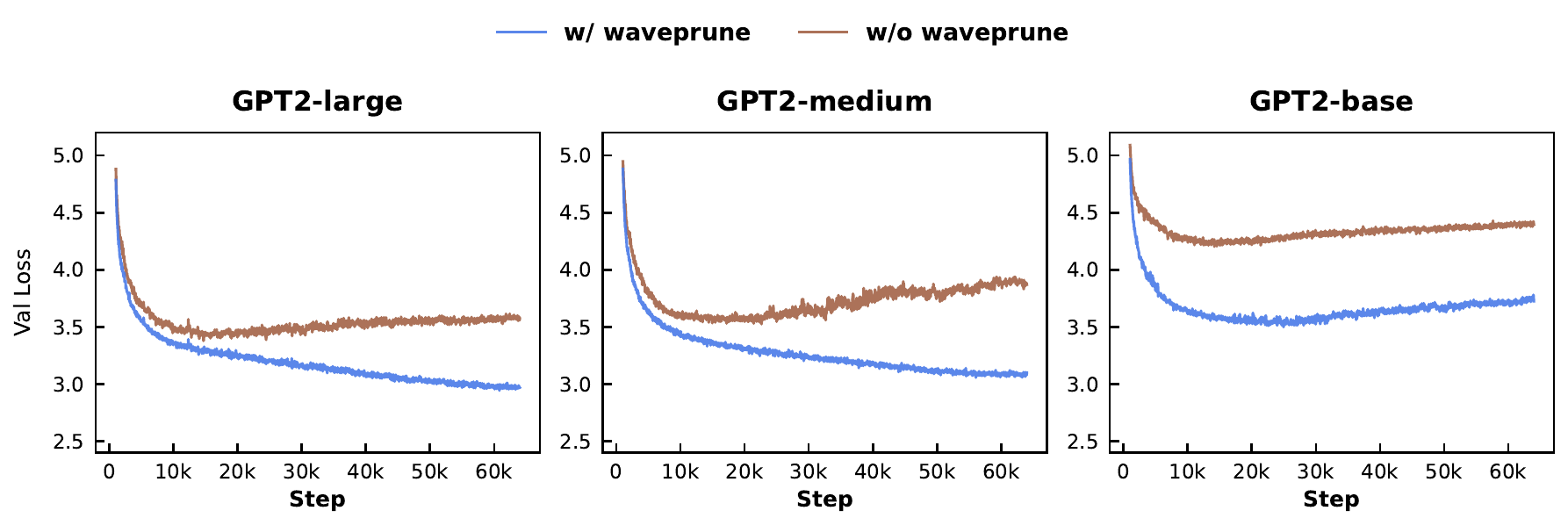}
\caption{Length-extrapolated validation loss of GPT2-large (left), GPT2-medium (middle), and GPT2-base (right) on OpenWebText during pretraining. Training data has length 2048 and validation data has length 4096. Across all model scales, the models trained and evaluated with WavePrune consistently outperform the baselines at the extrapolated context length of 4096.}
\label{fig:pretrain_results_val}
\end{figure}


\textbf{Comparison with training-free integration.}
Since WavePrune can be independently enabled or disabled during pretraining and at inference, we evaluate all four resulting combinations on GPT2-medium and GPT2-large. As shown in Table~\ref{tab:gpt2_train_eval_4096}, training-free integration already improves over the baseline, yet pretraining with WavePrune consistently outperforms this plug-and-play setting: at the extrapolated length of 4096, WavePrune achieves 3.0953 vs.\ 3.3314 for GPT2-medium and 2.9536 vs.\ 3.1797 for GPT2-large. This indicates that WavePrune is most effective when incorporated into pretraining rather than applied only at inference.

\subsection{Controlled Experiments}\label{sec:controlled_experiments}

\textbf{Where does the improvement come from?} We hypothesize that WavePrune improves length extrapolation by resolving two distinct failure modes of standard RoPE, one affecting short-range dependencies and one affecting long-range dependencies. We state each hypothesis below and test it with a controlled experiment on GPT2-base, using the same training length of 2048 and validation length of 4096 as in Section~\ref{sec:pretraining_experiments}.

\textbf{Hypothesis 1: Attention to nearby tokens is diluted by ``ghost'' sub-diagonals, and the dilution grows with context length.} A longer context contains more ``ghost'' sub-diagonals (Section~\ref{sec:periodicity_distracts}), so a head attending to nearby tokens loses more attention mass at the extrapolated length. WavePrune removes them, so the attention mass should be better preserved.

\textit{Tetranacci experiment.} To isolate short-range dependencies, we train GPT2-base on the Tetranacci sequence, defined by $a_n \coloneqq (a_{n-1} + a_{n-2} + a_{n-3} + a_{n-4}) \bmod 1000$ from random initial values $a_0, a_1, a_2, a_3$ (construction details in Appendix~\ref{app:tetranacci_construction}); each term depends only on the four preceding terms. Figure~\ref{fig:toy_experiments} shows that WavePrune significantly reduces the length-extrapolated validation loss. Furthermore, we find that WavePrune assigns higher attention mass to nearby tokens, and when extrapolating to longer sequences, this attention mass declines less (82.8\% retained vs.\ 76.7\% for the baseline; Table~\ref{tab:att_scores_tetranacci}), which confirms Hypothesis 1.

\textbf{Hypothesis 2: Attention to distant tokens is perturbed by high-frequency channels.} Long-range retrieval relies on content matching in low-frequency channels \citep{barbero2025round}. However, under standard RoPE, the contributions of high-frequency channels highly depend on the exact distance to the target and do not vanish after completing many rotations, causing attention logits to fluctuate chaotically with distance (Figure~\ref{fig:long_term_decay}). WavePrune removes the contribution beyond each channel's first rotation period, offering a stable attention pattern that should support more robust long-range modeling \citep{xu2024base}.

\textit{NIAH experiment.} To isolate long-range dependencies, we train GPT2-base on a needle-in-a-haystack (NIAH) task \citep{needle-in-haystack}, in which a seven-digit ``magic number'' is randomly inserted within a long passage of irrelevant text and must be recalled at the end (construction details in Appendix~\ref{app:niah_construction}). WavePrune substantially reduces the length-extrapolated validation loss (Figure~\ref{fig:toy_experiments}). Table~\ref{tab:att_scores_niah} reports the attention that each retrieval head assigns to the magic number across validation samples: while baseline heads attend to it on some samples but assign near-zero attention on others, WavePrune heads attend to it much more consistently, which confirms Hypothesis 2.

\begin{figure}[t]
    \centering
    \includegraphics[width=0.72\textwidth]{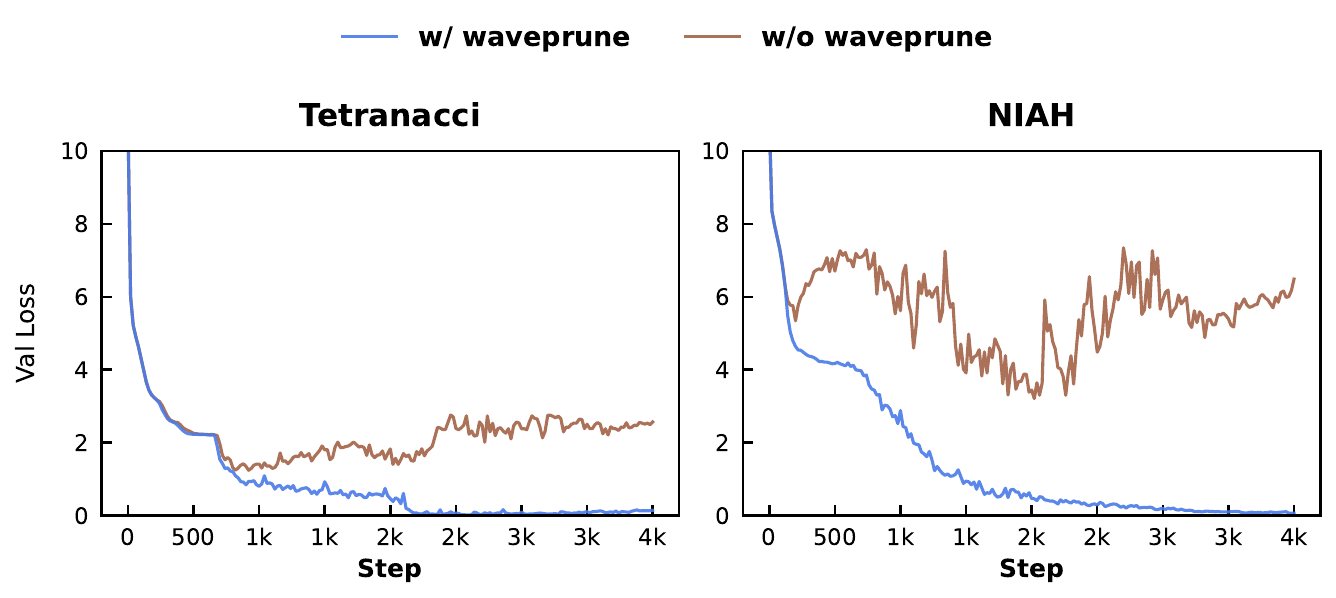}
    \caption{Length-extrapolated validation loss curves for the Tetranacci (left, for short-range dependencies) and NIAH (right, for long-range dependencies) tasks. Training set has length 2048 and validation set has length 4096. WavePrune significantly improves length extrapolation in both tasks.}
    \label{fig:toy_experiments}
\end{figure}

\subsection{System Efficiency}\label{sec:system_efficiency}

\textbf{Experimental setup.} We benchmark our WavePrune kernel (Section~\ref{sec:system_design}) against FlashAttention-2 (FA2) \citep{dao2024flashattention} in terms of inference latency. Unless specified otherwise, we use \texttt{bf16} precision, a batch size of 32, Grouped-Query Attention (GQA) with 32 query heads and 8 KV heads, causal attention, head dimension $d_h = 128$, and sequence lengths $n \in \{2^{13}, 2^{14}, 2^{15}, 2^{16}, 2^{17}\}$. All measurements are collected on an NVIDIA A100 GPU as the median of 10 (prefill) / 50 (decode) timed iterations following 3 / 10 warmup runs, respectively. We report the speedup as the ratio of FA2 latency to WavePrune latency; values greater than $1$ indicate WavePrune is faster.


\textbf{Kernel speed.} We compare FA2 and WavePrune across prefill and decode phases. As illustrated in Table~\ref{tab:system_efficiency_full}, WavePrune consistently outperforms FA2 in terms of latency. This efficiency gain stems directly from our design: skipping QK multiplications decreases FLOPs crucial for prefilling, while bypassing key-cache loads reduces memory traffic crucial for decoding.

\textbf{System ablations.} In Table~\ref{tab:system_efficiency_ablation}, we ablate each kernel optimization: suffix-only QK GEMM, suffix-only K loading, and split points specialization (Appendix~\ref{app:kernel_implementation}). The results demonstrate that removing any of these optimizations increases latency, confirming that each is necessary.

\subsection{Ablation Studies}\label{sec:ablation_studies}

\textbf{Ablation on pruning thresholds and approximation.} 
Recall that all experiments in this section are based on FlashWavePrune, which approximates WavePrune with its group-wise approximation.
To see how this approximation affects the performance of WavePrune, we conduct experiments with Qwen3-8B on the HELMET benchmark at a 32K context length.
Figure~\ref{fig:ablation_alpha} shows four observations: (i) the gap between FlashWavePrune and the exact version is marginal for moderately large $\alpha$; (ii) aggressive pruning at small $\alpha$ significantly degrades performance; (iii) excessively large $\alpha$ brings no further gains ($\alpha\to\infty$ corresponds to no pruning); and (iv) FlashWavePrune peaks at $\alpha=2$, where its gap to exact channel-wise pruning is smallest.

\textbf{Ablation on the number of pruned channels.} To isolate the effect of pruning different RoPE channels, we apply FlashWavePrune to only the first $n \in \{0,8,16,\ldots\}$ highest-frequency channels. Since WavePrune yields its largest gains on In-Context Learning (ICL) (Table~\ref{tab:main_results}), we use the HELMET ICL subset for a clearer signal. As shown in Figure~\ref{fig:ablation_pruned_chans}, performance generally improves as more channels are pruned, with notable jumps at $n=8,16,24$, indicating that WavePrune's benefits extend across a broad range of high- and mid-frequency channels. Beyond $n=40$, where wavelengths exceed the context length (Table~\ref{tab:channel_wavelengths}), pruning offers no further benefit.

\section{Related Work}\label{sec:related}
\begin{figure}[t]
\centering
\begin{subfigure}{0.48\textwidth}
\centering
\includegraphics[width=\textwidth]{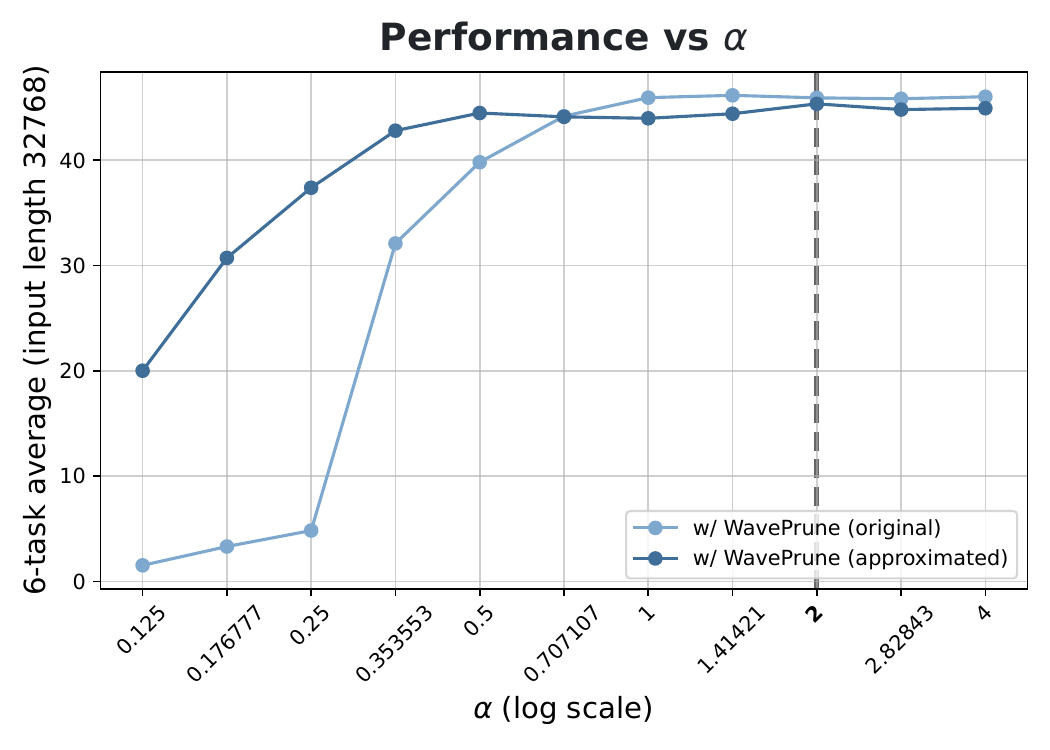}
\caption{Ablation on the pruning coefficient $\alpha$.}
\label{fig:ablation_alpha}
\end{subfigure}
\begin{subfigure}{0.48\textwidth}
\centering
\includegraphics[width=\textwidth]{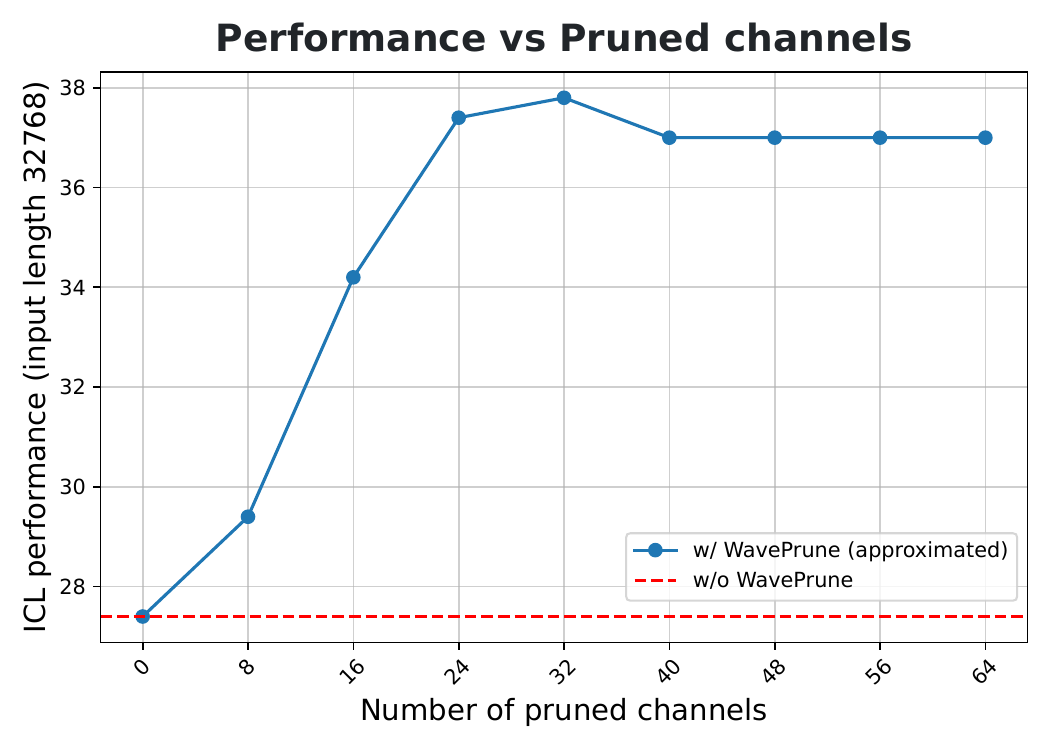}
\caption{Ablation on the number of pruned channels.}
\label{fig:ablation_pruned_chans}
\end{subfigure}
\caption{Ablation studies of WavePrune on HELMET using Qwen3-8B at 32K context length. \textbf{(a)} Effect of the pruning coefficient $\alpha$ and comparison between the exact and approximated versions of WavePrune. \textbf{(b)} Effect of pruning different numbers of channels on ICL performance.}
\label{fig:ablations}
\end{figure}

\textbf{Understanding Rotary Position Embeddings.} RoPE~\citep{su2024roformer} encodes relative positions by rotating queries and keys at different frequencies along the head dimension. Mechanistic studies reveal that models partition these frequencies for distinct purposes. \citet{barbero2025round} show that low frequencies retrieve semantic context, whereas high frequencies produce position-based attention patterns with massive activation values~\citep{jin2025massive}. These patterns support various specialized attention mechanisms, such as retrieval heads~\citep{wu2025retrieval} and previous-token heads~\citep{elhage2021mathematical,olsson2022context,yedidia2023positional}. \citet{cheng2026demystifying} further identify ``slash-dominant heads'' (SDHs), which leverage RoPE's high-frequency channels to create strong sub-diagonal lines at specific lags $\Delta$ in the attention map. They categorize these lags as small-$\Delta$ and large-$\Delta$, the latter of which remain largely unexamined and are implicitly presumed necessary. \citet{yang2026why} likewise observe these periodically repeated attention peaks; \citet{du2026ropedistinguishespositionstokens} prove that such patterns are an inherent consequence of RoPE and can induce positional ambiguity and inversion. Building on these findings, WavePrune prunes such distractions.

\textbf{Rotary Position Embeddings and Length Extrapolation.} To enhance length extrapolation \citep{press2022train}, one line of work modifies RoPE frequencies. The central motivation is that RoPE produces out-of-distribution rotation angles under extended contexts, particularly for low-frequency channels that never complete a full $2\pi$ rotation within the training length \citep{chen2023extendingcontextwindowlarge}. Accordingly, \citet{peng2024yarn,ding2024longrope,wang-etal-2024-resonance,tian2026mrrope,li2026copeclippedropescalable} rescale RoPE frequencies to keep rotation angles in-distribution while preserving high-frequency channels for position sensitivity. A second line of work \citep{sun-etal-2023-length,xiao2024infllm,jin2024llm,han-etal-2024-lm} instead modifies attention mechanics, e.g., via position-dependent query scaling or altered attention masks. Unlike these methods, WavePrune's attention mechanism is tied to RoPE: it eliminates RoPE's periodicity without altering the frequency spectrum, and can be combined with frequency-rescaling methods to further improve length extrapolation.

\textbf{Efficient Attention Mechanisms.} Efficient attention mechanisms reduce the $\mathcal{O}(N^2D)$ cost along sequence length $N$ and head dimension $D$. Along the sequence dimension, some approaches \citep{child2019generatinglongsequencessparse,beltagy2020longformerlongdocumenttransformer,NEURIPS2020_c8512d14,ding2023longnetscalingtransformers1000000000,xiao2024efficient} design fixed sparse patterns to limit the evaluated token pairs, while others \citep{zhang2023ho,jiang2024minference,yuan-etal-2025-native} dynamically select tokens at inference. Recent works prune the sparsely utilized head dimension \citep{barbero2025round,wang2026adarope}. FASA \citep{wang2026fasa} uses a small subset of dominant RoPE frequencies to guide token selection. ThinK \citep{xu2025think} prunes the key cache in a query-dependent manner, while LeanK \citep{zhang-etal-2025-leank} learns static channel sparsity. WavePrune unifies sequence- and head-wise pruning by tailoring local attention patterns to each RoPE frequency, opening an underexplored design space for attention mechanisms.

\textbf{Concurrent Work.} Concurrently, ATFlash \citep{hayashi2026atflashperropewavelengthattentionwindows} also restricts each RoPE channel to a local window proportional to its wavelength. Both works apply this channel-wise truncation training-free to pretrained models, and use head-dimension grouping to keep the QK product on Tensor Cores. Our work differs in three aspects: (i) we find that a single rotation period per channel often suffices, whereas their default window spans two. They keep the second period because it adds little computation, assuming wider windows only improve quality; we instead find quality peaks near one period (Figure~\ref{fig:ablation_alpha}); (ii) we identify the ``ghost'' sub-diagonals induced by RoPE's periodicity (\Cref{thm:periodicity}) as the concrete motivation for pruning, whereas they describe the pruned terms only as ``noise''; and (iii) we also apply WavePrune in pretraining, showing it improves length extrapolation.


\section{Conclusion and Future Work}\label{sec:conclusion}

In this work, we demonstrate that RoPE's periodicity makes attention leak onto ``ghost'' sub-diagonals, which causes distractions at long contexts. WavePrune removes it by restricting each channel to its first rotation period. Without any training, this simple change improves long context performance for most LLMs we evaluate and boosts length extrapolation in pretraining.
Its sparsity pattern also lets our CUDA kernel, FlashWavePrune, deliver practical speedups. Overall, our results suggest that one period is often enough for RoPE.

We outline the following limitations and directions for future work:
\begin{itemize}[leftmargin=2.4em,topsep=3.2pt,partopsep=3.2pt,parsep=3.2pt,itemsep=0pt]
\item \textbf{Performance variations.} WavePrune's training-free performance varies across different models. Future work can investigate the underlying causes through mechanistic analysis, and explore whether advanced fine-tuning strategies can bridge the performance gap.
\item \textbf{Head-specific thresholds.} WavePrune uses the same per-channel pruning thresholds for all attention heads, but different heads may play different roles and future work can design or even learn different thresholds for different heads.
\item \textbf{Generalizing the fine-grained attention mask.} WavePrune's fine-grained attention mask over both the sequence and head dimensions is based on RoPE. Future work can generalize this masking principle to other positional encoding methods and model architectures \citep{wang2026adarope,thinkingmachines2026inkling,lin2025forgetting}, which can potentially improve long-context processing and accelerate inference.
\end{itemize}
\newif\ifwaveack
\ificlrpreprint
\waveacktrue
\else
\ificlrfinal
\waveacktrue
\else
\waveackfalse
\fi
\fi
\ifwaveack
\section*{Acknowledgments}
We thank Haofeng Huang for insightful discussions on this project.
\fi

\ificlrfinal\else\ificlrpreprint\else
\section*{AI Use Statement}
We used generative AI tools to implement auxiliary code (plotting, data processing) and to polish writing and assist in retrieving related work. Our core CUDA kernels, training code, and all other tasks requiring disclosure under the ICLR 2027 AI Policy for Authors were performed without AI. All AI-assisted code, wording, and references were manually reviewed and verified by the authors, who take full responsibility for the final content.

\section*{Reproducibility Statement}
Complete proofs and assumptions for our theoretical results are in Appendix~\ref{app:theory}. Our synthetic dataset construction, pretraining hyperparameters, kernel implementation, and evaluation protocol are detailed in Appendices~\ref{app:toy_experiments}, \ref{app:pretraining_experiments}, \ref{app:kernel_implementation}, and \ref{app:pretrained_integration}, respectively. We will release our code, including the CUDA kernels and training and evaluation scripts, upon publication.

\fi\fi

\bibliography{iclr2027_conference}
\bibliographystyle{iclr2027_conference}

\newpage
\appendix
\section{\texorpdfstring{A Direct Proof of \Cref{thm:periodicity}}{A Direct Proof of Theorem 1}}\label{app:direct_proof}

In this section, we give a short and self-contained proof of \Cref{thm:periodicity}, independent of the machinery developed in Appendix~\ref{app:theory}. We first restate the theorem.

\periodicitythm*

\textbf{Proof.} Let $m_c^Q, m_c^K$ and $\omega_c^Q, \omega_c^K$ denote the magnitudes and angles of channel $c$ ($0 \le c \le \frac {d_h} {2}-1$) of $q$ and $k$, and let $\omega_c \coloneqq \omega_c^Q - \omega_c^K$ be their angle difference. Since RoPE rotates the channel $c$ of $q_i$ by $i \, f_c$ and the channel $c$ of $k_j$ by $j \, f_c$, 
for all $i > \Delta$,
\begin{equation}\label{eq:direct_logit_lag}
L_{i,i-\Delta} = \frac 1 {\sqrt {d_h}} \sum_{c=0}^{d_h/2-1} m_c^Q m_c^K \cos\left(\Delta \, f_c + \omega_c\right).
\end{equation}
We denote by $L_{\Delta} := L_{i,i-\Delta}$ the attention logit at lag $\Delta$ without explicitly writing the position index $i$ since it depends only on $\Delta$.
Then $L^\star = \sup_{\Delta \ge 0} L_{\Delta}$, which is finite since
\begin{equation}
|L_{\Delta}| \le \frac 1 {\sqrt {d_h}} \sum_c m_c^Q m_c^K \eqqcolon M.
\end{equation}

First, we derive a simple upper bound to show that lags with close rotation phases have close logits. Let $\text{wrap}(x) \coloneqq x + 2k\pi$, where $k \in \mathbb{Z}$ is the unique integer for which $x + 2k\pi \in [-\pi,\pi)$. Because $\cos$ is $2\pi$-periodic and $1$-Lipschitz, shifting the lag $\Delta$ by $P \in \mathbb{Z}$ perturbs $\cos\left(\Delta \, f_c + \omega_c\right)$ by at most $|\text{wrap}(P f_c)|$. Hence
\begin{equation}\label{eq:direct_lipschitz}
\bigg|\,L_{\Delta+P} - L_{\Delta}\,\bigg| \le
\frac 1 {\sqrt {d_h}} \sum_{c=0}^{d_h/2-1} m_c^Q m_c^K \,\big|\text{wrap}(P \, f_c)\big|
\eqqcolon r(P),
\quad \forall \Delta, P \in \mathbb{N}.
\end{equation}
We call $P$ an $\epsilon$-\textit{almost period} if $r(P) \le \epsilon$, so that such a $P$ shifts every logit by at most $\epsilon$.

Now we fix $\epsilon > 0$ and show that 
there are actually infinitely many $\epsilon$-almost periods.
Let $x_j$ be a vector that collects the per-channel rotation phases of position $j$, where $x_{j,c} \coloneqq \text{wrap}(j \, f_c)$.
Note that $x_j \in V \coloneqq [-\pi,\pi)^{d_h/2}$ for every $j$. Fix any integer $N \ge \frac {2\pi M} {\epsilon}$ and split each coordinate of $V$ into $N$ intervals of length $\frac {2\pi} N$, which means it partitions $V$ into $N^{d_h/2}$ cells (hyperrectangles). Since there are finitely many cells but infinitely many positions, the pigeonhole principle implies that there is a single cell that contains $x_{j_1}, x_{j_2}, \ldots$ for some infinite increasing sequence $j_1 < j_2 < \cdots$.

Then for any $s \ge 2$, we have $\text{wrap}(x_{j_s,c} - x_{j_1,c}) = \text{wrap}((j_s - j_1) f_c)$. Since $x_{j_s}$ and $x_{j_1}$ lie in the same cell, they differ by less than $\frac {2\pi} N$ in every coordinate, so $|\text{wrap}((j_s - j_1) f_c)| \le |x_{j_s,c} - x_{j_1,c}|$. Therefore,
\begin{equation}
r(j_s - j_1) \le \frac 1 {\sqrt {d_h}} \sum_c m_c^Q m_c^K \,|\,x_{j_s,c} - x_{j_1,c}\,| \le \frac {2\pi M} {N} \le \epsilon.
\end{equation}
This means that $\{j_s - j_1\}_{s \ge 2}$ is an infinite set of $\epsilon$-almost periods.

Finally, by the definition that $L^\star$ is a supremum,
for any tolerance $\tau > 0$,
there exists a lag $\Delta_0 \ge 0$ such that $L_{\Delta_0} \ge L^\star - \frac {\tau} 2$. Applying Step 2 with $\epsilon \coloneqq \frac {\tau} 2$ and then Equation~\ref{eq:direct_lipschitz}, every $\frac {\tau} 2$-almost period $P$ satisfies $L_{\Delta_0 + P} \ge L_{\Delta_0} - \frac {\tau} 2 \ge L^\star - \tau$. As there are infinitely many such $P$, the lags $\Delta \coloneqq \Delta_0 + P$ form an infinite set, and Equation~\ref{eq:direct_logit_lag} turns the above into\vspace{0.3em}
\begin{equation}
L_{i,i-\Delta} \ge L^\star - \tau, \quad \forall i > \Delta,
\end{equation}
which completes the proof.

\textbf{Remark.} In the above proof, we partition all $d_h/2$ channels together, so the periods it yields can be as large as $N^{d_h/2}$, which is vacuous at practical head dimensions. In Appendix~\ref{app:theory}, we improve this analysis by taking into account the sparsity of the channels. Also, we allow the queries and keys to be random rather than identical across positions (Theorem~\hyperref[thm:main]{2}).



\section{RoPE and Attention Periodicity: Theory and Implications}\label{app:theory}

\subsection{Slash-Dominant Attention Pattern}

After training, the model develops various structures in the attention map $A$ to implement diverse functionalities. In this section, we focus on a specific structure known as the \textbf{slash-dominant pattern} \citep{cheng2026demystifying}.

We say that the attention score matrix $A \in \mathbb{R}^{n \times n}$ exhibits a slash-dominant pattern if there exists a constant $\Delta \in \mathbb{N}$ such that $A$ has large values along the $\Delta$-offset sub-diagonal $\{A_{i,i-\Delta}\}_{i=\Delta+1}^n$. Intuitively, this means each token $i (i>\Delta)$ attends strongly to the token at position $i-\Delta$, which enables information to flow in a position-wise ``slash-like" manner. Such patterns have been empirically observed in various LLMs and play a critical role in tasks such as in-context learning (ICL) \citep{cheng2026demystifying,elhage2021mathematical,olsson2022context}.

We formalize the notion of slash-dominant attention patterns as follows:

\textbf{Definition 1 (Slash-Dominant Pattern).}\label{def:slash_dominant_pattern} \textit{For any lag $\Delta \in \mathbb{N}$ and $\epsilon,\eta > 0$, the attention score matrix $A \in \mathbb{R}^{n \times n}$ (or, equivalently, the attention head) is said to be $(\Delta, \epsilon, \eta)$-slash-dominant, if}
\begin{equation}
\sum_{i=\Delta+1}^n \mathbf{1}\left\{A_{i,i-\Delta} \ge 1 - \epsilon\right\} \ge (1-\eta)(n-\Delta)
\end{equation}

\subsection{Main Theoretical Result}

\subsubsection{Assumptions and Definitions}

Before stating our main theorem, we introduce the following assumptions and definitions.

\textbf{Assumption 1 (Symmetric i.i.d. Distributions of Queries and Keys).}\label{assumption:query_key_distribution} \textit{There exist distributions $\mathcal{Q},\mathcal{K}$ over $\mathbb{R}^{d_h}$ such that the queries $q_i$ are i.i.d. samples from $\mathcal{Q}$ and the keys $k_i$ are i.i.d. samples from $\mathcal{K}$. Furthermore, $\mathcal{Q}$ and $\mathcal{K}$ are symmetric distributions centered at their respective means $q$ and $k$ (e.g., multivariate Gaussians); i.e., for any measurable set $S \subseteq \mathbb{R}^{d_h}$, let $\widetilde{S} \coloneqq \{2q - x \mid x \in S\}$ be the reflection of $S$ about $q$, then $\mathrm{Pr}(q_i \in S) = \mathrm{Pr}(q_i \in \widetilde{S})$. The same holds for $\mathcal{K}$ with mean $k$.}
\vspace{0.2em}

\textbf{Definition 2 (Channel Sensitivity Measures and $\boldsymbol{\tau}$-Truncation Sparsity).}\label{def:channel_sensitivity_and_truncation_sparsity} \textit{For $c=0,1,\ldots,d_h/2-1$, let $m^Q_c \coloneqq \sqrt{q_{2c+1}^2 + q_{2c+2}^2}$ and $m^K_c \coloneqq \sqrt{k_{2c+1}^2 + k_{2c+2}^2}$ denote the magnitudes of channel $c$ of the mean vectors $q$ and $k$, respectively. We define the \textbf{sensitivity measure} of channel $c$ as}\vspace{0.2em}
\begin{equation}\label{eq:channel_sensitivity_measure}
t_c \coloneqq \frac {2\pi m^Q_c m^K_c} {\sqrt{d_h}} \quad (c=0,1,\ldots,d_h/2-1)
\vspace{0.2em}
\end{equation}
\textit{For any $\tau > 0$, the \textbf{$\boldsymbol{\tau}$-truncation sparsity} of the sensitivity measures $t$ is defined as}\vspace{0.2em}
\begin{equation}
\sigma_{\tau}(t) \coloneqq \min\left\{\|z\|_0 \mid \|z-t\|_1 \le \tau, \, z \in \mathbb{R}^{d_h/2}\right\}
\vspace{0.2em}
\end{equation}
\subsubsection{Main Theorem}

With Assumption~\hyperref[assumption:query_key_distribution]{1} and Definition~\hyperref[def:channel_sensitivity_and_truncation_sparsity]{2} in place, we are now ready to state our main theorem. Intuitively, it shows that the attention logit matrix $L$ (Equation~\ref{eq:attention_logit_definition}) is approximately periodic\footnote{Here, we use the term ``periodic'' to loosely refer to the \textit{similarity} of attention logits at positions separated by certain lags, rather than the strict mathematical property.} across different position lags (see Figure~\ref{fig:slash_dominant_head_rotation_collisions} for an illustration). Consequently, the $\Delta$-offset sub-diagonal (the intended lag) competes with its periodic repetitions, causing attention weight to distribute across multiple ghost sub-diagonals (sub-diagonals with unexpectedly high attention scores) rather than concentrating on the desired lag $\Delta$.

\newpage
\textbf{Theorem 2 (Periodicity of Attention Logits, Formal Statement).}\label{thm:main} \textit{Let the query and key distributions $\mathcal{Q}, \mathcal{K}$ satisfy Assumption~\hyperref[assumption:query_key_distribution]{1}. For any $\tau > 0$, there exist infinitely many periods $P \in \mathbb{N}^+$ such that}\vspace{0.2em}
\begin{equation}
\forall \Delta \ge 0, \, \forall i > \Delta + P, \quad
\mathrm{Pr}\left(L_{i,i-\Delta} - L_{i,i-\Delta-P} \le 2\tau\right) \ge \frac 1 4
\vspace{0.2em}
\end{equation}
\textit{Let $t^s$ be the sensitivity measures (Definition~\hyperref[def:channel_sensitivity_and_truncation_sparsity]{2}) sorted in descending order. Then, the minimum period $P_{\min}$ satisfying the above condition is upper bounded by a sparsity-dependent quantity:}\vspace{0.2em}
\begin{equation}\label{eq:periodicity_bound}
P_{\min} \le \prod_{i=1}^{\sigma_{\tau}(t)} \left\lceil \frac {\sigma_{\tau}(t) \cdot t_i^s} {\tau} \right\rceil
\vspace{0.2em}
\end{equation}
\textit{If $t_i^s = \mathcal{O}(1)$ for all $i$\footnote{This assumption is reasonable since $X$ is typically normalized by LayerNorm, which means $m^Q_c$ and $m^K_c$ are bounded by constants.} and $\sigma_{\tau}(t) \le S$ for some constant $S \ge 0$, the above bound can be simplified to}\vspace{0.2em}
\begin{equation}
P_{\min} = \mathcal{O}\left(\left(\frac {1} {\tau}\right)^S\right)
\vspace{0.2em}
\end{equation}


\subsubsection{Implication on Attention Scores}

According to the definition of attention score matrix $A_{i,j} \coloneqq \frac{\exp(L_{i,j})}{\sum_{k=1}^{n} \exp(L_{i,k})}$,
which is a softmax of the attention logit matrix, the periodicity of attention logits established in Theorem~\hyperref[thm:main]{2} directly implies that attention is prevented from concentrating exclusively on the desired sub-diagonal. We formalize this implication in the following corollary.

\textbf{Corollary 1 (Upper Bound on Attention Scores).}\label{corollary:attention_score_bound} \textit{Under the same conditions as Theorem~\hyperref[thm:main]{2}, for any $\epsilon \in \left(0,\frac 1 2\right)$, let $\tau_{\epsilon} = \ln\left(\frac {1-\epsilon} {\epsilon}\right)$ be the log-odds associated with $\epsilon$. Then, there exists $P_{\min} \in \mathbb{N}^+$ such that}\vspace{0.2em}
\begin{equation}\label{eq:attention_score_bound}
\forall \Delta \ge 0, \, \forall i > \Delta + P_{\min}, \quad
\mathrm{Pr}\left(A_{i,i-\Delta} \le 1-\epsilon\right) \ge \frac 1 4
\vspace{0.2em}
\end{equation}
where $P_{\min}$ is upper bounded by the same quantity as in Equation~\ref{eq:periodicity_bound} with $\tau = \frac {\tau_{\epsilon}} {2}$.

\textbf{Remark (Fragility of Slash-Dominant Patterns).} Corollary~\hyperref[corollary:attention_score_bound]{1} indicates that for sufficiently small $\epsilon$ and $\eta$ in Definition~\hyperref[def:slash_dominant_pattern]{1}, any slash-dominant pattern will eventually collapse as the sequence length increases. Note that this does not contradict the empirical observations in \citet{cheng2026demystifying}: their study identifies slash-dominant heads using loose thresholds (e.g., $0.1$ instead of $1-\epsilon$) and evaluates them on relatively short context lengths (up to $6000$ tokens).

\subsection{Empirical Evidence and Analysis}

\textbf{Sparsity of channel utilization and the periodicity issue.} As illustrated in Theorem~\hyperref[thm:main]{2}, the sparsity of channel utilization is a dominant factor determining the period of attention logits. Extensive prior work has observed sparsity in real models \citep{xu2025think,barbero2025round,chiang-yogatama-2025-rotary,wang2026fasa,wang2026adarope}, and we also empirically verify this. For instance, as shown in Figure~\ref{fig:qwen3_layer0_head20_channel_visualization_original}, only 7 out of 64 channels are actively utilized in Qwen3-8B's layer 0 head 20 (a slash-dominant head with $\Delta=1$; as evidenced by the attention score statistics in Table~\ref{tab:qwen3_attention_percentiles}). Another example is shown in Figure~\ref{fig:qwen3_layer0_head6_channel_visualization_original}. This can indeed be understood as a low $\tau$-truncation sparsity $\sigma_{\tau}(t)$ (Theorem~\hyperref[thm:main]{2}), which often leads to a small period $P_{\min}$ by Theorem~\hyperref[thm:main]{2}.\footnote{While channel sparsity is empirically observed, the mathematically defined $\sigma_\tau(t)$ may not be always small because inactive channels can still have non-negligible sensitivity values (note that two vectors with large magnitudes can yield a near-zero dot product). This gap primarily arises from omitting the rotation frequencies $f_i$ in Definition~\hyperref[def:channel_sensitivity_and_truncation_sparsity]{2} for technical tractability; readers may intuitively think of these inactive channels as having zero query/key magnitudes. We leave tighter analysis to future work.}

As evidenced in Figure~\ref{fig:periodicity_orig_head20}, Figure~\ref{fig:qwen3_original_attn_scores} and Figure~\ref{fig:ministral_original_attn_scores}, such periodicity causes attention logits of various models to spike at multiple ghost sub-diagonals. We argue that this behavior is detrimental to model performance on non-periodic tasks (including almost all NLP tasks), since it prevents the model from losslessly passing signals along a single desired lag.

\textbf{Eliminating ghost sub-diagonals through WavePrune.} The periodicity issue identified in Theorem~\hyperref[thm:main]{2} can be resolved via WavePrune by truncating attention beyond each channel's wavelength. As can be seen in Figures \ref{fig:qwen3_layer0_attn_scores_comparison} and \ref{fig:ministral_layer0_attn_scores_comparison}, the ghost sub-diagonals are effectively eliminated after applying WavePrune. This improvement is particularly evident in the channel-wise decomposition of attention logits (Figure~\ref{fig:qwen3_layer0_head20_channel_visualization_original} and Figure~\ref{fig:qwen3_layer0_head20_channel_visualization_waveprune}), where periodic patterns are suppressed across all channels.

\textbf{WavePrune breaks Theorem~\hyperref[thm:main]{2}.} Beyond empirical observations, we also theoretically show that WavePrune breaks the attention logits periodicity established in Theorem~\hyperref[thm:main]{2}. Consider a minimal 2D example comprising a single channel with rotation frequency $\theta \in (0, 2\pi)$ and wavelength $w \coloneqq 2\pi/\theta$. Suppose the query and key are unit vectors $q, k \in \mathbb{R}^2$ with an initial relative angle $\omega = -\theta$. Ignoring the $1/\sqrt{d_h}$ scaling, the attention logit under WavePrune (Equation~\ref{eq:waveprune}) at lag $\Delta$ is given by $L_{\Delta} = \cos(\Delta\theta + \omega) \cdot \mathbf{1}_{\Delta \le w} = \cos(\theta(\Delta - 1)) \cdot \mathbf{1}_{\Delta \le w}$. At $\Delta = 1$, the RoPE rotation exactly cancels the initial relative angle, yielding a maximum logit $L_1 = \cos(0) = 1$. For $1 < \Delta \le w$, $\theta(\Delta - 1)$ cannot reach any non-zero multiple of $2\pi$, so $L_{\Delta} < 1$; for $\Delta > w$, the indicator evaluates to zero, yielding a constant logit $L_{\Delta} = 0$. By permanently forcing the logit to zero beyond the wavelength, WavePrune prevents any periodic spikes, invalidating Theorem~\hyperref[thm:main]{2}.

Finally, note that the periodicity issue is not exclusive to slash-dominant heads but is observable across numerous other types of attention heads (Figure \ref{fig:qwen3_layer32_attn_scores_comparison} and \ref{fig:ministral_layer1_attn_scores_comparison}), as Theorem~\hyperref[thm:main]{2} relies on no assumptions regarding the attention pattern. Nonetheless, since the elimination of ghost sub-diagonals is particularly beneficial for slash-dominant heads, we use them as a representative example.

\begin{table}[!t]
    \centering
    \caption{Percentile-based attention score statistics for Qwen3-8B layer 0 head 20 across different position lags ($\Delta$). We compute attention scores for each of the 100 PG19 samples, extract the percentiles for each sub-diagonal, and then average these percentiles. The $\Delta=1$ sub-diagonal exhibits significantly higher attention scores, confirming a slash-dominant pattern with $\Delta=1$.}
    \label{tab:qwen3_attention_percentiles}
    {\fontsize{9}{10.8}\selectfont
    \begin{tabular}{c|ccccc}
        \toprule
        \textbf{Percentile} & $\boldsymbol{\Delta=0}$ & $\boldsymbol{\Delta=1}$ & $\boldsymbol{\Delta=2}$ & $\boldsymbol{\Delta=3}$ & $\boldsymbol{\Delta=4}$ \rule{0pt}{2.2ex} \\[0.3ex]
        \midrule\\[-1.35em]
        75th & 0.0126 & \textbf{0.0611} & 0.0178 & 0.0065 & 0.0037 \rule{0pt}{2.3ex} \\[0ex]
        90th & 0.0255 & \textbf{0.1420} & 0.0399 & 0.0182 & 0.0112 \rule{0pt}{2.3ex} \\[0ex]
        95th & 0.0392 & \textbf{0.1573} & 0.0634 & 0.0295 & 0.0185 \rule{0pt}{2.3ex} \\[0ex]
        \bottomrule
    \end{tabular}}
\end{table}

\subsection{Proof of Theorem 2}

\textbf{Proof sketch.} The proof has two main steps. In the first step, we treat the deterministic case, where every query and key equals its mean. Note that RoPE turns each attention logit into a weighted sum of per-channel cosines of the relative position. Because cosine is Lipschitz, the gap between two logits is bounded by a weighted $\ell_1$ distance between the two key positions' rotation phases across channels. Since these phases all lie in a bounded box $(-\pi,\pi]^{d_h/2}$, we can cover the box with finitely many small cells of diameter $2\tau$; by pigeonhole, two key positions must ultimately fall into the same cell, and their offset yields a period \textbf{due to RoPE's periodicity}. Hence $P_{\min}$ is bounded by the minimum number of cells required to cover the box. By absorbing low-sensitivity channels into the truncation budget $\tau$, this number depends only on the few dominant channels. In the second step, where queries and keys are random, we note that rotation and subtraction preserve symmetry, and a simple reflection argument shows that the logit difference falls below its mean with probability at least $\frac 1 4$.

\subsubsection{Step 1: Periodicity under Point Mass Distributions}

\textbf{Lemma 1 (Periodicity under Point Mass Distributions).}\label{lemma:periodicity_deterministic} \textit{Let $\mathcal{Q} = \delta_q$ and $\mathcal{K} = \delta_k$ be point mass distributions concentrated at $q$ and $k$, respectively. Then, there exist infinitely many periods $P \in \mathbb{N}^+$ such that}\vspace{0.2em}
\begin{equation}\label{eq:periodicity_deterministic}
\forall \Delta \ge 0, \, \forall i > \Delta + P, \quad
L_{i,i-\Delta} - L_{i,i-\Delta-P} \le 2\tau.
\vspace{0.2em}
\end{equation}
\textit{Let $t^s$ be the sensitivity measures (Definition~\hyperref[def:channel_sensitivity_and_truncation_sparsity]{2}) sorted in descending order. Then, the minimum period $P_{\min}$ satisfying the above condition is upper bounded by a sparsity-dependent quantity:}\vspace{0.2em}
\begin{equation}\label{eq:periodicity_bound_deterministic}
P_{\min} \le \prod_{i=1}^{\sigma_{\tau}(t)} \left\lceil \frac {\sigma_{\tau}(t) \cdot t_i^s} {\tau} \right\rceil.
\vspace{0.2em}
\end{equation}
\textbf{Proof.} Let $m_c^Q, m_c^K, \omega_c^Q, \omega_c^K$ denote the magnitudes and angles of channel $c$ of $q$ and $k$ for $0 \le c \le \frac {d_h} {2}-1$, respectively. i.e.,\vspace{0.2em}
\begin{equation}
\begin{aligned}
&q_{2c+1} = m_c^Q \cos(\omega_c^Q), \ \; q_{2c+2} = m_c^Q \sin(\omega_c^Q) \\
&k_{2c+1} = m_c^K \cos(\omega_c^K), \ k_{2c+2} = m_c^K \sin(\omega_c^K).
\end{aligned}
\vspace{0.2em}
\end{equation}
Denote $\omega_c \coloneqq \omega_c^Q - \omega_c^K$ as the angle difference between channel $c$ of $q$ and $k$, then the attention logit $L_{i,j}(i \ge j)$ can be expressed as:\vspace{0.2em}
\begin{equation}
\begin{aligned}
L_{i,j}
&=
\frac 1 {\sqrt {d_h}} \varphi_i(q_i)^{\top} \varphi_j(k_j)
\\ &=
\frac 1 {\sqrt {d_h}} \sum_{c=0}^{d_h/2-1} m_c^Q m_c^K \cos\left((i-j)\,f_c + \omega_c\right).
\end{aligned}
\vspace{0.2em}
\end{equation}
Let $\text{wrap}(x) \coloneqq x + 2k\pi$, where $k \in \mathbb{Z}$ is the unique integer for which $x + 2k\pi \in [-\pi,\pi)$. Then, the attention logit difference $L_{i,j_1} - L_{i,j_2} (i \ge j_1 > j_2)$ can be bounded as follows:\vspace{0.2em}
\begin{equation}
\begin{aligned}
L_{i,j_1} - L_{i,j_2}
&=
\frac 1 {\sqrt {d_h}} \sum_{c=0}^{d_h/2-1} m_c^Q m_c^K \bigg(\!\cos\left((i-j_1)\,f_c + \omega_c\right) - \cos\left((i-j_2)\,f_c + \omega_c\right)\!\bigg)
\\ &\le
\frac 1 {\sqrt {d_h}} \sum_{c=0}^{d_h/2-1} m_c^Q m_c^K \,\bigg|\cos\left((i-j_1)\,f_c + \omega_c\right) - \cos\left((i-j_2)\,f_c + \omega_c\right)\,\!\bigg|
\\ &\le
\frac 1 {\sqrt {d_h}} \sum_{c=0}^{d_h/2-1} m_c^Q m_c^K \,\bigg|\,\text{wrap}\left(j_1 \, f_c - j_2 \, f_c\right) \,\bigg|,
\end{aligned}
\vspace{0.2em}
\end{equation}
where the last inequality holds because $\cos$ is $2\pi$-periodic and $1$-Lipschitz.

For any $1 \le j \le i$, let $x_j \in [-\pi,\pi)^{d_h/2}$ be the vector defined as $x_{j,c} \coloneqq \text{wrap}(j\, f_c)$, then\vspace{0.2em}
\begin{equation}
\begin{aligned}
L_{i,j_1} - L_{i,j_2}
&\le
\frac 1 {\sqrt {d_h}} \sum_{c=0}^{d_h/2-1} m_c^Q m_c^K \,\bigg|\,\text{wrap}(x_{j_1,c} - x_{j_2,c})\,\bigg|
\\ &\le
\frac 1 {\sqrt {d_h}} \sum_{c=0}^{d_h/2-1} m_c^Q m_c^K \,\bigg|\,x_{j_1,c} - x_{j_2,c}\,\bigg|.
\end{aligned}
\vspace{0.2em}
\end{equation}
We define the metrics $\rho_c, \rho$ on $V \coloneqq [-\pi,\pi)^{d_h/2}$ as follows:\vspace{0.2em}
\begin{equation}\label{eq:rho_c_definition}
\begin{aligned}
\rho_c(x,y) \coloneqq \frac 1 {\sqrt {d_h}} \sum_{c=0}^{d_h/2-1} m_c^Q m_c^K \,\bigg|\,\text{wrap}(x_c - y_c)\,\bigg|, \\
\end{aligned}
\end{equation}
\begin{equation}
\rho(x,y) \coloneqq \frac 1 {\sqrt {d_h}} \sum_{c=0}^{d_h/2-1} m_c^Q m_c^K \,\bigg|\,x_c - y_c\,\bigg|.
\vspace{0.2em}
\end{equation}
Then\vspace{0.2em}
\begin{equation}
L_{i,j_1} - L_{i,j_2} \le \rho_c(x_{j_1}, x_{j_2}) \le \rho(x_{j_1}, x_{j_2}).
\vspace{0.2em}
\end{equation}

Suppose $V$ can be covered by $M$ sets $C_1, C_2, \ldots, C_M$ with $\rho$-diameter at most $2\tau$, i.e.,\vspace{0.2em}
\begin{equation}
\bigcup_{m=1}^M C_m = V, \quad \forall m, \, \forall x,y \in C_m, \, \rho(x,y) \le 2\tau.
\vspace{0.2em}
\end{equation}
Then, for any $\Delta \ge 0$ and $i > \Delta + M$, there must exist $i-\Delta-M \le j_2 < j_1 \le i-\Delta$ such that $x_{j_1}, x_{j_2} \in C_m$ for some $m$ because of the pigeonhole principle. Consequently, we have $\rho_c(x_{j_1}, x_{j_2}) \le \rho(x_{j_1}, x_{j_2}) \le 2\tau$. Note that $\rho_c(x_{j_1}, x_{j_2})$ only depends on the position difference $j_1-j_2$ rather than the specific values of $j_1$ and $j_2$, Thus, we can set $P \coloneqq j_1 - j_2$ and yield the desired periodicity (Equation~\ref{eq:periodicity_deterministic}):\vspace{0.2em}
\begin{equation}
L_{i,i-\Delta} - L_{i,i-\Delta-P} \le \rho_c(x_{i-\Delta}, x_{i-\Delta-P}) = \rho_c(x_{j_1}, x_{j_2}) \le 2\tau,
\vspace{0.2em}
\end{equation}
which holds for all $i > \Delta + P$ since $L_{i,i-\Delta} - L_{i,i-\Delta-P}$ is independent of $i$.

If $\rho_c(x_{i-\Delta}, x_{i-\Delta-P}) = 0$, then $\rho_c(x_{i-\Delta}, x_{i-\Delta-kP}) = 0$ for all $k \in \mathbb{N}^+$, which means any integer multiple of $P$ can serve as a period. Otherwise, we can set $\tau' \coloneqq \frac 1 3 \rho_c(x_{i-\Delta}, x_{i-\Delta-P}) > 0$ as the new threshold and repeat the above argument to find another period $P'$ such that $\forall \Delta \ge 0, \, \forall i > \Delta + P', \, L_{i,i-\Delta} - L_{i,i-\Delta-P'} \le 2\tau' < 2\tau$. By iterating this process, we can find infinitely many periods that satisfy Equation~\ref{eq:periodicity_deterministic}.

Now, it suffices to show that the covering number $M$ is upper bounded by the quantity in Equation~\ref{eq:periodicity_bound_deterministic}. Note that the metric space $\left(V, \rho\right)$ is isometric to $V_0 \coloneqq [0, t_0] \times [0, t_1] \times \ldots \times [0, t_{d_h/2-1}]$ equipped with the $\ell_1$ metric, where $t_c$ is the sensitivity measure defined in Equation~\ref{eq:channel_sensitivity_measure}. This reduces the problem to covering $V_0$ with sets of $\ell_1$-diameter at most $2\tau$.

We use a simple grid-based approach to cover $V_0$. Specifically, let $g_0, g_1, \ldots, g_{d_h/2-1} > 0$ denote the grid sizes for each dimension, then we partition $V_0$ into axis-aligned hyperrectangles by placing hyperplanes at intervals of $g_c$, ensuring each hyperrectangle has $\ell_1$-diameter at most $\sum\nolimits_c g_c \le 2\tau$. The total number of hyperrectangles is then
\begin{equation}
M = \prod_{c=0}^{d_h/2-1} \left\lceil \frac {t_c} {g_c} \right\rceil,
\vspace{0.2em}
\end{equation}
which leads to the following optimization problem for minimizing $M$:\vspace{0.2em}
\begin{equation}
\begin{aligned}
\min &\quad \prod_{c=0}^{d_h/2-1} \left\lceil \frac {t_c} {g_c} \right\rceil \\
\text{s.t.} &\quad \sum_{c=0}^{d_h/2-1} g_c \le 2\tau \\
&\quad g_c > 0 \quad (c=0,1,\ldots,d_h/2-1).
\end{aligned}
\vspace{0.2em}
\end{equation}
Denote $\kappa \coloneqq \sigma_{\tau}(t)$ as the $\tau$-truncation sparsity of $t$ (Definition~\hyperref[def:channel_sensitivity_and_truncation_sparsity]{2}), and $T \coloneqq \{c_1,c_2,\ldots,c_{\kappa}\}$ as the indices of the top $\kappa$ sensitivity measures. By setting $g_{c} \coloneqq \frac {\tau} {\kappa}$ for $c \in T$ and $g_c \coloneqq t_c$ for $c \notin T$, we have\vspace{0.2em}
\begin{equation}
\sum_{c} g_c =
\sum_{c \in T} \frac {\tau} {\kappa} + \sum_{c \notin T} t_c \le
2 \tau,
\end{equation}
\begin{equation}
M =
\prod_{c \in T} \left\lceil \frac {t_c} {\frac {\tau} \kappa} \right\rceil \cdot \prod_{c \notin T} \left\lceil \frac {t_c} {t_c} \right\rceil =
\prod_{i=1}^{\kappa} \left\lceil \frac {\kappa \cdot t_i^s} {\tau} \right\rceil,
\vspace{0.2em}
\end{equation}
which completes the proof.

\subsubsection{Step 2: Extension to General Symmetric Distributions}

\textbf{Lemma 2 (Symmetric Distributions and Probability Bounds).}\label{lemma:symmetric_distribution} \textit{Let the query and key distributions $\mathcal{Q}, \mathcal{K}$ satisfy Assumption~\hyperref[assumption:query_key_distribution]{1}. Denote $L^0$ as the attention logit matrix when $q_i=q$ and $k_i=k$ for all $i$. Then, for any $j_1<j_2 \le i$, we have}\vspace{0.2em}
\begin{equation}
\mathrm{Pr}\left(L_{i,j_1} - L_{i,j_2} \le L^0_{i,j_1} - L^0_{i,j_2}\right) \ge \frac 1 4.
\vspace{0.2em}
\end{equation}
To prove Lemma~\hyperref[lemma:symmetric_distribution]{2}, we need the following lemmas regarding symmetric distributions.

\textbf{Lemma 3.}\label{lemma:symmetric_difference} \textit{Let $(\Omega, \mathscr{F}, \mathrm{P})$ be a probability space, $V$ be a vector space, and $X,Y: \Omega \to V$ be random variables. Assume $X$ and $Y$ are independent and symmetric about their respective means $q$ and $k$. Then, $X-Y$ is symmetric about $q-k$.}

\textbf{Proof.} For any measurable set $\mathcal{S} \subseteq V$, we denote the translation of $\mathcal{S}$ by $x$ as $\mathcal{S}+x \coloneqq \{s+x \mid s \in \mathcal{S}\}$ and other set operations similarly. Then\vspace{0.2em}
\begin{equation}
\begin{aligned}
\mathrm{Pr}(X-Y \in \mathcal{S})
&=
\int_Y \mathrm{Pr}(X \in \mathcal{S}+Y) \,d\mathrm{P}
\\ &=
\int_Y \mathrm{Pr}(X \in 2q-\mathcal{S}-Y) \,d\mathrm{P}
\\ &=
\int_X \mathrm{Pr}(Y \in 2q-\mathcal{S}-X) \,d\mathrm{P}
\\ &=
\int_X \mathrm{Pr}(Y \in 2k-2q+\mathcal{S}+X) \,d\mathrm{P}
\\ &=
\int_X \mathrm{Pr}(X-Y \in 2(q-k)-\mathcal{S}) \,d\mathrm{P}
\\[5pt] &=
\mathrm{Pr}(X-Y \in 2(q-k)-\mathcal{S}).
\end{aligned}
\vspace{0.2em}
\end{equation}

\textbf{Lemma 4.}\label{lemma:symmetric_linear_image} \textit{Let $(\Omega, \mathscr{F}, \mathrm{P})$ be a probability space, $L: V \to W$ be a linear operator, and $X: \Omega \to V$ be a random variable. Assume $X$ is symmetric about its mean $q$. Then, $L(X)$ is symmetric about $L(q)$.}

\textbf{Proof.} For any measurable set $\mathcal{S} \subseteq W$, we have\vspace{0.2em}
\begin{equation}
\begin{aligned}
\mathrm{Pr}(L(X) \in \mathcal{S})
&=
\mathrm{Pr}(X \in L^{-1}(\mathcal{S}))
\\ &=
\mathrm{Pr}(2q-X \in L^{-1}(\mathcal{S}))
\\ &=
\mathrm{Pr}(L(2q-X) \in \mathcal{S})
\\ &=
\mathrm{Pr}(L(X) \in 2L(q)-\mathcal{S}).
\end{aligned}
\vspace{0.2em}
\end{equation}

\textbf{Lemma 5.}\label{lemma:inner_product_probability} \textit{Let $(\Omega, \mathscr{F}, \mathrm{P})$ be a probability space, $V$ be an inner product space, and $X,Y: \Omega \to V$ be random variables. Assume $X$ and $Y$ are independent and symmetric about their respective means $q$ and $k$. Then, $\mathrm{Pr}(\langle X, Y \rangle \le \langle q, k \rangle) \ge \frac 1 4$.}

\textbf{Proof.} We define four sets based on reflections about $q$ and $k$:\vspace{0.2em}
\begin{equation}
\begin{aligned}
T_0 &\coloneqq \{(x,y) \in V \times V \mid \langle x, y \rangle \le \langle q, k \rangle\} \\
T_1 &\coloneqq \{(x,y) \in V \times V \mid \langle x,2k-y \rangle \le \langle q, k \rangle\} \\
T_2 &\coloneqq \{(x,y) \in V \times V \mid \langle 2q-x,y \rangle \le \langle q, k \rangle\} \\
T_3 &\coloneqq \{(x,y) \in V \times V \mid \langle 2q-x,2k-y \rangle \le \langle q, k \rangle\}.
\end{aligned}
\vspace{0.2em}
\end{equation}
By symmetry, these four sets have equal probability.

Observe that for any $(x,y) \in V \times V$, we have\vspace{0.2em}
\begin{equation}
\left\langle x, y \right\rangle + \left\langle x, 2k-y \right\rangle + \left\langle 2q-x, y \right\rangle + \left\langle 2q-x, 2k-y \right\rangle = 4 \langle q, k \rangle,
\vspace{0.2em}
\end{equation}
which implies at least one of the four pairs must belong to $T_0$.

Therefore, $T_0 \cup T_1 \cup T_2 \cup T_3 = V \times V$. Hence $\mathrm{Pr}((X,Y) \in T_0) + \mathrm{Pr}((X,Y) \in T_1) + \mathrm{Pr}((X,Y) \in T_2) + \mathrm{Pr}((X,Y) \in T_3) \ge 1$ and $\mathrm{Pr}((X,Y) \in T_0) \ge \frac 1 4$ follows immediately.
\vspace{0.2em}

\textbf{Proof of Lemma 2.} We express the attention logit difference $L_{i,j_1} - L_{i,j_2}$ in terms of $Q$ and $K$:\vspace{0.2em}
$$
L_{i,j_1} - L_{i,j_2} = \frac 1 {\sqrt {d_h}} \left\langle \varphi_i(q_i), \varphi_{j_1}(k_{j_1}) - \varphi_{j_2}(k_{j_2}) \right\rangle.
\vspace{0.2em}
$$
By Lemma~\hyperref[lemma:symmetric_linear_image]{4}, $\varphi_i(q_i), \varphi_{j_1}(k_{j_1}), \varphi_{j_2}(k_{j_2})$ are symmetric about $\varphi_i(q), \varphi_{j_1}(k), \varphi_{j_2}(k)$, respectively. By Lemma~\hyperref[lemma:symmetric_difference]{3}, the difference $\varphi_{j_1}(k_{j_1}) - \varphi_{j_2}(k_{j_2})$ is symmetric about $\varphi_{j_1}(k) - \varphi_{j_2}(k)$. Finally, by Lemma~\hyperref[lemma:inner_product_probability]{5}, we have\vspace{0.2em}
\begin{equation}
\mathrm{Pr}\left(\left\langle \varphi_i(q_i), \varphi_{j_1}(k_{j_1}) - \varphi_{j_2}(k_{j_2}) \right\rangle \le \left\langle \varphi_i(q), \varphi_{j_1}(k) - \varphi_{j_2}(k) \right\rangle\right) \ge \frac 1 4.
\vspace{0.2em}
\end{equation}
Hence\vspace{0.2em}
\begin{equation}
\mathrm{Pr}\left(L_{i,j_1} - L_{i,j_2} \le L^0_{i,j_1} - L^0_{i,j_2}\right) \ge \frac 1 4,
\vspace{0.2em}
\end{equation}
which completes the proof.

\section{Kernel Implementation Details}\label{app:kernel_implementation}

This section details the kernel design of FlashWavePrune (Section~\ref{sec:system_design}).

\subsection{Forward Kernel}

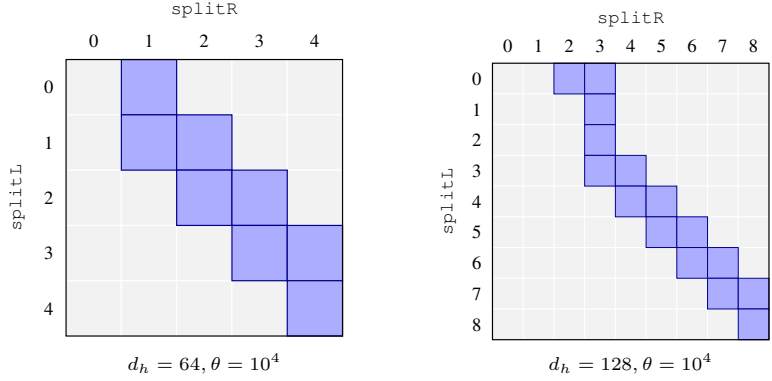
\begin{figure}[t]
\centering
\begin{subfigure}[t]{0.4\linewidth}
\centering
\begin{tikzpicture}[x=0.288in,y=0.288in,font=\scriptsize]
\foreach \l in {0,...,4} {
  \foreach \r in {0,...,4} {
    \draw[fill=gray!10, draw=white, line width=0.45pt] (\r,-\l) rectangle ++(1,-1);
  }
}
\foreach \l/\r in {0/1,1/1,1/2,2/2,2/3,3/3,3/4,4/4} {
  \draw[fill=blue!32, draw=blue!55!black, line width=0.45pt] (\r,-\l) rectangle ++(1,-1);
}
\draw[black, line width=0.45pt] (0,0) rectangle (5,-5);
\foreach \i in {0,...,4} {
  \node[above] at (\i+0.5,0.06) {\i};
  \node[left] at (-0.06,-\i-0.5) {\i};
}
\node at (2.5,0.9) {\texttt{splitR}};
\node[rotate=90] at (-0.9,-2.5) {\texttt{splitL}};
\node[below] at (2.5,-5.15) {$d_h=64,\theta=10^4$};
\end{tikzpicture}
\end{subfigure}
\begin{subfigure}[t]{0.4\linewidth}
\centering
\begin{tikzpicture}[x=0.16in,y=0.16in,font=\scriptsize]
\foreach \l in {0,...,8} {
  \foreach \r in {0,...,8} {
    \draw[fill=gray!10, draw=white, line width=0.35pt] (\r,-\l) rectangle ++(1,-1);
  }
}
\foreach \l/\r in {0/2,0/3,1/3,2/3,3/3,3/4,4/4,4/5,5/5,5/6,6/6,6/7,7/7,7/8,8/8} {
  \draw[fill=blue!32, draw=blue!55!black, line width=0.35pt] (\r,-\l) rectangle ++(1,-1);
}
\draw[black, line width=0.45pt] (0,0) rectangle (9,-9);
\foreach \i in {0,...,8} {
  \node[above] at (\i+0.5,0.05) {\i};
  \node[left] at (-0.05,-\i-0.5) {\i};
}
\node at (4.5,1.45) {\texttt{splitR}};
\node[rotate=90] at (-1.45,-4.5) {\texttt{splitL}};
\node[below] at (4.5,-9.15) {$d_h=128,\theta=10^4$};
\end{tikzpicture}
\end{subfigure}
\caption{Two example grids of (\texttt{splitL,splitR}) branches. Blue squares denote reachable branches; gray squares denote unreachable ones. Note that the reachable set depends on $\{d'_g\}$; we set the model's RoPE base $\theta$ to $10^4$ here for illustration.}
\label{fig:split_pair_grid}
\end{figure}


\textbf{Improving ILP via compile-time specialization.} If \texttt{splitL} and \texttt{splitR} are purely treated as runtime variables, the compiler has to preserve the data-dependent branches and cannot fully unroll the loop over head-dimension groups. This reduces Instruction-Level Parallelism (ILP) and prevents the efficient scheduling of shared-memory loads and tensor-core MMA instructions. To resolve this, we enumerate the (splitL, splitR) pairs and pass each pair as arguments to the kernel template, so the split points become compile-time constants. The retained suffix is thus compiled into a highly optimized static GEMM, and the masking overhead is strictly confined to the boundary region.

\textbf{Controlling branching and code size.} Naively specializing each (\texttt{splitL,splitR}) branch creates separate template instances that scale quadratically with the number of head-dimension groups (e.g., $8$ for $d_h=128$ and $16$ for $d_h = 256$). This bloats the SASS code size and exhausts the compiler's inlining budget, forcing helper functions into out-of-line device calls (\texttt{CALL.FUNC}). We also observe that this hinders the compiler's ability to conduct interprocedural optimizations across the entire kernel. To avoid these issues, we specialize only the \textit{reachable} (\texttt{splitL,splitR}) branches for a given model, which are far fewer than the full Cartesian product (Figure~\ref{fig:split_pair_grid}). Although this necessitates recompiling the kernel for each distinct model, it is an acceptable tradeoff, as the inlining of hot paths and the simplified control flow outweigh the one-time compilation cost.

\subsection{Backward Kernel}

\textbf{Asymmetry with the forward kernel.} We first name the computation performed by the forward kernel. In the standard MMA notation, the $M$ and $N$ axes index the rows and columns of the output matrix, and the $K$ axis is the contraction axis that is summed over to produce the output. For the forward attention computation, these axes correspond to the query dimension ($i$), the key dimension ($j$), and the head dimension ($k$), respectively. In WavePrune, the head dimension is grouped into 16-dimensional head-dimension groups, where each head-dimension group is assigned a pruning threshold $d'_g$ (Equation~\ref{eq:waveprune_approximation}).

Let $Q \in \mathbb{R}^{m \times k}, K \in \mathbb{R}^{n \times k}$ denote the RoPE-encoded query and key matrices. The core computation of the forward kernel is $P \coloneqq \mathrm{WMM}_P(Q,K)$ (WMM denotes WavePrune matrix multiplication):\vspace{-0.8em}
\begin{equation}\label{eq:wmm}
P = \left[\mathrm{WMM}_P(X,Y)\right]_{i,j} \coloneqq \sum_{k=1}^{d_h} X_{i,k} Y_{j,k} \cdot \mathbf{1}_{|i-j| \le d'_{g(k)}}.
\end{equation}
Crucially, the backward pass has a different axis arrangement. Given $\mathrm{d}P$ and $K$, $\mathrm{d}Q$ is computed as
$$
\mathrm{d}Q = \left[\mathrm{WMM}_Q(\mathrm{d}P,K)\right]_{i,k} \coloneqq \sum_{j=1}^{n} \mathrm{d}P_{i,j} K_{j,k} \cdot \mathbf{1}_{|i-j| \le d'_{g(k)}},
$$
where the $M,N,K$ axes correspond to the query dimension ($i$), the head dimension ($k$), and the key dimension ($j$), respectively. The computation of $\mathrm{d}K$ is similar:
$$
\mathrm{d}K = \left[\mathrm{WMM}_K(\mathrm{d}P,Q)\right]_{j,k} \coloneqq \sum_{i=1}^{m} \mathrm{d}P_{i,j} Q_{i,k} \cdot \mathbf{1}_{|i-j| \le d'_{g(k)}},
$$
where the $M,N,K$ axes correspond to the key dimension ($j$), the head dimension ($k$), and the query dimension ($i$), respectively.

Unlike forward computation ($\mathrm{WMM}_P$), the head dimension for $\mathrm{WMM}_K$ and $\mathrm{WMM}_Q$ is the $N$ axis rather than the shared $K$ axis.

\begin{table}[t]
    \centering
    \caption{Ablation of WavePrune kernel optimizations. Each entry reports the runtime ratio relative to the final optimized kernel at sequence length 32768; larger values indicate slower execution. Since our prefill implementation does not use suffix-only K loading, these entries are marked as N/A.}
    \label{tab:system_efficiency_ablation}
    {\fontsize{9}{10.8}\selectfont
    \setlength{\tabcolsep}{3.5pt}
    \begin{tabular}{l|cccc}
        \toprule
        \textbf{Removed Optimizations} & \textbf{Prefill H64} $\downarrow$ & \textbf{Prefill H128} $\downarrow$ & \textbf{Decode H64} $\downarrow$ & \textbf{Decode H128} $\downarrow$ \rule{0pt}{2.2ex} \\
        \midrule\\[-1.35em]
        w/o splits specialization & 1.26$\times$ & 1.47$\times$ & 1.09$\times$ & 1.04$\times$ \rule{0pt}{2.3ex} \\[0ex]
        w/o splitR specialization & 1.25$\times$ & 1.54$\times$ & 1.04$\times$ & 1.08$\times$ \rule{0pt}{2.3ex} \\[0ex]
        w/o suffix-only K load & \textcolor{gray}{N/A} & \textcolor{gray}{N/A} & 0.99$\times$ & 1.14$\times$ \rule{0pt}{2.3ex} \\[0ex]
        w/o suffix-only QK GEMM & 1.38$\times$ & 1.46$\times$ & 1.01$\times$ & 1.08$\times$ \rule{0pt}{2.3ex} \\[0ex]
        \bottomrule
    \end{tabular}}
\end{table}

\textbf{Implementation overview.} We modify only the $\mathrm{d}Q$ and $\mathrm{d}K$ computations as $\mathrm{d}V$ and $\mathrm{d}P$ follow FA2's standard path. Despite the axis mismatch above, both gradients can still be computed using the same idea as the forward kernel. Under the Ampere MMA atom \texttt{M16N8K16}, an atom spans $8$ elements along the $N$ axis (head dimension), which divides the group size of $16$. No atom therefore straddles two head-dimension groups, so the pruning threshold $d'_g$ can be applied at the atom level, and the three-way split of the forward kernel carries over (Figure~\ref{fig:waveprune_split_design}). The primary difference is that we must mask the operand $\mathrm{d}P$ \textit{before} the MMA rather than the accumulator afterwards.

\textbf{Joint $\boldsymbol{\mathrm{d}Q}$ and $\boldsymbol{\mathrm{d}K}$ computation.} For a given head-dimension group $g$, the two gradients $\mathrm{d}Q$ and $\mathrm{d}K$ apply the same predicate $\mathbf{1}_{|i-j| \le d'_g}$ on $\mathrm{d}P$, so we can mask $\mathrm{d}P$ once and use the result twice by jointly computing $\boldsymbol{\mathrm{d}Q}$ and $\boldsymbol{\mathrm{d}K}$. Specifically, we enumerate all critical groups (Section~\ref{sec:system_design}) \textit{at runtime}, load a short-lived $\mathrm{d}P$ fragment from shared memory, apply the predicate on the fragment, and then issue MMAs for both gradients. This joint computation halves the masking work and addresses the ALU bottleneck. For the fully retained suffix $[\texttt{splitR},d_h/16)$, we directly consume the unmasked tiles and template-specialize $\texttt{splitR}$ in the same way as the forward kernel.

\textbf{Why per-pair specialization fails.} Surprisingly, the $(\texttt{splitL,splitR})$ compile-time specialization that accelerates the forward kernel is counterproductive here. The backward kernel body is massive: it contains the score recomputation, the $\mathrm{d}V$ and $\mathrm{d}P$ GEMM, and the joint $\mathrm{d}Q, \mathrm{d}K$ GEMM. Cloning the joint GEMM per (\texttt{splitL}, \texttt{splitR}) pair replicates a far larger body, which prevents the compiler from generating efficient machine code. The cost does not grow gradually but appears as a discrete cliff: with $m_{\text{block}} = 64, n_{\text{block}} = 128, d_h=128$, specializing even a single (\texttt{splitL,splitR}) pair grows the per-thread stack frame from 120 to 512 bytes, the static local load/store counts from 34/38 to 275/239, and the kernel from 89,856 to 122,752 bytes of SASS. Once \texttt{ptxas} crosses its inlining and allocation budgets, it holds program state in local memory and moves code out of line, adding control overhead that outweighs the pruning savings.

\textbf{Results.} Table~\ref{tab:backward_kernel_latency} reports backward kernel latencies under the same setup as Section~\ref{sec:system_efficiency}. The overhead of WavePrune backward kernel is $8.58\%$ at 8K tokens, which narrows to $3.81\%$ at 16K and $0.18\%$ at 32K. Comparing with standard FA2, the overall training-time gap is negligible.

\begin{table}[htbp]
\centering
\caption{Backward kernel latencies. Overhead refers to the latency increase of WavePrune over FA2.}
\label{tab:backward_kernel_latency}
\setlength{\tabcolsep}{6pt}
{\fontsize{9}{10.8}\selectfont
\renewcommand{\arraystretch}{1.1}
{%
\begin{tabular}{c|rrr}
\toprule
\textbf{Length} & \textbf{FA2 (ms)} & \textbf{WavePrune (ms)} & \textbf{Overhead (\%)} \\
\midrule
8K & 93.30 & 101.31 & 8.58\% \\
16K & 363.24 & 377.08 & 3.81\% \\
32K & 1434.34 & 1436.88 & 0.18\% \\
\bottomrule
\end{tabular}%
}}
\renewcommand{\arraystretch}{1}
\end{table}

\section{Training-Free Integration}\label{app:pretrained_integration}

\textbf{HELMET benchmark.} HELMET \citep{yen2025helmet} is a comprehensive benchmark comprising diverse long-context tasks and extending to sequence lengths as long as 128K. Specifically, it includes the following \textit{tasks}, each of which consists of multiple \textit{subtasks}:
\begin{itemize}
    \item \textbf{Recall:} Retrieving hidden information from a large context (inherited from RULER \citep{hsieh2024ruler} and JSON KV \citep{liu-etal-2024-lost}).
    \item \textbf{RAG:} Answering open-domain questions by retrieving relevant information from a provided noisy corpus.
    \item \textbf{Re-rank:} Ranking a set of passages based on their relevance to a given query.
    \item \textbf{Cite:} Generating long-form answers to multi-faceted questions while providing correct inline passage citations.
    \item \textbf{LongQA:} Answering questions based on long documents like books or movie scripts.
    \item \textbf{ICL:} Solving new classification tasks using thousands of in-context examples.
    \item \textbf{Summ:} Synthesizing and summarizing information from extensive contexts.
\end{itemize}

Due to the substantial cost of LLM-as-a-Judge (which requires proprietary models such as GPT-4o in the original HELMET evaluation), we switch to ROUGE-1 F1-score for NarrativeQA evaluation (a subtask of LongQA). In Summ, where the ground-truth and model-generated summaries are often more than 1,000 tokens long, we find F1 score fails to capture the quality of summaries, so we drop Summ from our evaluation and report only the remaining six tasks. For all subtasks where LLM-as-a-Judge is not required, we use the same evaluation protocol as HELMET.

During evaluation, we find the maximum output tokens are insufficient for RULER MV and JSON KV (subtasks of Recall), which leads to extremely unstable results due to early truncation of model responses. Therefore, we increase the maximum output tokens of RULER MV and JSON KV to 200 and 300 respectively.

\textbf{General benchmarks.} To ensure a comprehensive evaluation, we select benchmarks from four core capabilities:
\begin{itemize}
\item \textbf{Language modeling}: MMLU \citep{hendryckstest2021}, CMMLU \citep{li2023cmmlu}
\item \textbf{General reasoning}: BBH \citep{suzgun2022challenging}, DROP \citep{dua2019drop}
\item \textbf{Code}: MBPP \citep{austin2021program}, HumanEval \citep{chen2021codex}
\item \textbf{Math}: GSM8K \citep{cobbe2021training}, MATH \citep{hendrycksmath2021}
\end{itemize}
All evaluations are conducted using the \texttt{lm-evaluation-harness} framework \citep{eval-harness} under default configurations.

\textbf{Evaluated LLMs.} Table~\ref{tab:pretrained_integration_models} summarizes the architectural properties of the LLMs evaluated in Section~\ref{sec:pretrained_integration}. Following \citet{yang2025qwen3technicalreport}, we use YaRN \citep{peng2024yarn} for Qwen3-8B and Qwen3-32B at context lengths above 32K, and native RoPE otherwise. To reduce evaluation cost, we disable the thinking mode of Qwen3-8B and Qwen3-32B.

\begin{table}[htbp]
    \centering
    {\fontsize{8}{9.6}\selectfont
    \setlength{\tabcolsep}{3.5pt}
    \resizebox{\textwidth}{!}{%
    \begin{tabular}{lcccccc}
        \toprule
        \textbf{Model} & \textbf{Context length} & \textbf{RoPE $\theta$} & \textbf{RoPE scaling} & \textbf{Attention type} & \textbf{Release date} \rule{0pt}{2.2ex} \\
        \midrule\\[-1.35em]
        Qwen3-8B & 32K (128K w/ YaRN) & $10^6$ & None (YaRN optional) & Standard full attention & Apr 2025 \rule{0pt}{2.3ex} \\[0.15ex]
        Qwen3-32B & 32K (128K w/ YaRN) & $10^6$ & None (YaRN optional) & Standard full attention & Apr 2025 \rule{0pt}{2.3ex} \\[0.15ex]
        Gemma3-12B & 128K & $10^4$ / $10^6$ & Linear ($\times 8$) & Sliding window + full attention & Mar 2025 \rule{0pt}{2.3ex} \\[0.15ex]
        Ministral3-3B & 256K & $10^6$ & YaRN ($\times 16$) & Standard full attention & Dec 2025 \rule{0pt}{2.3ex} \\[0.15ex]
        Llama3.1-8B & 128K & $5 \times 10^5$ & Llama~3 ($\times 8$) & Standard full attention & Jul 2024 \rule{0pt}{2.3ex} \\[0.15ex]
        \bottomrule
    \end{tabular}%
    }}
    \caption{Properties of LLMs used in long-context evaluation.}
    \label{tab:pretrained_integration_models}
\end{table}

\section{Failed Attempts on Llama3.1-8B}\label{app:llama31_failure}

\textbf{Diagnosis: high-frequency channels carry long-range retrieval signals.} Because Llama3.1-8B's degradation (Table~\ref{tab:main_results}) is concentrated on retrieval tasks, we inspect a 32K-context NIAH example in which the needle is placed near the middle of the context. In the original model, layer~16, head~21 attends strongly to the needle, which is driven by a RoPE channel whose wavelength is approximately $571$. Surprisingly, this channel stably transmits the needle's information to the final tokens despite a query--key distance far exceeding its wavelength. This behavior contradicts the conventional view that high-frequency channels encode local positional information \citep{barbero2025round}. WavePrune truncates exactly this signal, causing the model to fail at retrieving the answer. Similar patterns are observed in other heads of Llama3.1-8B, sometimes driven by different RoPE channels.

\textbf{Failed attempts.} We attempted the following remedies; none of them recovers the performance of the unpruned baseline.
\begin{itemize}[leftmargin=2.4em,topsep=3.2pt,partopsep=3.2pt,parsep=3.2pt,itemsep=0pt]
\item Increasing the pruning coefficient $\alpha$ (Section~\ref{sec:channel_wise_attn_truncation}) to postpone truncation.
\item Truncating only the first few highest-frequency channels.
\item Applying WavePrune only to the first or last few layers.
\item Continuing long-context training with WavePrune under the ProLong recipe~\citep{gao-etal-2025-train}, on either the base or the instruct models.
\end{itemize}

\section{Controlled Experiment Details}\label{app:toy_experiments}

\subsection{Tetranacci Data Construction}\label{app:tetranacci_construction}

\textbf{Data Construction.} In this task, each sample encodes a Tetranacci sequence, which is defined by $a_n \coloneqq (a_{n-1} + a_{n-2} + a_{n-3} + a_{n-4}) \bmod 1000$. The initial values $a_0, a_1, a_2, a_3$ are drawn uniformly at random from $\{0, \ldots, 999\}$. Values are written in decimal, tokenized with the GPT-2 tokenizer \citep{radford2019language}, and each is followed by an EOT token that separates consecutive values. The four initial values, the EOT tokens between them, and all padding are excluded from the training loss; every subsequent value and EOT token are supervised. Since the recurrence is deterministic given the four most recent values, a sequence becomes periodic as soon as any block of four consecutive values repeats; we discard and regenerate such samples to guarantee that no training or validation sequence is periodic. The dataset comprises $500{,}000$ samples, split $99{:}1$ into training and validation sets, with sequence lengths of $2{,}048$ and $4{,}096$ tokens, respectively.

\textbf{Example.} With initial values $a_0 = 123$, $a_1 = 456$, $a_2 = 789$, and $a_3 = 12$, the next values are $a_4 = (123+456+789+12) \bmod 1000 = 380$ and $a_5 = (456+789+12+380) \bmod 1000 = 637$. The resulting token sequence is
\vspace{0.35em}
$$
\texttt{123 <eot> 456 <eot> 789 <eot> 12 <eot>}
\ 
\underbrace{\texttt{380 <eot> 637 <eot> }\cdots}_{\text{supervised}} \ ,
$$

\vspace{-1.35em}
where $\langle\texttt{eot}\rangle$ denotes the EOT token.

\subsection{NIAH Data Construction}\label{app:niah_construction}

\textbf{Data Construction.} In this task, each sample hides a ``magic number'' (a seven-digit integer drawn uniformly at random from $[10^6, 10^7-1]$) inside a haystack of irrelevant text sampled from OpenWebText. A sample concatenates (i) a fixed instruction stating that a magic number exists, (ii) a segment of background tokens containing the sentence ``The special magic number is $\langle\text{number}\rangle$'' inserted at a uniformly random position, (iii) the fixed instruction ``The special magic number mentioned in the text above is:'', and (iv) the target number followed by a final EOT token. Only the target number and the final EOT token contribute to the training loss. The dataset comprises $500{,}000$ samples, split $99{:}1$ into training and validation sets, with sequence lengths of $2{,}048$ and $4{,}096$ tokens, respectively.

\textbf{Example.} With the magic number $5{,}432{,}167$, a sample has the layout
\begingroup
\small
$$
\begin{aligned}
&\texttt{[instruction]}
\ \ \texttt{[background tokens]}\\
&\texttt{\textbackslash{}n The special magic number is 5432167.\textbackslash{}n}
\ \ \texttt{[background tokens]}\\
&\texttt{\textbackslash{}n The special magic number mentioned in the text above is:}
\ \ \underbrace{\texttt{5432167 <eot>}}_{\text{supervised}}.
\end{aligned}
$$
\endgroup

\section{Pretraining Hyperparameters}\label{app:pretraining_experiments}

\begin{table}[htbp]
\centering
{\fontsize{8}{9.6}\selectfont
\setlength{\tabcolsep}{3pt}
\renewcommand{\arraystretch}{1.1}
{%
\begin{tabular}{l@{\hspace{5pt}}|@{\hspace{8pt}}ccc}
    \toprule
    {\textbf{Hyperparameter}} & {\textbf{GPT2-base}} & {\textbf{GPT2-medium}} & {\textbf{GPT2-large}} \\
    \midrule
    Total Params & 117M & 345M & 774M \\
    Layers & 12 & 24 & 36 \\
    Heads & 12 & 16 & 20 \\
    Dimension & 768 & 1024 & 1280 \\
    Head Dimension & 64 & 64 & 64 \\
    \midrule
    Attention Module & \multicolumn{3}{c}{Multi-Head Attention~\citep{NIPS2017_3f5ee243}} \\
    Positional Encoding & \multicolumn{3}{c}{RoPE~\citep{su2024roformer}} \\
    RoPE $\theta$ & \multicolumn{3}{c}{10{,}000} \\
    \midrule
    Dataset & \multicolumn{3}{c}{OpenWebText} \\
    Vocab Size & \multicolumn{3}{c}{50{,}304} \\
    Train Sequence Length & \multicolumn{3}{c}{2{,}048} \\
    Val Sequence Length & \multicolumn{3}{c}{4{,}096} \\
    \midrule
    Micro Batch Size & \multicolumn{3}{c}{3} \\
    Gradient Accum. Steps & \multicolumn{3}{c}{40} \\
    Effective Batch Size (seqs) & \multicolumn{3}{c}{120} \\
    Training Steps & \multicolumn{3}{c}{64{,}000} \\
    Total Train Tokens & \multicolumn{3}{c}{$\sim$15.7B} \\
    \midrule
    Optimizer & \multicolumn{3}{c}{AdamW~\citep{loshchilov2018decoupled}} \\
    Adam $(\beta_1, \beta_2)$ & \multicolumn{3}{c}{$(0.9, 0.95)$} \\
    Base Learning Rate & \multicolumn{3}{c}{$6 \times 10^{-4}$} \\
    Min Learning Rate & \multicolumn{3}{c}{$6 \times 10^{-5}$} \\
    LR Scheduler & \multicolumn{3}{c}{Cosine decay w/ linear warmup} \\
    Warmup Steps & \multicolumn{3}{c}{2{,}000} \\
    Weight Decay & \multicolumn{3}{c}{0.1} \\
    Gradient Clipping & \multicolumn{3}{c}{1.0} \\
    Dropout & \multicolumn{3}{c}{0.0} \\
    Precision & \multicolumn{3}{c}{\texttt{bfloat16}} \\
    \bottomrule
\end{tabular}%
}}
\caption{Training hyperparameters for the GPT-2 series pretraining experiments.}
\label{tab:pretraining_hyperparameters}
\renewcommand{\arraystretch}{1}
\end{table}

\newpage
\section{Supplementary Figures and Tables}
This section collects the figures and tables referenced in the main text.

\begin{figure}[htbp]
\centering
\begin{subfigure}{0.45\textwidth}
\centering
\includegraphics[width=\textwidth]{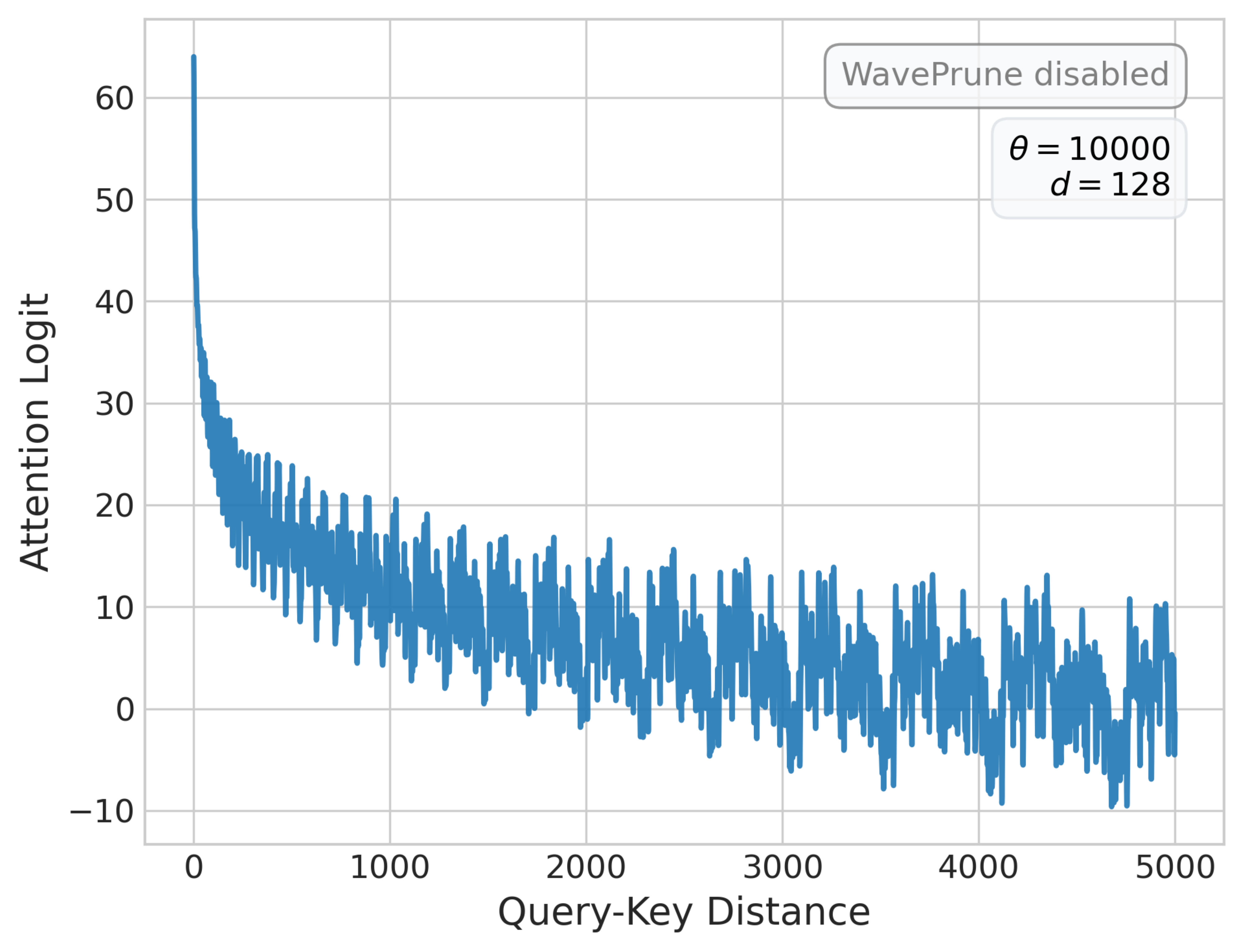}
\caption{Attention logit decay of standard RoPE.}
\label{fig:long_term_decay_rope}
\end{subfigure}
\begin{subfigure}{0.45\textwidth}
\centering
\includegraphics[width=\textwidth]{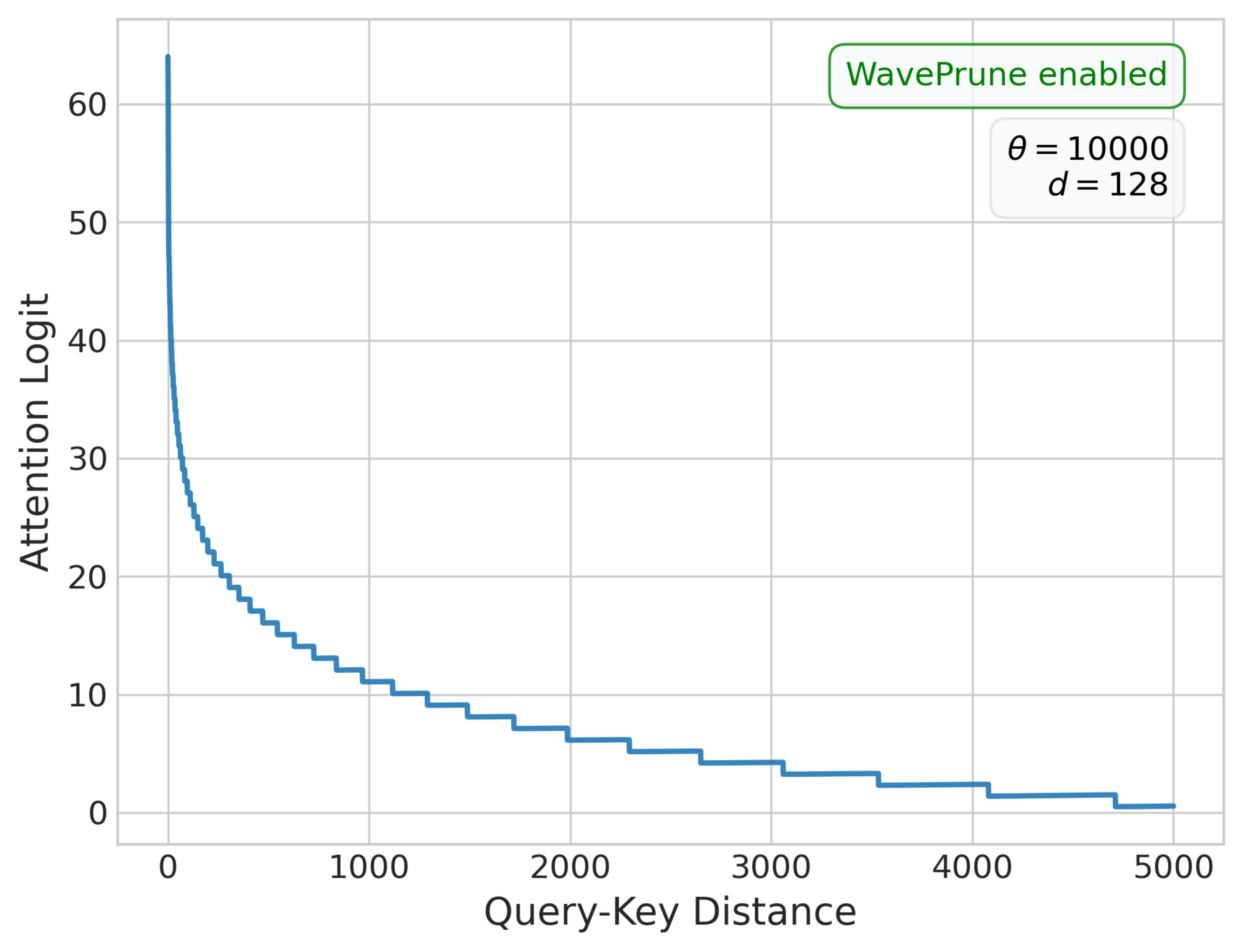}
\caption{Attention logit decay of WavePrune.}
\label{fig:long_term_decay_waveprune}
\end{subfigure}
\caption{Attention logit decay over distance for standard RoPE (left) and WavePrune (right) under the setting of \citet{su2024roformer}. Standard RoPE exhibits chaotic fluctuations that hinder long-range modeling, whereas WavePrune yields a stable, near-monotonic decay.}
\label{fig:long_term_decay}
\end{figure}

\vspace{1em}
\begin{figure}[htbp]
\centering
\begin{subfigure}{0.45\textwidth}
\centering
\includegraphics[width=\textwidth]{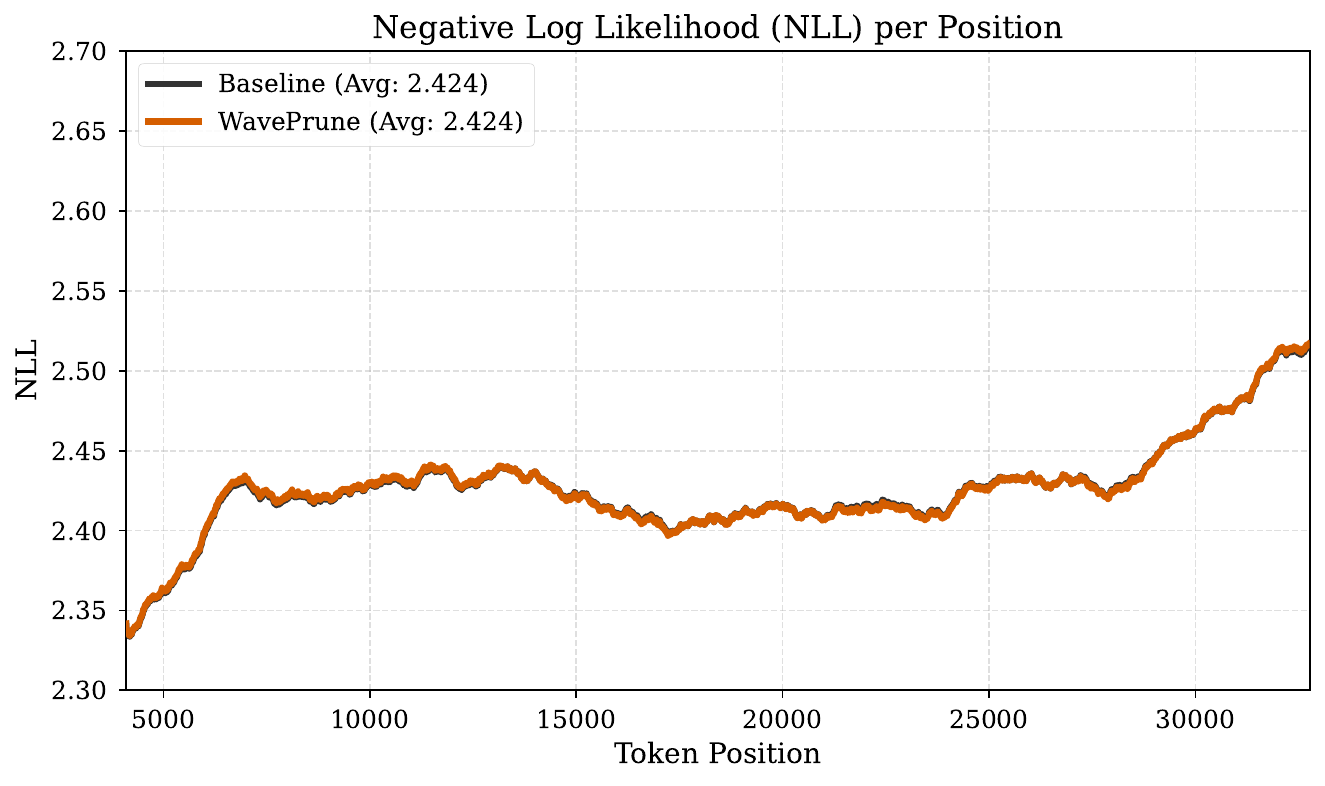}
\caption{NLL without YaRN.}
\label{fig:qwen3_nll_without_yarn}
\end{subfigure}
\begin{subfigure}{0.45\textwidth}
\centering
\includegraphics[width=\textwidth]{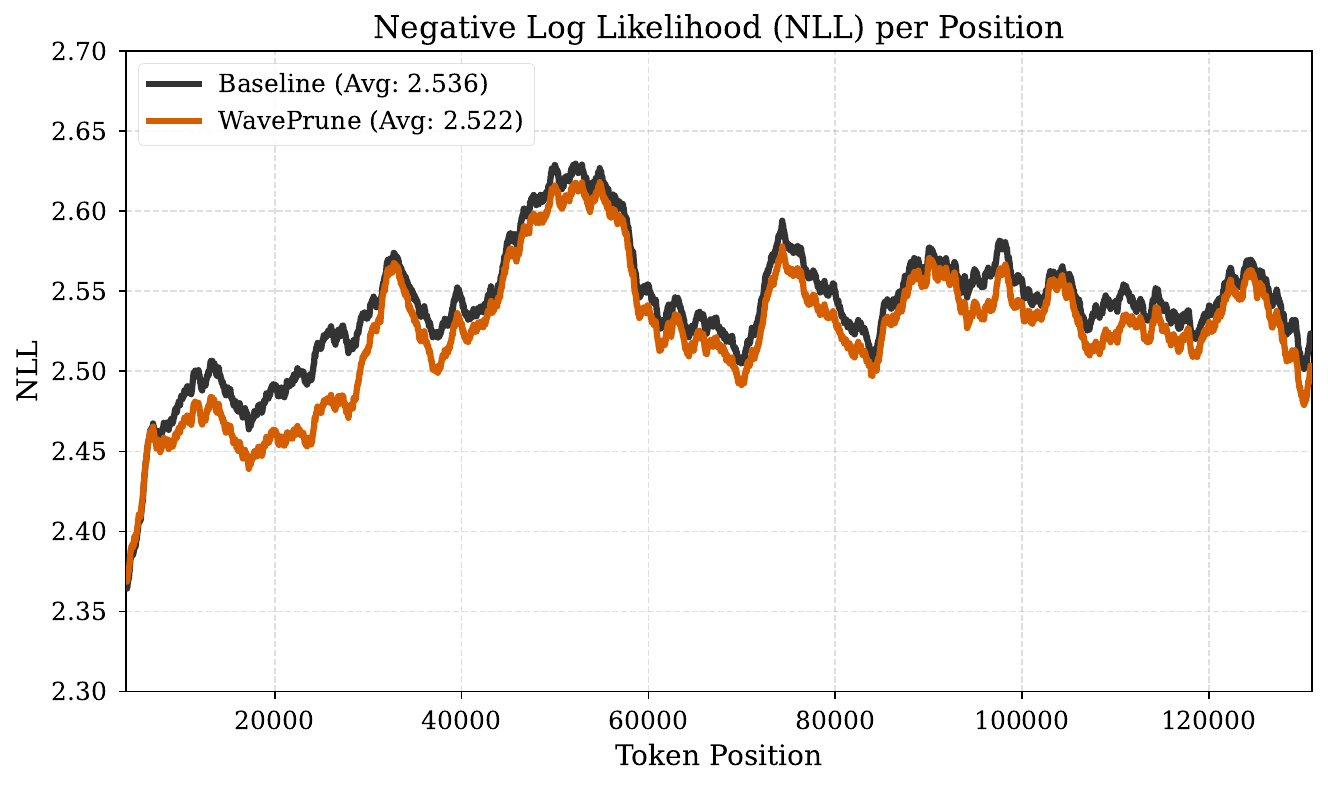}
\caption{NLL under YaRN.}
\label{fig:qwen3_nll_yarn}
\end{subfigure}
\caption{Negative log-likelihoods (NLL) of Qwen3-8B averaged over 100 PG19 samples at \textbf{(a)} 32K context length without YaRN and \textbf{(b)} 128K context length with YaRN. In \textbf{(a)}, the NLL curves nearly coincide; in \textbf{(b)}, WavePrune has a lower NLL than the original model.}
\label{fig:qwen3_nll_comparison}
\end{figure}

\vspace{1em}
\begin{figure}[htbp]
\begin{subfigure}{0.45\textwidth}
\centering
\includegraphics[width=\textwidth]{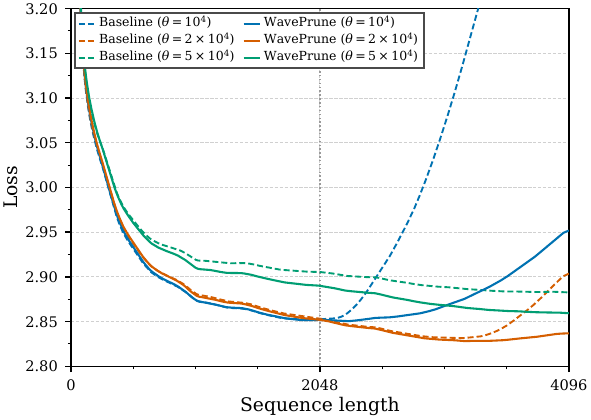}
\caption{Loss curves under GPT2-large.}
\end{subfigure}
\centering
\begin{subfigure}{0.45\textwidth}
\centering
\includegraphics[width=\textwidth]{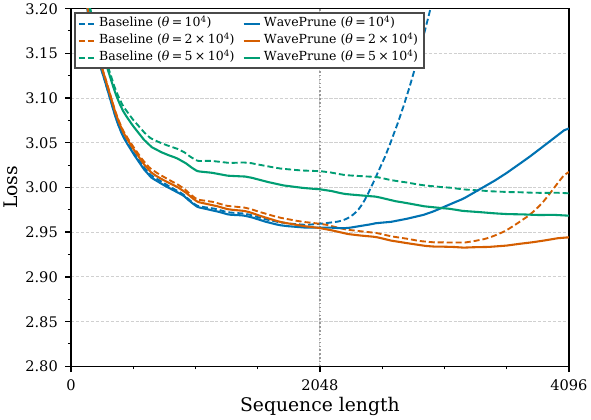}
\caption{Loss curves under GPT2-medium.}
\end{subfigure}
\caption{Validation loss versus sequence length under different RoPE base frequencies $\theta$. WavePrune attains lower validation loss than the baseline across all tested $\theta$.}
\label{fig:ppl_vs_length}
\end{figure}

\newpage
\begin{table}[htbp]
\centering
\caption{Results on general benchmarks.}
\vspace{-0.5em}
\label{tab:short_context_results}
{\fontsize{9}{10.8}\selectfont
\setlength{\tabcolsep}{2.5pt}
\renewcommand{\arraystretch}{1.1}
{%
\begin{tabular}{l@{\hspace{5pt}}|@{\hspace{5pt}}cccccccc@{\hspace{5pt}}|@{\hspace{5pt}}c}
\toprule
\multicolumn{1}{c}{} & \multicolumn{8}{c}{{\textbf{General Benchmarks}}} & \multicolumn{1}{c}{} \\
\cmidrule(lr){2-9}
{\textbf{Method}} & {\textbf{MMLU}} & {\textbf{CMMLU}} & {\textbf{BBH}} & {\textbf{DROP}} & {\textbf{GSM8K}} & {\textbf{MATH}} & {\textbf{MBPP}} & {\textbf{HumanEval}} & {\textbf{Avg.}} \\
\midrule
\multicolumn{10}{l}{\textit{Ministral3-3B}} \\
\hspace{0.6em}Baseline  & 67.41 & 60.18 & 30.52 & 6.06 & 76.27 & 44.66 & 43.40 & 51.22 & 47.47 \\
\hspace{0.6em}WavePrune & 67.55 & 60.31 & 31.15 & 6.02 & 77.10 & 45.22 & 45.00 & 52.44 & \textbf{48.10} \\
\midrule
\multicolumn{10}{l}{\textit{Qwen3-8B}} \\
\hspace{0.6em}Baseline  & 72.92 & 76.90 & 79.47 & 10.91 & 88.32 & 52.84 & 65.40 & 63.41 & 63.77 \\
\hspace{0.6em}WavePrune & 72.90 & 76.89 & 80.05 & 10.94 & 88.48 & 53.00 & 66.40 & 64.02 & \textbf{64.09} \\
\midrule
\multicolumn{10}{l}{\textit{Gemma3-12B}} \\
\hspace{0.6em}Baseline  & 71.51 & 61.38 & 80.49 & 13.93 & 88.48 & 47.76 & 71.00 & 11.59 & 55.77 \\
\hspace{0.6em}WavePrune & 71.54 & 61.34 & 80.22 & 13.97 & 88.70 & 47.60 & 70.60 & 15.24 & \textbf{56.15} \\
\midrule
\multicolumn{10}{l}{\textit{Qwen3-32B}} \\
\hspace{0.6em}Baseline  & 80.79 & 84.16 & 47.38 & 12.06 & 63.68 & 54.14 & 35.80 & 38.41 & 52.05 \\
\hspace{0.6em}WavePrune & 80.55 & 84.17 & 48.38 & 11.84 & 64.22 & 53.90 & 41.80 & 38.41 & \textbf{52.91} \\
\bottomrule
\end{tabular}%
}}
\renewcommand{\arraystretch}{1}
\end{table}

\begin{table}[H]
\centering
\caption{HELMET results at the 64K context split.}
\vspace{-0.5em}
\label{tab:helmet_64k}
{\fontsize{9}{10.8}\selectfont
\setlength{\tabcolsep}{3pt}
\renewcommand{\arraystretch}{1.1}
{%
\begin{tabular}{l@{\hspace{5pt}}|@{\hspace{5pt}}l@{\hspace{5pt}}|@{\hspace{5pt}}cccccc@{\hspace{5pt}}|@{\hspace{5pt}}c}
\toprule
\multicolumn{2}{c}{} & \multicolumn{6}{c}{{\textbf{HELMET}}} & \multicolumn{1}{c}{} \\
\cmidrule(lr){3-8}
{\textbf{Model}} & {\textbf{Method}} & {\textbf{Recall}} & {\textbf{RAG}} & {\textbf{Re-rank}} & {\textbf{Cite}} & {\textbf{LongQA}} & {\textbf{ICL}} & {\textbf{Avg.}} \\
\midrule
\multirow{2}{*}{Ministral3-3B} & Baseline & 88.25 & 59.00 & 28.39 & 13.37 & 36.97 & 26.20 & 42.03 \\
 & WavePrune & 89.06 & 58.75 & 28.10 & 13.76 & 37.05 & 32.40 & \textbf{43.19} \\
\midrule
\multirow{2}{*}{Qwen3-8B} & Baseline & 66.19 & 57.00 & 18.93 & 19.27 & 41.63 & 23.60 & 37.77 \\
 & WavePrune & 71.56 & 58.50 & 25.32 & 16.66 & 42.08 & 35.80 & \textbf{41.65} \\
\midrule
\multirow{2}{*}{Gemma3-12B} & Baseline & 36.25 & 59.00 & 34.91 & 1.04 & 44.16 & 76.20 & 41.93 \\
 & WavePrune & 37.69 & 58.00 & 34.72 & 1.99 & 44.97 & 76.20 & \textbf{42.26} \\
\midrule
\multirow{2}{*}{Qwen3-32B} & Baseline & 76.38 & 63.25 & 37.54 & 30.95 & 46.74 & 78.00 & 55.48 \\
 & WavePrune & 87.06 & 63.00 & 38.78 & 27.60 & 47.10 & 81.20 & \textbf{57.46} \\
\midrule\\[-1.55em]
\midrule
\multirow{2}{*}{Llama3.1-8B} & Baseline & 79.56 & 58.75 & 31.42 & 4.32 & 39.83 & 31.40 & \textbf{40.88} \\
 & WavePrune & 74.81 & 57.50 & 24.20 & 4.29 & 39.43 & 22.20 & 37.07 \\
\bottomrule
\end{tabular}%
}}
\renewcommand{\arraystretch}{1}
\end{table}

\begin{table}[H]
\centering
\caption{HELMET results at the 128K context split.}
\vspace{-0.5em}
\label{tab:helmet_128k}
{\fontsize{9}{10.8}\selectfont
\setlength{\tabcolsep}{3pt}
\renewcommand{\arraystretch}{1.1}
{%
\begin{tabular}{l@{\hspace{5pt}}|@{\hspace{5pt}}l@{\hspace{5pt}}|@{\hspace{5pt}}cccccc@{\hspace{5pt}}|@{\hspace{5pt}}c}
\toprule
\multicolumn{2}{c}{} & \multicolumn{6}{c}{{\textbf{HELMET}}} & \multicolumn{1}{c}{} \\
\cmidrule(lr){3-8}
{\textbf{Model}} & {\textbf{Method}} & {\textbf{Recall}} & {\textbf{RAG}} & {\textbf{Re-rank}} & {\textbf{Cite}} & {\textbf{LongQA}} & {\textbf{ICL}} & {\textbf{Avg.}} \\
\midrule
\multirow{2}{*}{Ministral3-3B} & Baseline & 83.12 & 57.00 & 19.19 & 6.09 & 36.75 & 14.80 & 36.16 \\
 & WavePrune & 83.19 & 56.75 & 18.33 & 5.79 & 36.37 & 17.20 & \textbf{36.27} \\
\midrule
\multirow{2}{*}{Qwen3-8B} & Baseline & 61.19 & 51.50 & 10.42 & 12.17 & 43.67 & 22.20 & 33.52 \\
 & WavePrune & 67.88 & 52.75 & 17.87 & 13.67 & 47.51 & 30.00 & \textbf{38.28} \\
\midrule
\multirow{2}{*}{Gemma3-12B} & Baseline & 26.69 & 49.25 & 12.03 & 0.47 & 42.20 & 76.60 & 34.54 \\
 & WavePrune & 28.44 & 48.00 & 12.57 & 0.11 & 43.95 & 77.80 & \textbf{35.15} \\
\midrule
\multirow{2}{*}{Qwen3-32B} & Baseline & 81.81 & 59.50 & 30.39 & 23.18 & 52.24 & 81.00 & 54.69 \\
 & WavePrune & 86.00 & 59.50 & 30.00 & 20.16 & 53.10 & 82.00 & \textbf{55.13} \\
\midrule\\[-1.55em]
\midrule
\multirow{2}{*}{Llama3.1-8B} & Baseline & 90.31 & 54.00 & 8.63 & 0.04 & 38.67 & 21.60 & \textbf{35.54} \\
 & WavePrune & 76.19 & 53.75 & 8.09 & 0.04 & 39.24 & 22.80 & 33.35 \\
\bottomrule
\end{tabular}%
}}
\renewcommand{\arraystretch}{1}
\end{table}

\newpage
\begin{table}[htbp]
\centering
\caption{Validation loss of GPT2-medium and GPT2-large on OpenWebText at lengths 2048 and 4096. ``Pretraining'' indicates whether WavePrune is enabled during pretraining, and ``Eval'' indicates whether WavePrune is enabled during evaluation.}
\label{tab:gpt2_train_eval}
{\fontsize{8}{9.6}\selectfont
\setlength{\tabcolsep}{4pt}
\renewcommand{\arraystretch}{1.1}
\begin{subtable}{0.49\textwidth}
\centering
\caption{Eval @2048}
\label{tab:gpt2_train_eval_2048}
\begin{tabular}{ll|cc}
\toprule
\multirow{2}{*}{Model} & \multirow{2}{*}{Pretraining} & \multicolumn{2}{c}{Eval} \\
\cmidrule(lr){3-4}
 & & Baseline & WavePrune \\
\midrule
\multirow{2}{*}{GPT2-medium} & Baseline & 3.0270 & 3.0261 \\
 & WavePrune & 3.6210 & \textbf{3.0196} \\
\midrule
\multirow{2}{*}{GPT2-large} & Baseline & 2.9009 & 2.9003 \\
 & WavePrune & 3.7816 & \textbf{2.8974} \\
\bottomrule
\end{tabular}%
\end{subtable}
\begin{subtable}{0.49\textwidth}
\centering
\caption{Eval @4096}
\label{tab:gpt2_train_eval_4096}
\begin{tabular}{ll|cc}
\toprule
\multirow{2}{*}{Model} & \multirow{2}{*}{Pretraining} & \multicolumn{2}{c}{Eval} \\
\cmidrule(lr){3-4}
 & & Baseline & WavePrune \\
\midrule
\multirow{2}{*}{GPT2-medium} & Baseline & 3.8564 & 3.3314 \\
 & WavePrune & 4.3849 & \textbf{3.0953} \\
\midrule
\multirow{2}{*}{GPT2-large} & Baseline & 3.5618 & 3.1797 \\
 & WavePrune & 4.5162 & \textbf{2.9536} \\
\bottomrule
\end{tabular}%
\end{subtable}
}
\end{table}

\vspace{1em}
\begin{table}[htbp]
\centering
\caption{Full comparison between FA2 and WavePrune latency across prefill and decode. We use the same setting as Section~\ref{sec:system_efficiency}, except that sequence lengths extend to 128K.}
\label{tab:system_efficiency_full}
{\fontsize{9}{10.8}\selectfont
\setlength{\tabcolsep}{6pt}
\renewcommand{\arraystretch}{1.1}
{%
\begin{tabular}{{c|c|r@{\hspace{10pt}}rr}}
\toprule
\multicolumn{1}{c|}{\textbf{Setting}} & \textbf{Length} & \textbf{FA2} & \textbf{WavePrune} & \textbf{Speedup} \\
\midrule
\multirow{5}{*}{Prefill, H64} & 8K & 46.96 & 46.07 & \textbf{1.02$\times$} \\
 & 16K & 183.54 & 174.37 & \textbf{1.05$\times$} \\
 & 32K & 731.00 & 673.71 & \textbf{1.09$\times$} \\
 & 64K & 2934.19 & 2659.50 & \textbf{1.10$\times$} \\
 & 128K & 12096.79 & 10775.51 & \textbf{1.12$\times$} \\
\midrule
\multirow{5}{*}{Prefill, H128} & 8K & 82.17 & 75.03 & \textbf{1.10$\times$} \\
 & 16K & 325.02 & 287.66 & \textbf{1.13$\times$} \\
 & 32K & 1306.79 & 1132.27 & \textbf{1.15$\times$} \\
 & 64K & 5337.95 & 4490.98 & \textbf{1.19$\times$} \\
 & 128K & \textcolor{gray}{OOM} & \textcolor{gray}{OOM} & \textcolor{gray}{N/A $\, \, \,$} \\
\midrule
\multirow{5}{*}{Decode, H64} & 8K & 0.448 & 0.432 & \textbf{1.04$\times$} \\
 & 16K & 0.803 & 0.767 & \textbf{1.05$\times$} \\
 & 32K & 1.567 & 1.490 & \textbf{1.05$\times$} \\
 & 64K & 3.095 & 2.938 & \textbf{1.05$\times$} \\
 & 128K & 6.155 & 5.838 & \textbf{1.05$\times$} \\
\midrule
\multirow{5}{*}{Decode, H128} & 8K & 0.826 & 0.734 & \textbf{1.13$\times$} \\
 & 16K & 1.538 & 1.277 & \textbf{1.20$\times$} \\
 & 32K & 3.059 & 2.459 & \textbf{1.24$\times$} \\
 & 64K & 6.111 & 4.816 & \textbf{1.27$\times$} \\
 & 128K & 12.193 & 9.528 & \textbf{1.28$\times$} \\
\bottomrule
\end{tabular}%
}}
\renewcommand{\arraystretch}{1}
\end{table}

\vspace{1em}
\begin{table}[htbp]
\centering
\caption{Channel IDs and their wavelengths in Qwen3-8B. These wavelengths serve as the pruning thresholds in FlashWavePrune.}
\label{tab:channel_wavelengths}
{\fontsize{9}{10.8}\selectfont
\begin{tabular}{c|cccccccc}
    \toprule
    \textbf{Channel ID} & 7 & 15 & 23 & 31 & 39 & 47 & 55 & 63 \rule{0pt}{2.2ex} \\
    \midrule\\[-1.35em]
    \textbf{Wavelength (rounded down)} & 28 & 160 & 900 & 5063 & 28472 & 160114 & 900388 & 5063255 \rule{0pt}{2.3ex} \\[0ex]
    \bottomrule
\end{tabular}}
\end{table}

\newpage
\vspace{1em}
\begin{table}[htbp]
\centering
\caption{Average attention scores to neighboring tokens (distance $\leq 15$), averaged across all 12 GPT2-base layers, evaluated on the Tetranacci training (length 2048) and validation (length 4096) sets. WavePrune assigns more attention mass to neighboring tokens at both lengths and retains more of it at the extrapolated length, indicating it better preserves local attention under length extrapolation.}
\label{tab:att_scores_tetranacci}
{\fontsize{9}{10.8}\selectfont
\begin{tabular}{l|cc|c}
    \toprule
    \textbf{Method} & \textbf{Train (len 2048)} & \textbf{Val (len 4096)} & \textbf{Retention (Val / Train)} \rule{0pt}{2.2ex} \\
    \midrule\\[-1.35em]
    WavePrune & \textbf{0.179} & \textbf{0.148} & \textbf{82.8\%} \rule{0pt}{2.3ex} \\[0ex]
    Baseline  & 0.161 & 0.123 & 76.7\% \rule{0pt}{2.3ex} \\[0ex]
    \bottomrule
\end{tabular}}
\end{table}

\vspace{1em}
\begin{table}[htbp]
    \centering
    \caption{Average attention scores to the ``magic number'' for each retrieval head \citep{wu2025retrieval}, evaluated on samples from the NIAH validation set. We select the first eight retrieval heads of each model, sorted by layer and head index. For each validation set sample, we report the average attention allocated to tokens near the magic number (distance $\leq 5$) from tokens in the final answer span. Notably, the WavePrune model's retrieval heads are consistently activated across different samples, while the baseline model's heads frequently fail to attend to the magic number at all.}
    \label{tab:att_scores_niah}
    {\fontsize{9}{10.8}\selectfont
    \setlength{\tabcolsep}{3pt}
    \begin{tabular}{c|cccc|cccc}
        \toprule
        \multirow{2}{*}{\textbf{Head}} & \multicolumn{4}{c|}{\textbf{WavePrune}} & \multicolumn{4}{c}{\textbf{Baseline}} \rule{0pt}{2.2ex} \\
        \cmidrule(lr){2-5} \cmidrule(lr){6-9}
        & \textbf{Sample 1} & \textbf{Sample 2} & \textbf{Sample 3} & \textbf{Sample 4} & \textbf{Sample 1} & \textbf{Sample 2} & \textbf{Sample 3} & \textbf{Sample 4} \\[0.3ex]
        \midrule\\[-1.35em]
        1  & 0.750 & 0.750 & 0.798 & 0.751 & 0.245 & \uwave{0.000} & \uwave{0.000} & \uwave{0.000} \rule{0pt}{2.3ex} \\[0ex]
        2  & 0.749 & 0.742 & 0.523 & 0.516 & 0.750 & \uwave{0.000} & \uwave{0.000} & \uwave{0.000} \rule{0pt}{2.3ex} \\[0ex]
        3  & 0.748 & 0.412 & 0.800 & 0.250 & \uwave{0.000} & \uwave{0.000} & \uwave{0.000} & 0.373 \rule{0pt}{2.3ex} \\[0ex]
        4  & 0.413 & 0.516 & 0.603 & 0.502 & \uwave{0.000} & 0.250 & 0.600 & \uwave{0.000} \rule{0pt}{2.3ex} \\[0ex]
        5  & 0.944 & \uwave{0.000} & \uwave{0.000} & 0.986 & \uwave{0.021} & 1.000 & 1.000 & \uwave{0.003} \rule{0pt}{2.3ex} \\[0ex]
        6  & 0.750 & 0.750 & 0.800 & 0.750 & \uwave{0.000} & 0.500 & 0.400 & 0.500 \rule{0pt}{2.3ex} \\[0ex]
        7  & 0.751 & 0.750 & 0.800 & 0.750 & \uwave{0.001} & 0.750 & 0.984 & 0.500 \rule{0pt}{2.3ex} \\[0ex]
        8  & 0.394 & 0.500 & 0.800 & 0.250 & \uwave{0.000} & 1.000 & \uwave{0.008} & \uwave{0.000} \rule{0pt}{2.3ex} \\[0ex]
        \midrule\\[-1.35em]
        \textbf{Loss} & \textbf{0.030} & \textbf{0.005} & \textbf{0.003} & \textbf{0.001} & 6.328 & 4.734 & 10.840 & 5.767 \rule{0pt}{2.3ex} \\[0ex]
        \bottomrule
    \end{tabular}}
\end{table}

\vspace{1em}
\begin{figure}[htbp]
\centering
\includegraphics[width=0.895\textwidth]{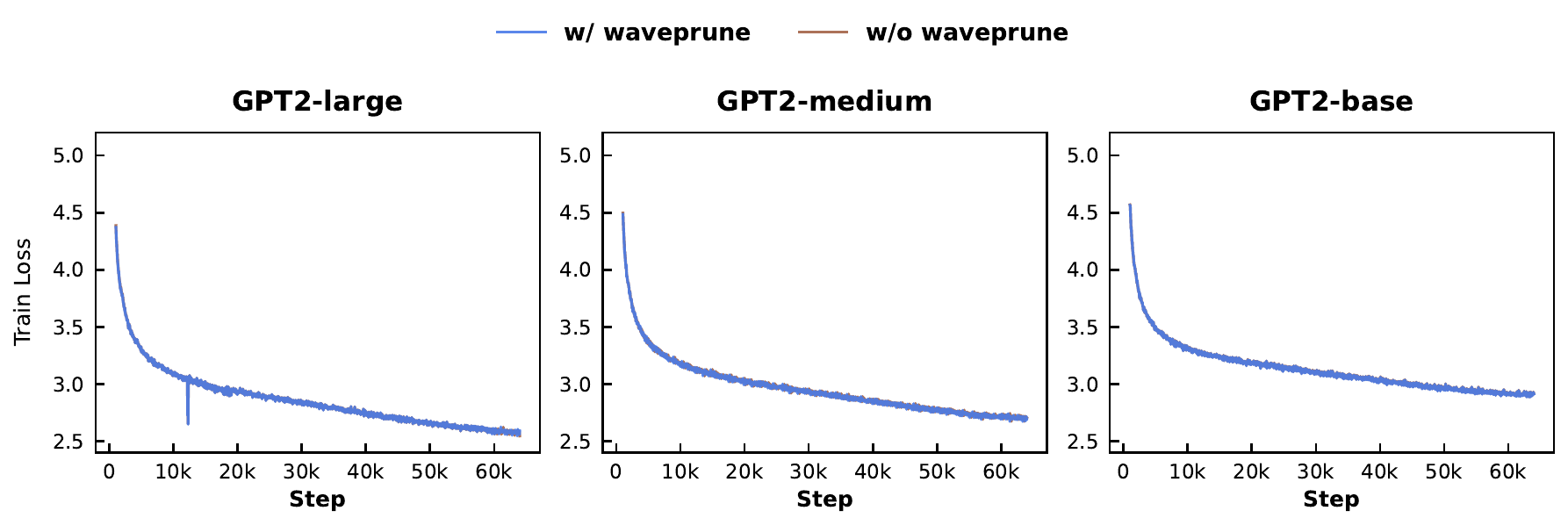}
\caption{Training loss of GPT2-large (left), GPT2-medium (middle), and GPT2-base (right).}
\label{fig:pretrain_results_train}
\end{figure}

\newpage
\vspace{2em}
\begin{figure}[t]
    \centering
    \includegraphics[width=0.7\textwidth]{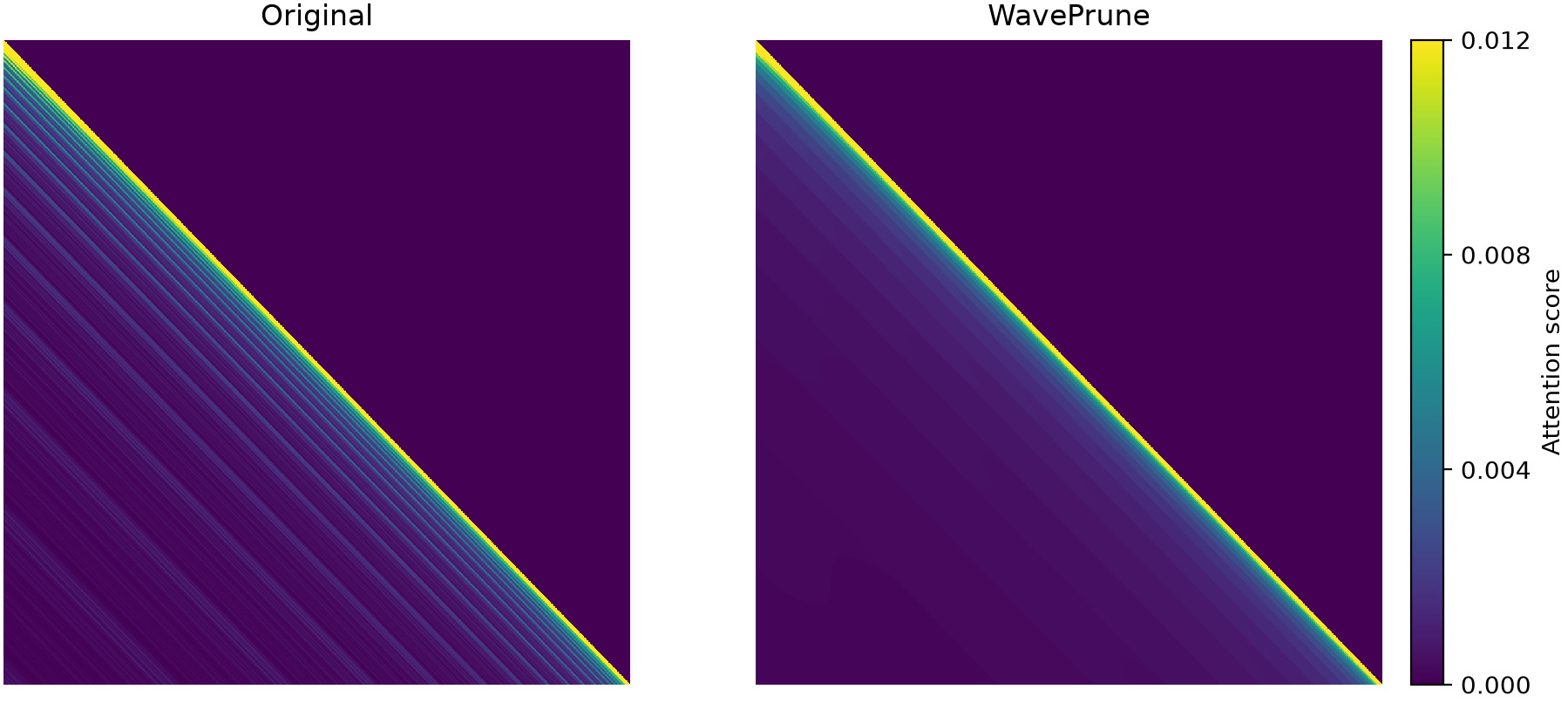}
    \caption{Attention maps under constant queries and keys ($q_i = k_i = \mathbf{1}$; $N = 1024$, $d_h = 64$, $\theta = 10^4$), where the intended lag is $0$ (i.e., each token should attend only to itself). RoPE alone produces many ghost sub-diagonals (left); WavePrune suppresses these artifacts and recovers a clean attention map (right).}
    \label{fig:controlled_qk_attn}
\end{figure}

\section{Attention Visualizations}\label{app:additional_attention_score_visualizations}
\subsection{Experimental Details}\label{app:attention_visualization_details}

We produce the attention maps by running Qwen3-8B and Ministral3-3B on 100 sequences sampled from PG19. Each sequence is truncated to 4096 tokens. For each model, we extract the attention scores and average these maps across the 100 sequences.

\subsection{Layer-wide and Head-level Comparisons}
Our visualization of attention maps includes four layer-wide attention score and two head-level per-channel logit comparisons. The example in the main text is Qwen3-8B layer 0, head 20 (Figure~\ref{fig:qwen3_layer0_attn_scores_comparison}); additional layer-wise visualizations below use Qwen3-8B layer 32 head 12 (Figure~\ref{fig:qwen3_layer32_attn_scores_comparison}), Ministral3-3B layer 0 head 23 (Figure~\ref{fig:ministral_layer0_attn_scores_comparison}), and Ministral3-3B layer 1 head 17 (Figure~\ref{fig:ministral_layer1_attn_scores_comparison}). In each visualization, ``original'' denotes the unmodified attention computation and ``after WavePrune'' denotes attention with each channel truncated at its wavelength. The full-map panels show layer-wide scores, while the cropped panels show the indicated head in greater detail. The per-channel figures show centered logit contributions for Qwen3-8B layer 0, heads 20 (Figure~\ref{fig:qwen3_layer0_head20_channel_visualization_original} vs. Figure~\ref{fig:qwen3_layer0_head20_channel_visualization_waveprune}) and 6 (Figure~\ref{fig:qwen3_layer0_head6_channel_visualization_original} vs. Figure~\ref{fig:qwen3_layer0_head6_channel_visualization_waveprune}), each before and after applying WavePrune.

\begin{figure}[!htbp]
    \centering
    \begin{minipage}[b]{0.44\textwidth}
        \centering
        \appendixfig[width=\textwidth,height=2\textwidth]{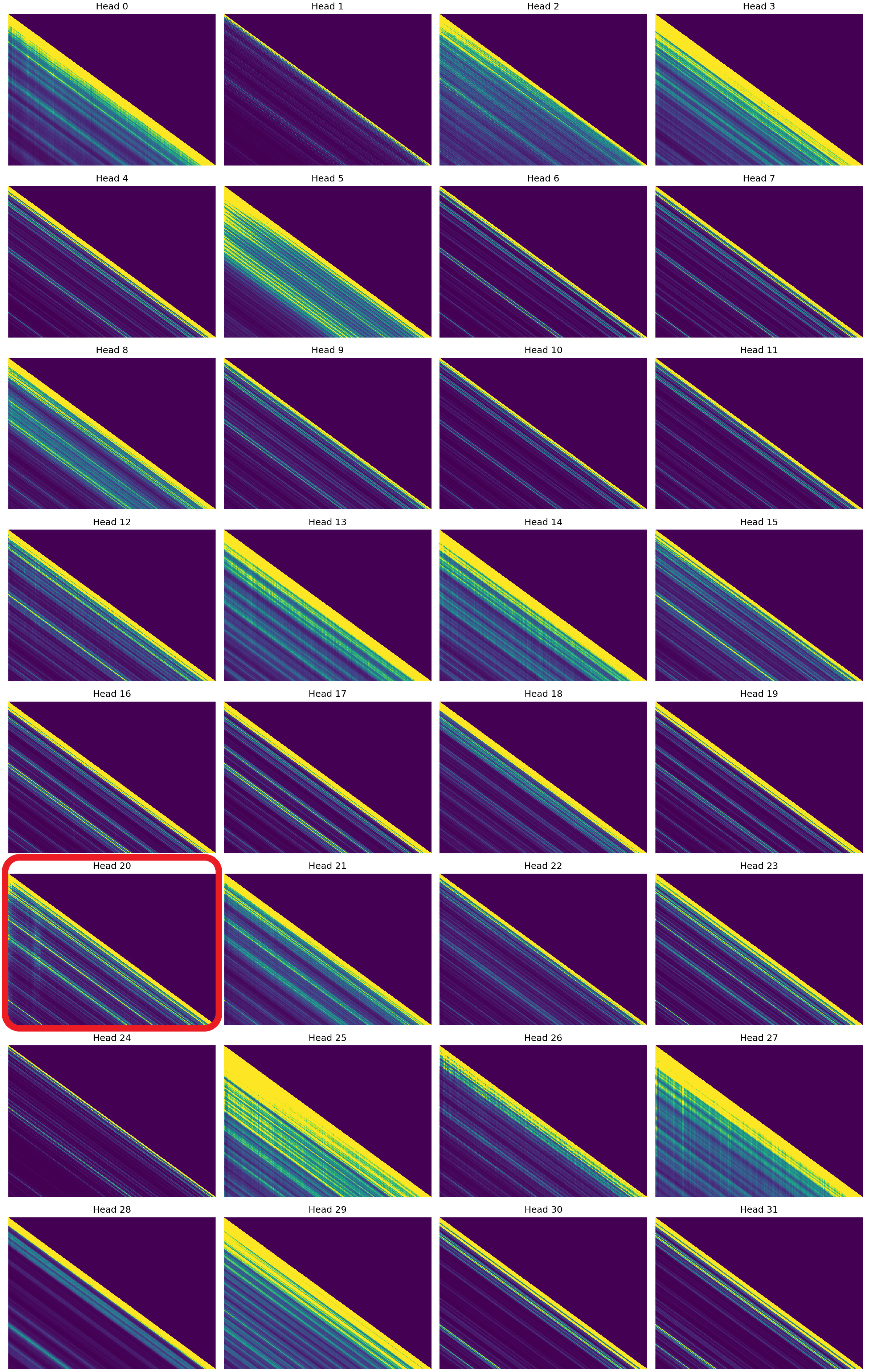}
        \subcaption{Attention scores (original)}
        \label{fig:qwen3_original_attn_scores}
    \end{minipage}
    \hspace{0.08\textwidth}
    \begin{minipage}[b]{0.44\textwidth}
        \centering
        \appendixfig[width=\textwidth,height=2\textwidth]{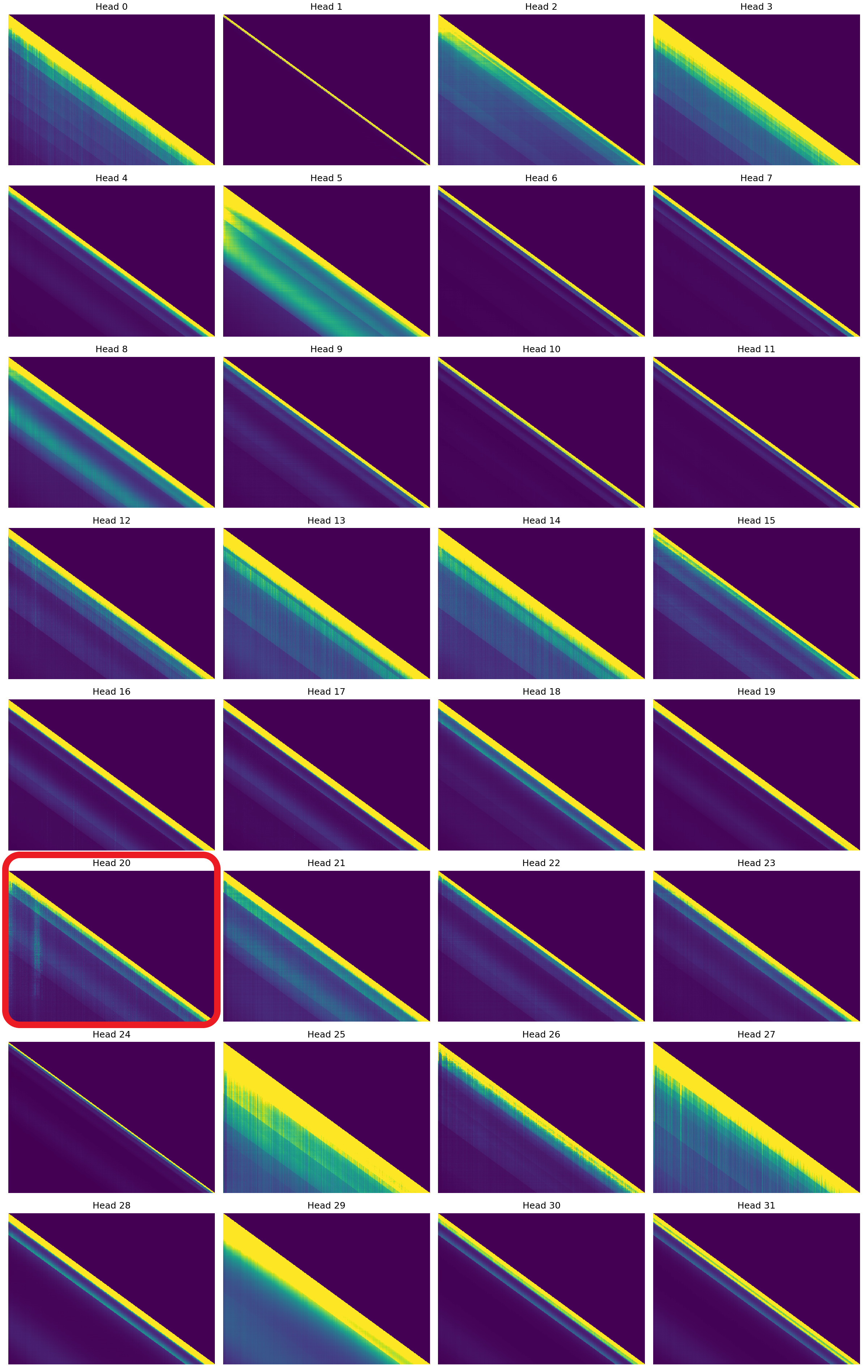}
        \subcaption{Attention scores (after WavePrune)}
        \label{fig:qwen3_waveprune_attn_scores}
    \end{minipage}
    \\[0.5em]
    \begin{minipage}[b]{0.44\textwidth}
        \centering
        \appendixfig[width=\textwidth,height=\textwidth]{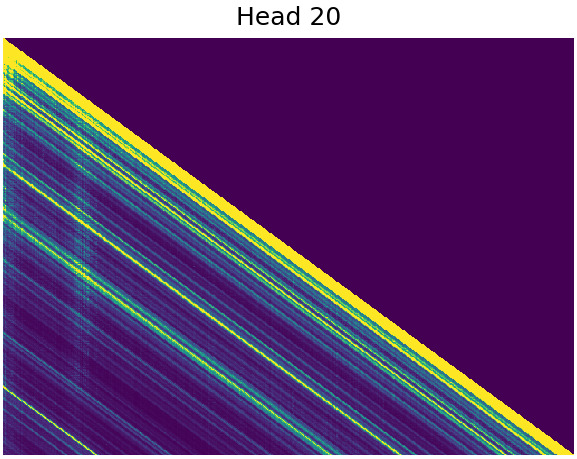}
        \subcaption{Attention scores (original, head 20)}
        \label{fig:qwen3_original_attn_scores_head20_crop}
    \end{minipage}
    \hspace{0.08\textwidth}
    \begin{minipage}[b]{0.44\textwidth}
        \centering
        \appendixfig[width=\textwidth,height=\textwidth]{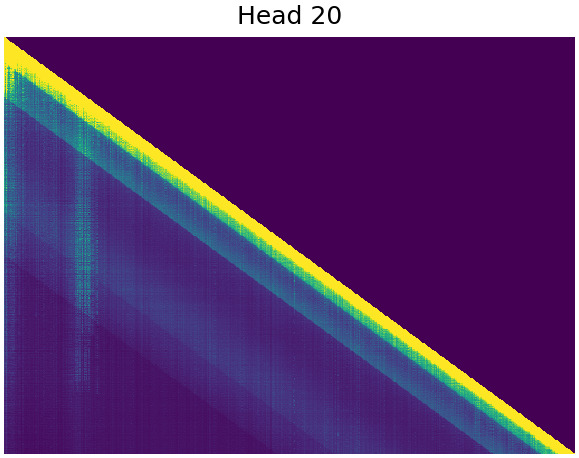}
        \subcaption{Attention scores (after WavePrune, head 20)}
        \label{fig:qwen3_waveprune_attn_scores_head20_crop}
    \end{minipage}
    \caption{Attention scores in Qwen3-8B's layer 0. (a-b) Layer-wide attention scores. In the original model (a), attention is scattered across multiple ghost sub-diagonals. WavePrune (b) effectively eliminates these artifacts, allowing attention to concentrate on the intended lags. (c-d) We use head 20 as an example to illustrate the elimination of ghost sub-diagonals after applying WavePrune.}
    \label{fig:qwen3_layer0_attn_scores_comparison}
\end{figure}

\newpage
\begin{figure}[!htbp]
    \centering
    \begin{minipage}[b]{0.44\textwidth}
        \centering
        \appendixfig[width=\textwidth,height=2\textwidth]{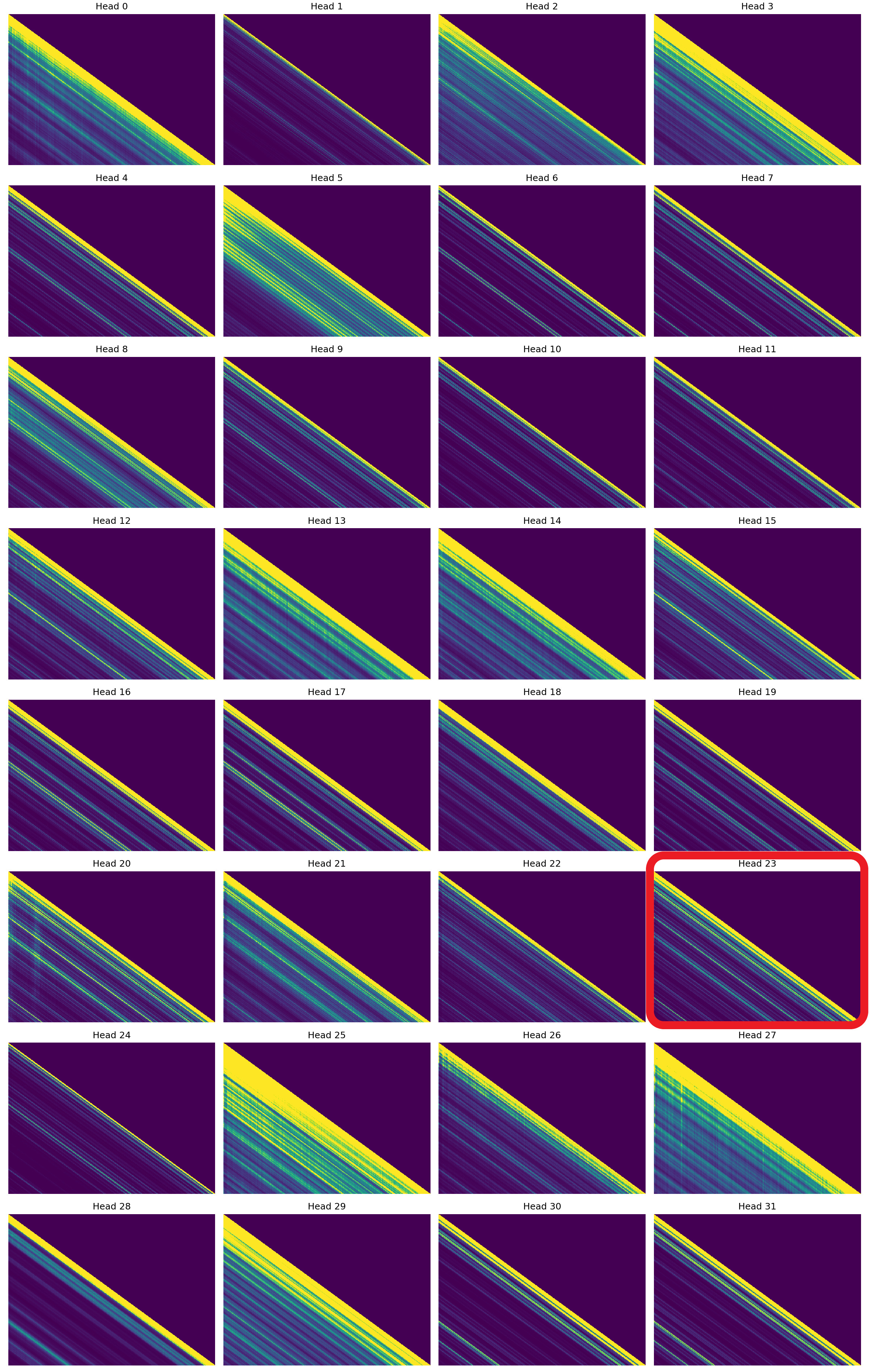}
        \subcaption{Attention scores (original)}
        \label{fig:ministral_original_attn_scores}
    \end{minipage}
    \hspace{0.08\textwidth}
    \begin{minipage}[b]{0.44\textwidth}
        \centering
        \appendixfig[width=\textwidth,height=2\textwidth]{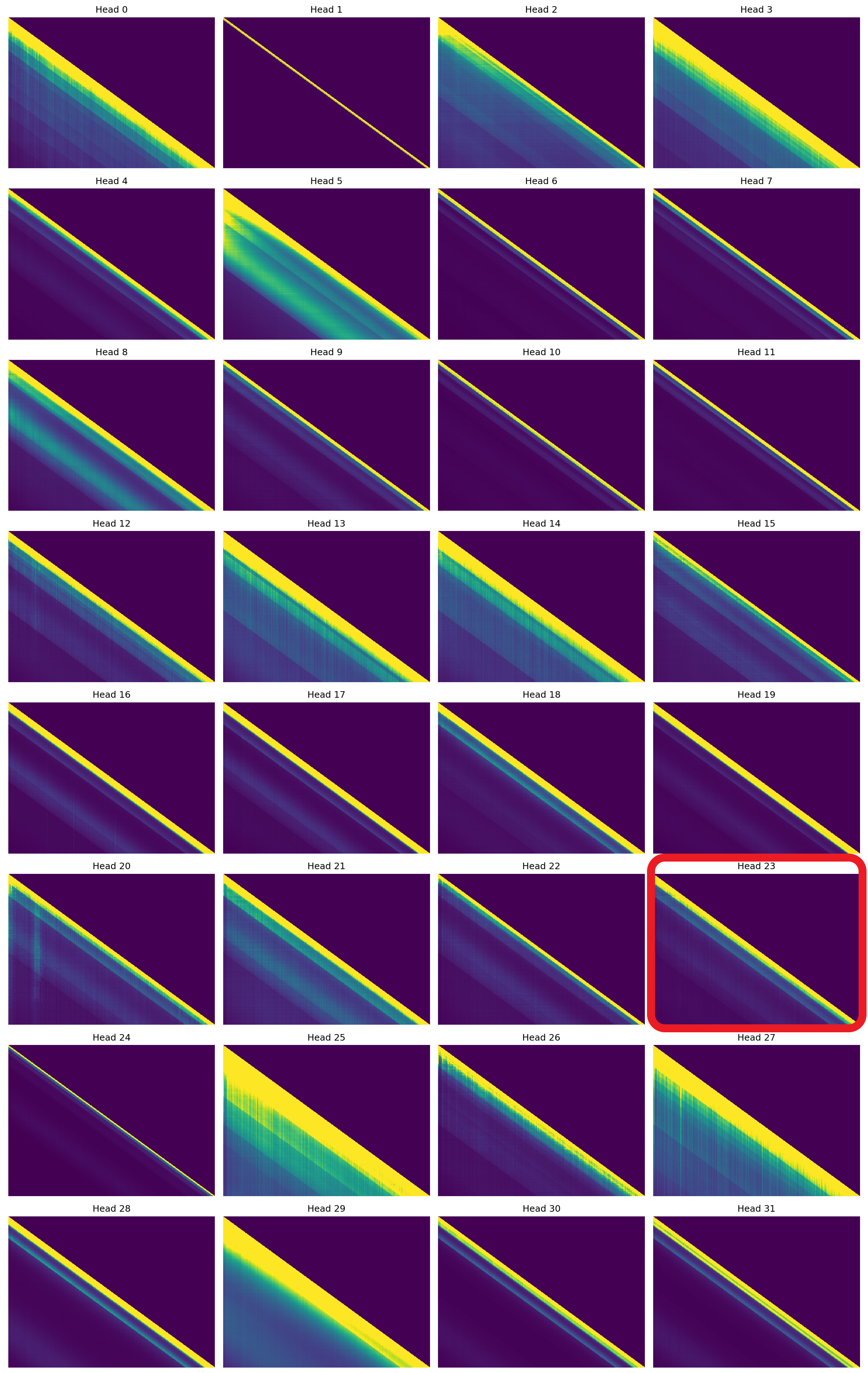}
        \subcaption{Attention scores (after WavePrune)}
        \label{fig:ministral_waveprune_attn_scores}
    \end{minipage}
    \\[0.5em]
    \begin{minipage}[b]{0.44\textwidth}
        \centering
        \appendixfig[width=\textwidth,height=\textwidth]{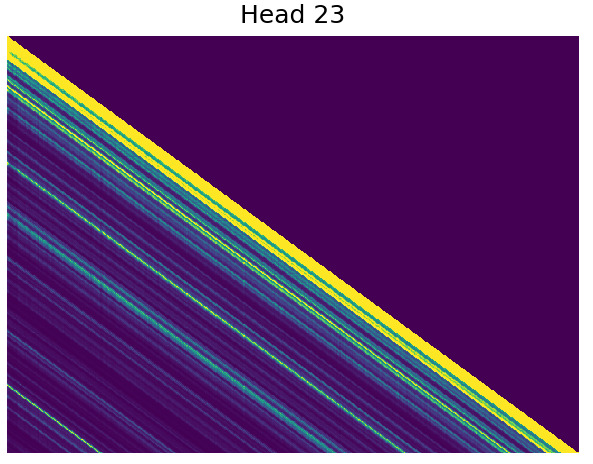}
        \subcaption{Attention scores (original, head 23)}
        \label{fig:ministral_original_attn_scores_head23_crop}
    \end{minipage}
    \hspace{0.08\textwidth}
    \begin{minipage}[b]{0.44\textwidth}
        \centering
        \appendixfig[width=\textwidth,height=\textwidth]{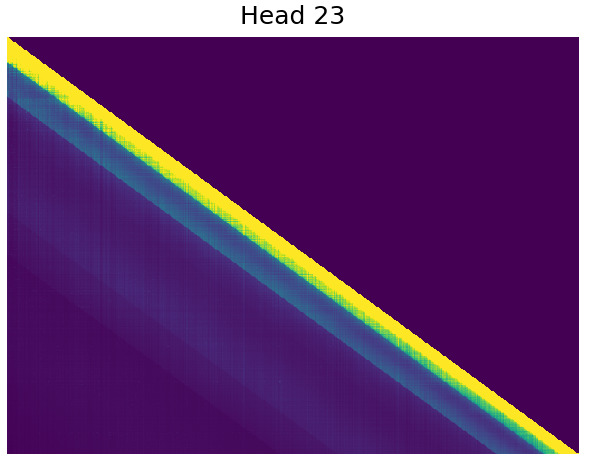}
        \subcaption{Attention scores (after WavePrune, head 23)}
        \label{fig:ministral_waveprune_attn_scores_head23_crop}
    \end{minipage}
    \caption{Attention scores in Ministral3-3B's layer 0. (a-b) Layer-wide attention scores. In the original model (a), attention is scattered across multiple ghost sub-diagonals. WavePrune (b) effectively eliminates these artifacts, allowing attention to concentrate on the intended lags. (c-d) We use head 23 as an example to illustrate the elimination of ghost sub-diagonals after applying WavePrune.}
    \label{fig:ministral_layer0_attn_scores_comparison}
\end{figure}

\newpage
\begin{figure}[!htbp]
    \centering
    \begin{minipage}[b]{0.44\textwidth}
        \centering
        \appendixfig[width=\textwidth,height=2\textwidth]{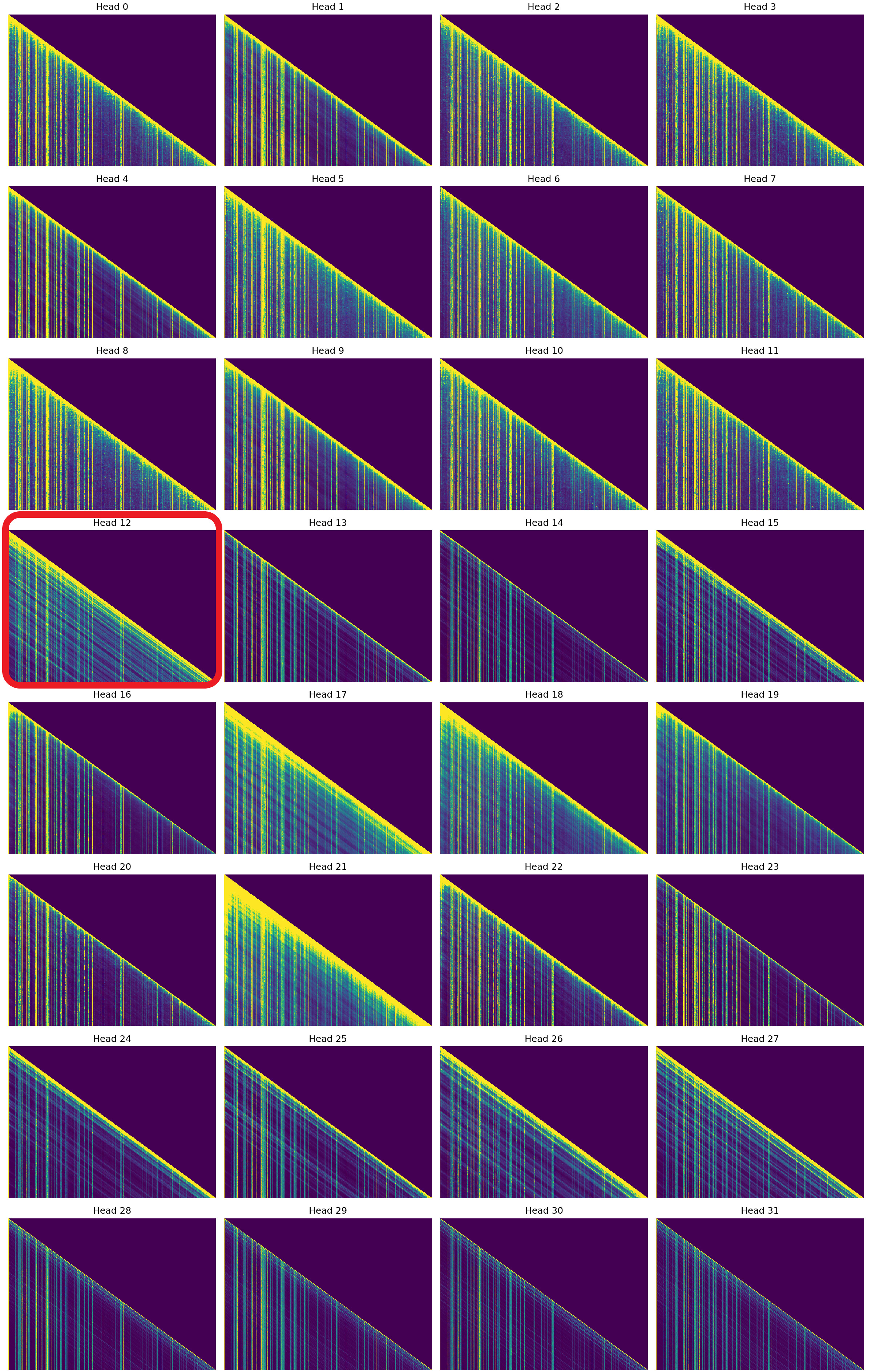}
        \subcaption{Attention scores (original)}
    \end{minipage}
    \hspace{0.08\textwidth}
    \begin{minipage}[b]{0.44\textwidth}
        \centering
        \appendixfig[width=\textwidth,height=2\textwidth]{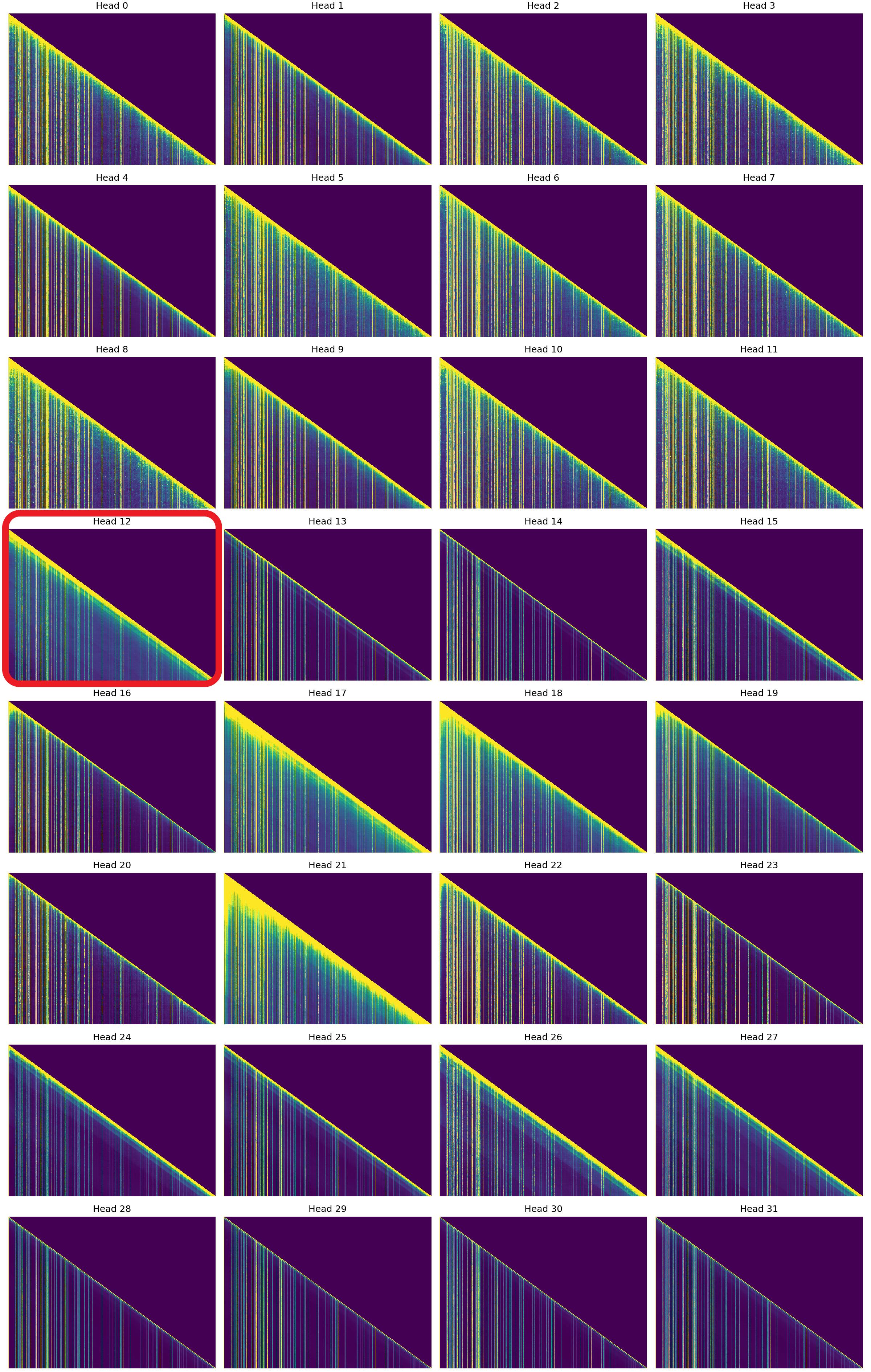}
        \subcaption{Attention scores (after WavePrune)}
    \end{minipage}
    \\[0.5em]
    \begin{minipage}[b]{0.44\textwidth}
        \centering
        \appendixfig[width=\textwidth,height=\textwidth]{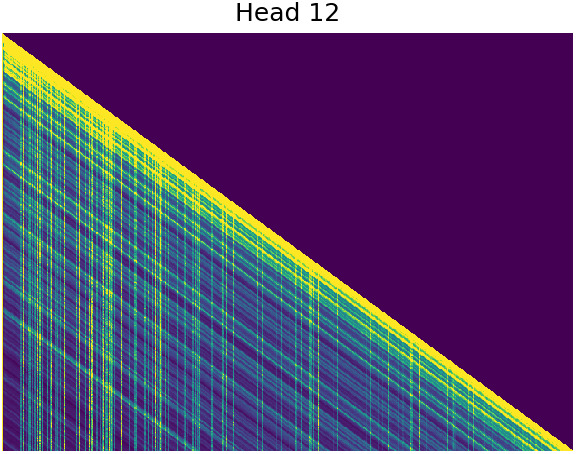}
        \subcaption{Attention scores (original, head 12)}
    \end{minipage}
    \hspace{0.08\textwidth}
    \begin{minipage}[b]{0.44\textwidth}
        \centering
        \appendixfig[width=\textwidth,height=\textwidth]{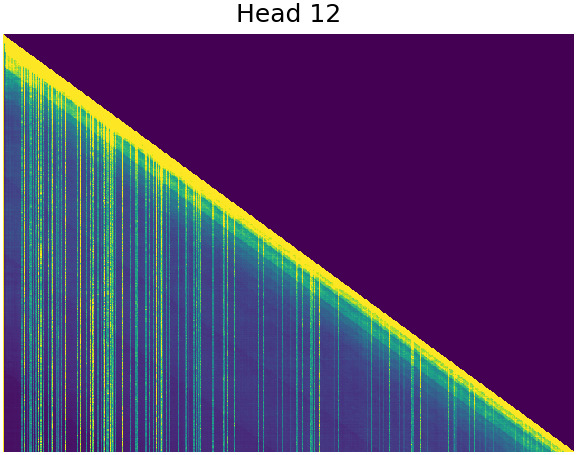}
        \subcaption{Attention scores (after WavePrune, head 12)}
    \end{minipage}
    \caption{Attention scores in Qwen3-8B's layer 32. (a-b) Layer-wide attention scores. In the original model (a), attention is scattered across multiple ghost sub-diagonals. WavePrune (b) effectively eliminates these artifacts, allowing attention to concentrate on the intended lags. (c-d) We use head 12 as an example to illustrate the elimination of ghost sub-diagonals after applying WavePrune.}
    \label{fig:qwen3_layer32_attn_scores_comparison}
\end{figure}

\newpage
\begin{figure}[!htbp]
    \centering
    \begin{minipage}[b]{0.44\textwidth}
        \centering
        \appendixfig[width=\textwidth,height=2\textwidth]{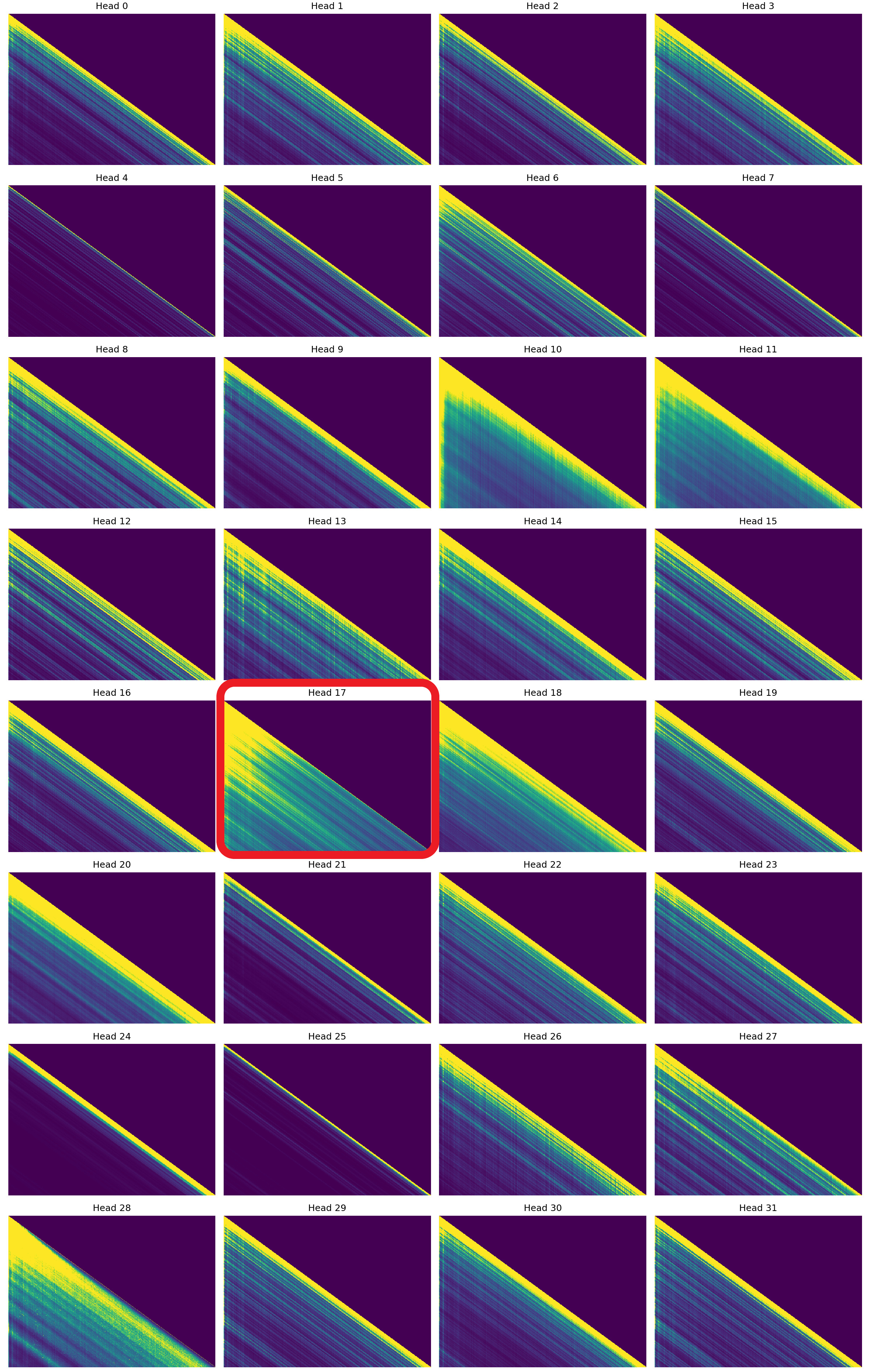}
        \subcaption{Attention scores (original)}
    \end{minipage}
    \hspace{0.08\textwidth}
    \begin{minipage}[b]{0.44\textwidth}
        \centering
        \appendixfig[width=\textwidth,height=2\textwidth]{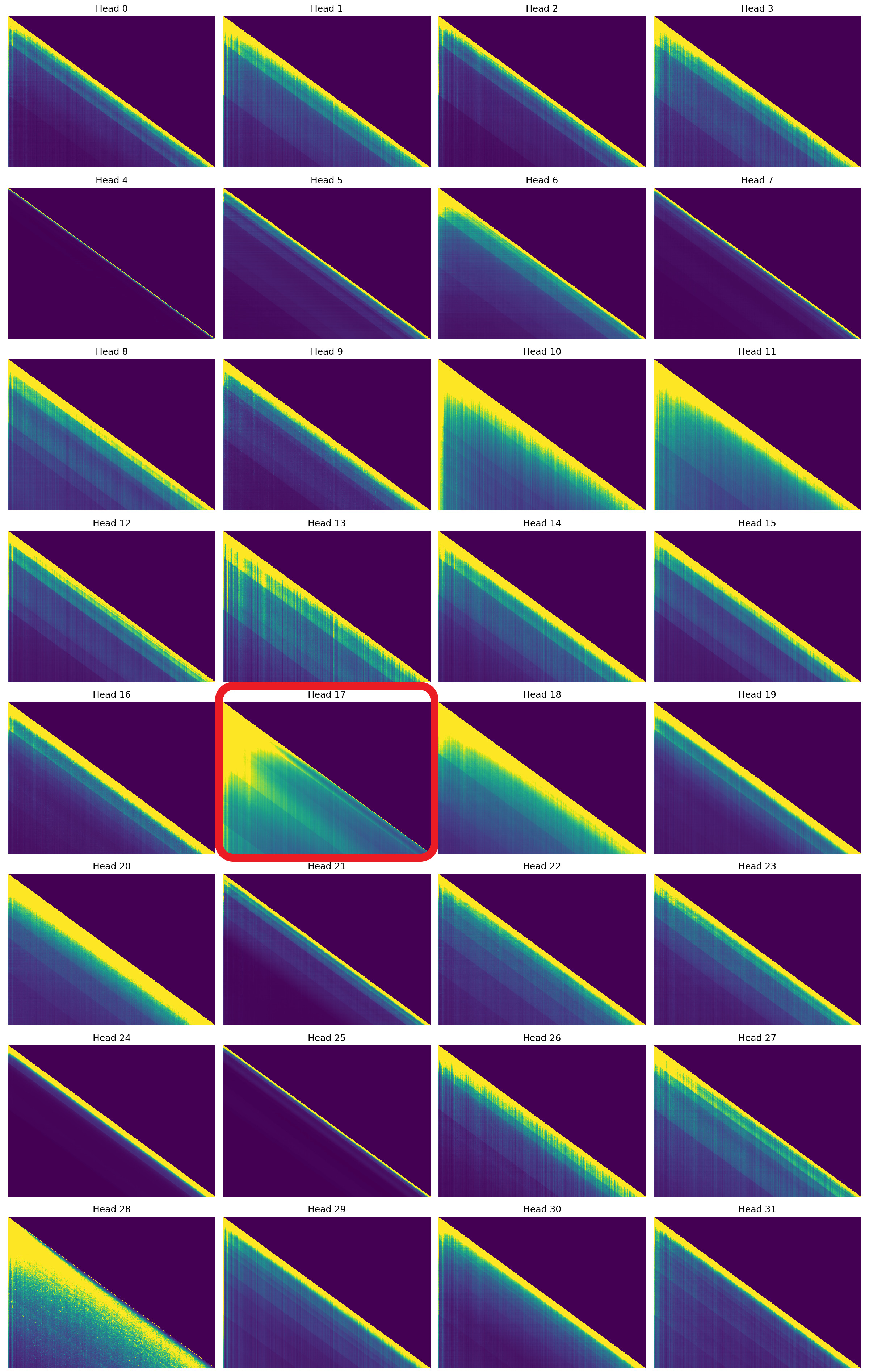}
        \subcaption{Attention scores (after WavePrune)}
    \end{minipage}
    \\[0.5em]
    \begin{minipage}[b]{0.44\textwidth}
        \centering
        \appendixfig[width=\textwidth,height=\textwidth]{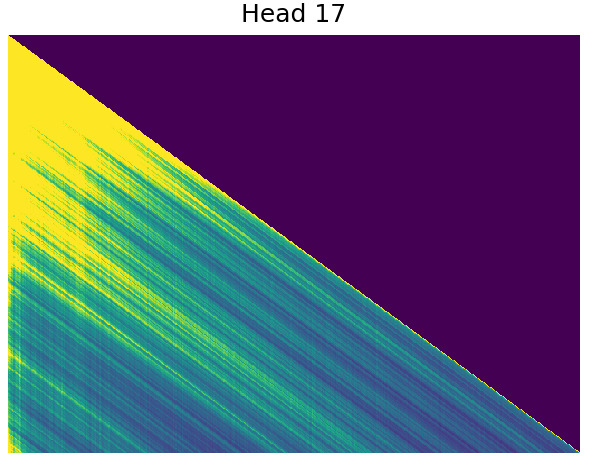}
        \subcaption{Attention scores (original, head 17)}
    \end{minipage}
    \hspace{0.08\textwidth}
    \begin{minipage}[b]{0.44\textwidth}
        \centering
        \appendixfig[width=\textwidth,height=\textwidth]{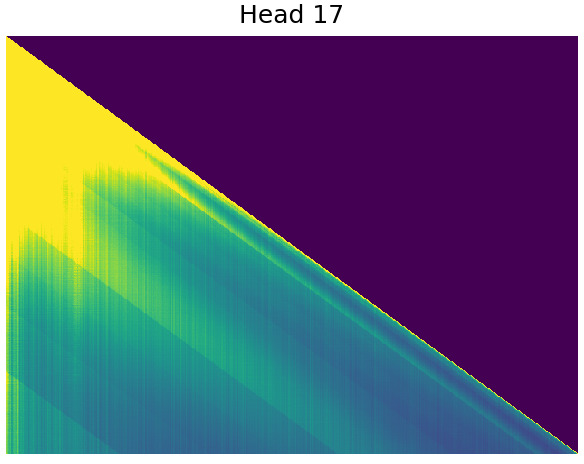}
        \subcaption{Attention scores (after WavePrune, head 17)}
    \end{minipage}
    \caption{Attention scores in Ministral3-3B's layer 1. (a-b) Layer-wide attention scores. In the original model (a), attention is scattered across multiple ghost sub-diagonals. WavePrune (b) effectively eliminates these artifacts, allowing attention to concentrate on the intended lags. (c-d) We use head 17 as an example to illustrate the elimination of ghost sub-diagonals after applying WavePrune.}
    \label{fig:ministral_layer1_attn_scores_comparison}
\end{figure}

\newpage

\begin{figure}[!htbp]
    \centering
    \appendixfig[width=0.66\textwidth,height=0.66\textwidth]{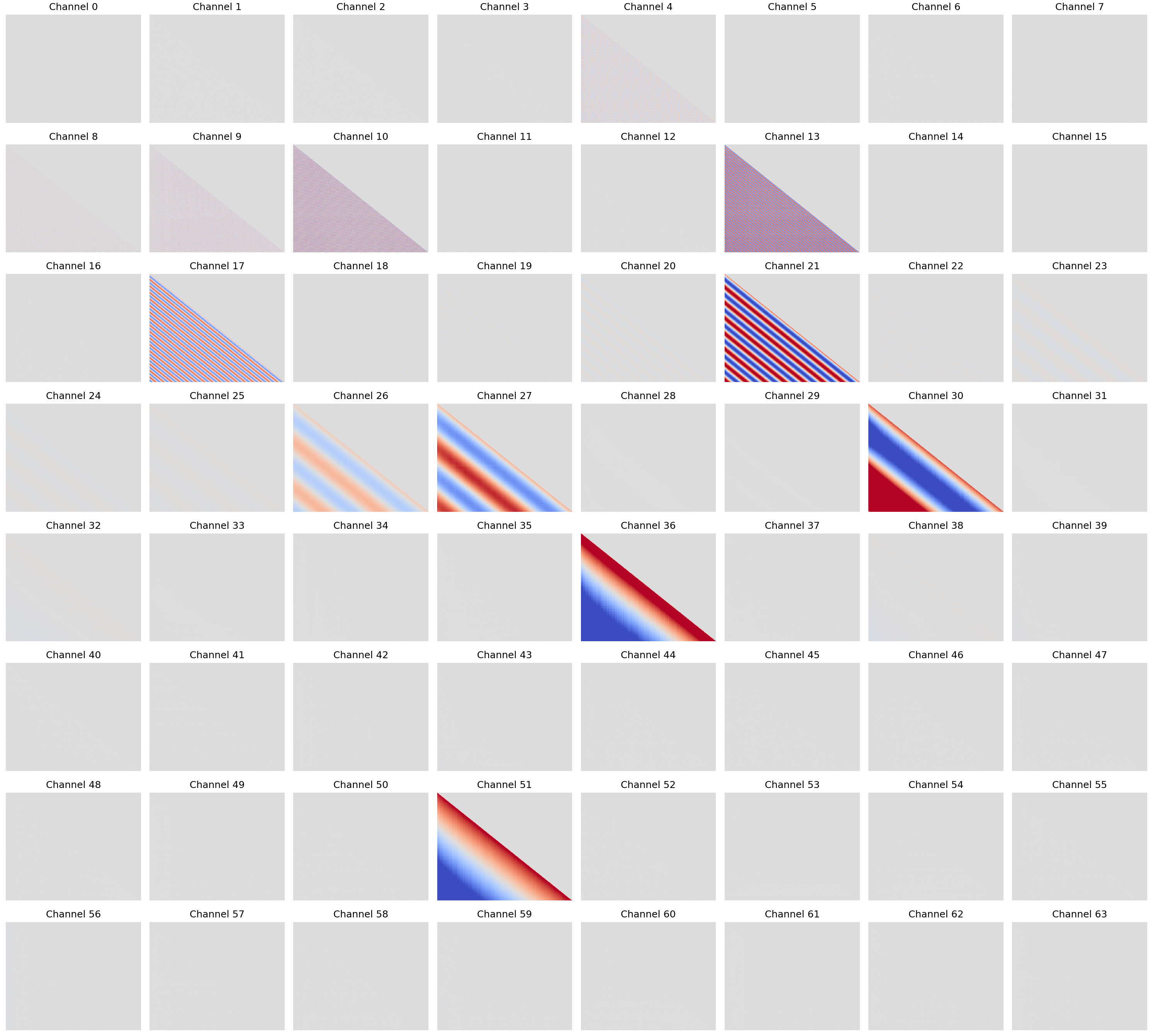}
    \caption{Channel-wise decomposition of attention logits for Qwen3-8B, layer 0 head 20 (centered by subtracting the mean logit value).}
    \label{fig:qwen3_layer0_head20_channel_visualization_original}
\end{figure}
\begin{figure}[!htbp]
    \centering
    \appendixfig[width=0.66\textwidth,height=0.66\textwidth]{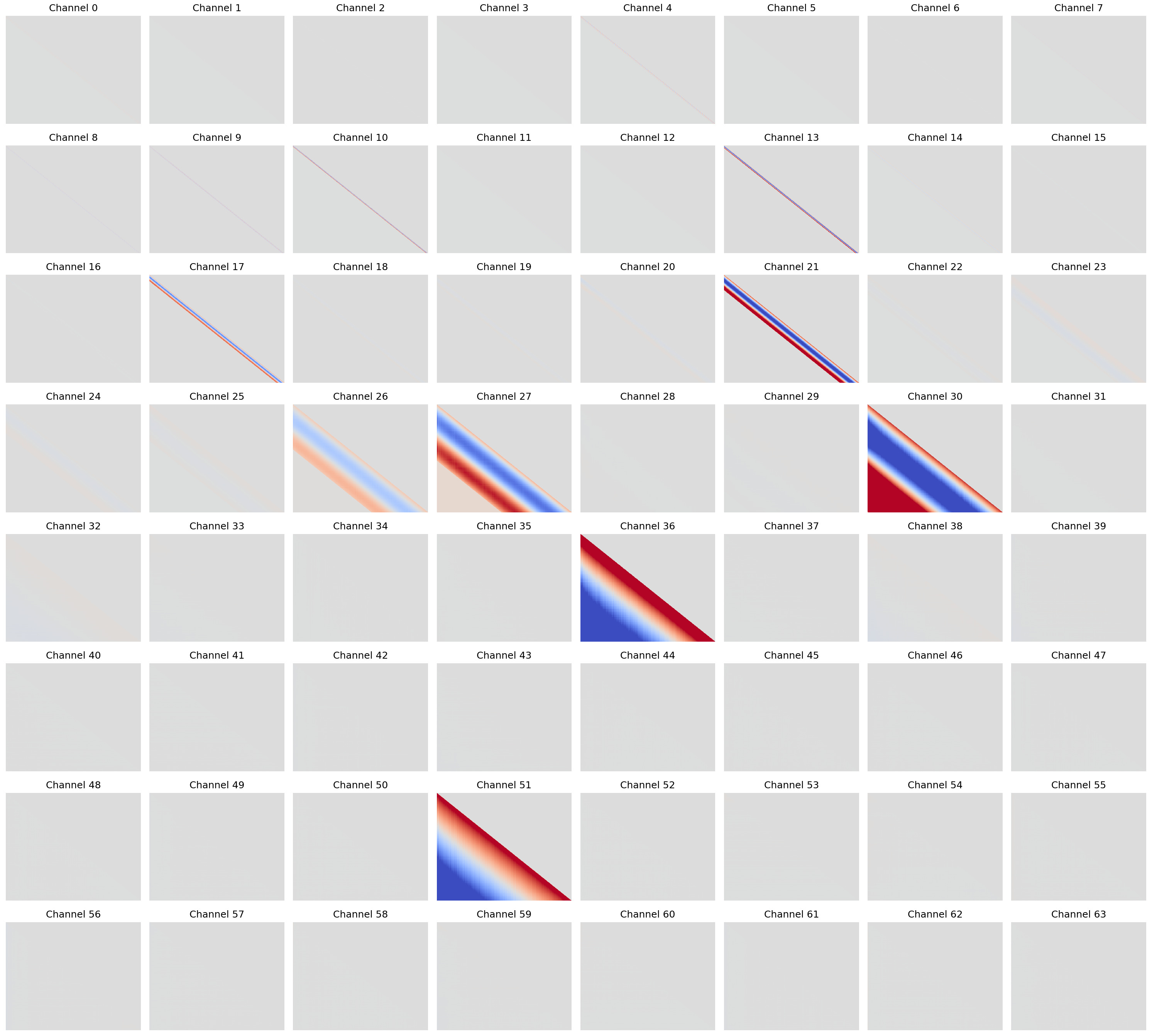}
    \caption{Channel-wise decomposition of attention logits for Qwen3-8B, layer 0 head 20 after applying WavePrune (centered by subtracting the mean logit value).}
\label{fig:qwen3_layer0_head20_channel_visualization_waveprune}
\end{figure}

\newpage
\begin{figure}[!htbp]
    \centering
    \appendixfig[width=0.66\textwidth,height=0.66\textwidth]{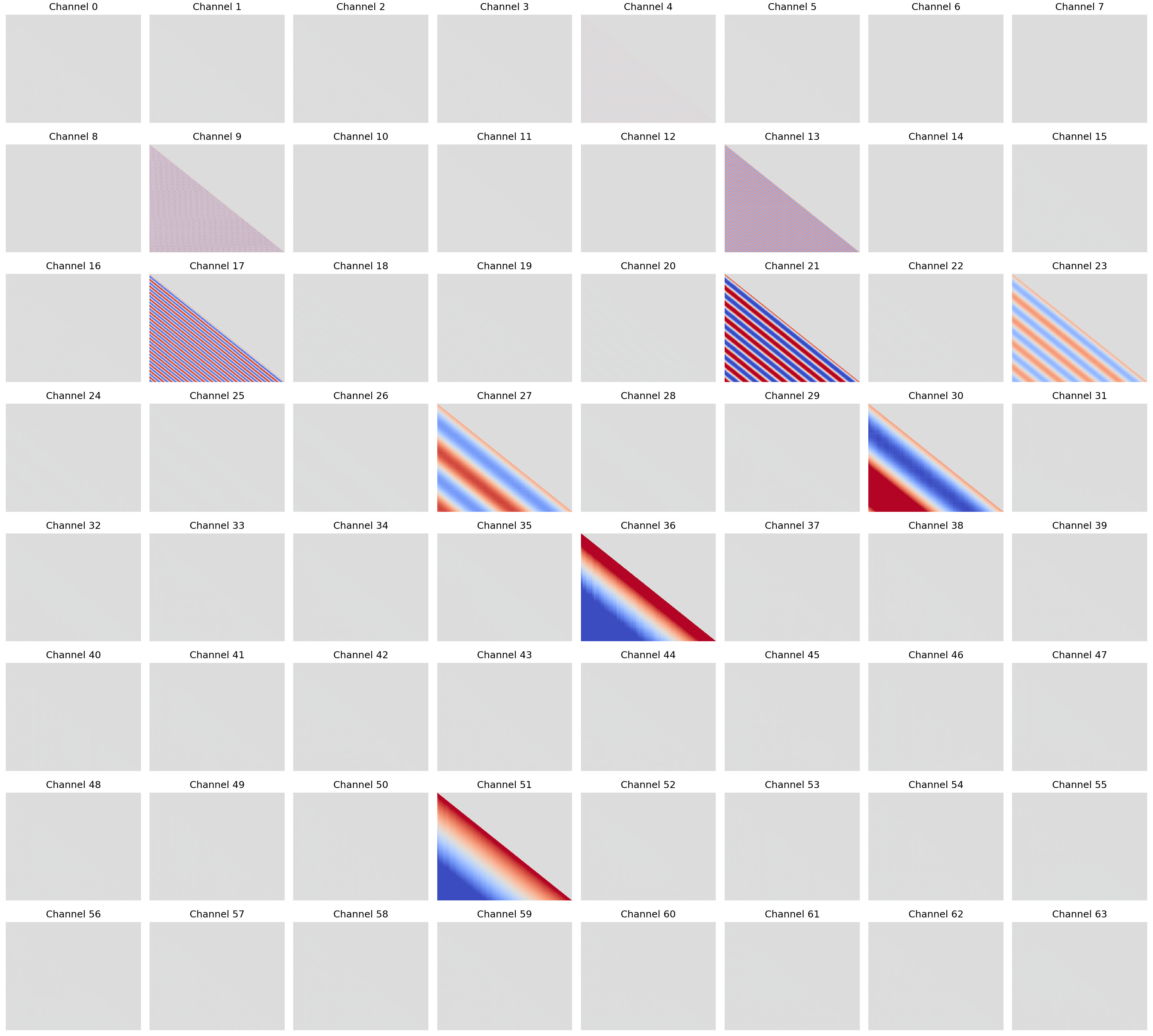}
    \caption{Channel-wise decomposition of attention logits for Qwen3-8B, layer 0 head 6 (centered by subtracting the mean logit value).}
    \label{fig:qwen3_layer0_head6_channel_visualization_original}
\end{figure}
\begin{figure}[!htbp]
    \centering
    \appendixfig[width=0.66\textwidth,height=0.66\textwidth]{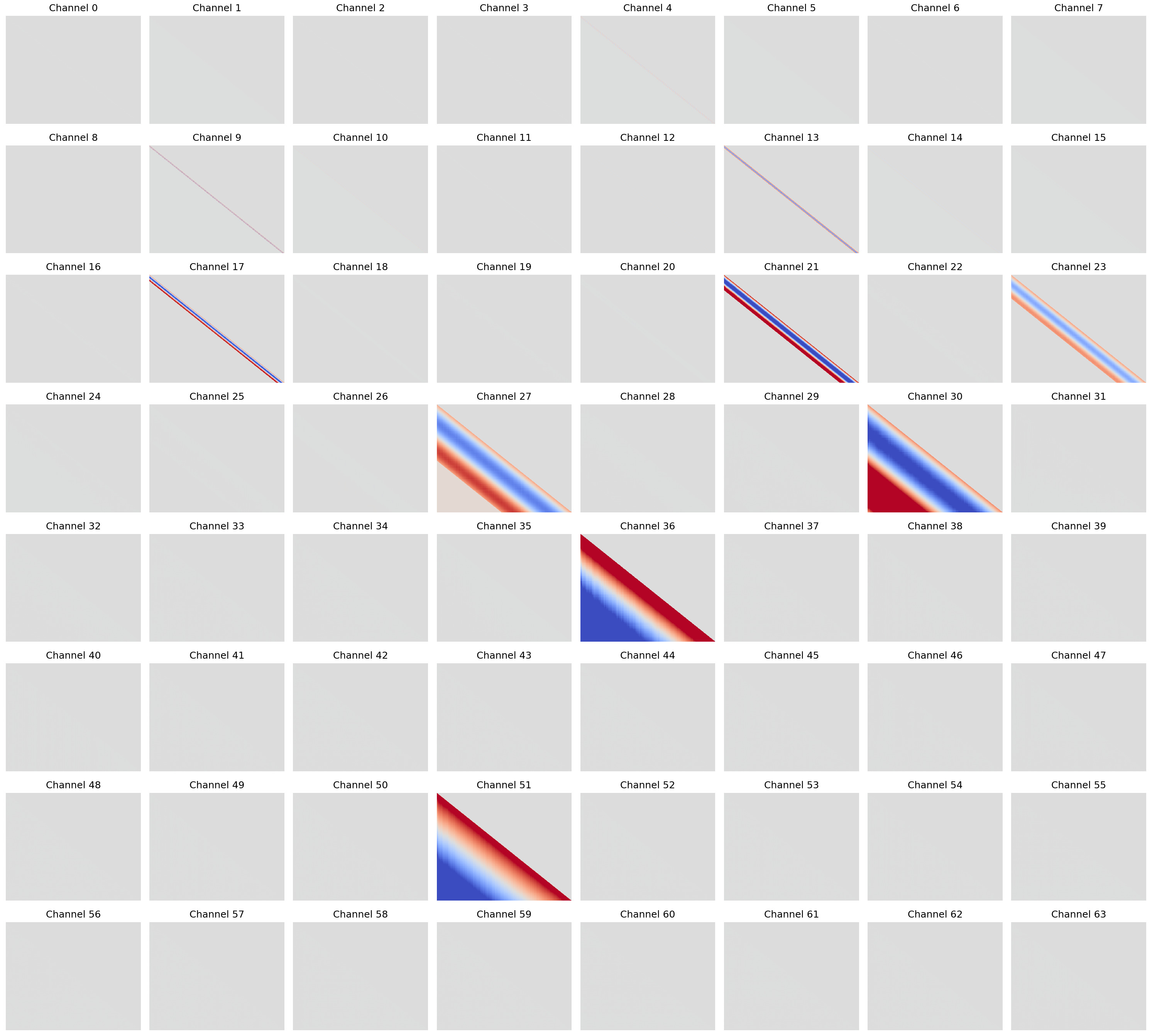}
    \caption{Channel-wise decomposition of attention logits for Qwen3-8B, layer 0 head 6 after applying WavePrune (centered by subtracting the mean logit value).}
    \label{fig:qwen3_layer0_head6_channel_visualization_waveprune}
\end{figure}

\end{document}